\documentclass{article}

\usepackage{amsmath}   %{amsmath} below has been commented.
\allowdisplaybreaks[4]   %Allow equation to break into multiple pages
\makeatletter
{% \begin{breakablealgorithm}
	\begin{center}
		\refstepcounter{algorithm}% New algorithm
		\hrule height.8pt depth0pt \kern2pt% \@fs@pre for \@fs@ruled
		\renewcommand{\caption}[2][\relax]{% Make a new \caption
			{\raggedright\textbf{\ALG@name~\thealgorithm} ##2\par}%
			\ifx\relax##1\relax % #1 is \relax
			\addcontentsline{loa}{algorithm}{\protect\numberline{\thealgorithm}##2}%
			\else % #1 is not \relax
			\addcontentsline{loa}{algorithm}{\protect\numberline{\thealgorithm}##1}%
			\fi
			\kern2pt\hrule\kern2pt
		}
	}{% \end{breakablealgorithm}
		\kern2pt\hrule\relax% \@fs@post for \@fs@ruled
	\end{center}
}
\makeatother

\usepackage{microtype}
\usepackage{graphicx}
\usepackage{subfigure}
\usepackage{booktabs} % for professional tables
\usepackage{algorithmic,algorithm}
\usepackage{wrapfig,caption}
\usepackage[algo2e]{algorithm2e}
\usepackage{multirow} 
\usepackage{hyperref}
\usepackage{amsthm}
\usepackage{amsfonts}
\usepackage{amssymb}
\usepackage{mathrsfs}
\usepackage{enumerate}
\usepackage{empheq}

\makeatletter
\newcommand*\dotp{\mathpalette\dotp@{.5}}
\newcommand*\dotp@[2]{\mathbin{\vcenter{\hbox{\scalebox{#2}{$\m@th#1\bullet$}}}}}
\makeatother

\DeclareMathOperator*{\argmin}{arg\,min}

\usepackage{enumerate}
\usepackage[shortlabels]{enumitem} 
\usepackage{cleveref}
\newtheorem{lemma}{Lemma}

\newtheorem{thm}{Theorem}

\newtheorem{assum}{Assumption}

\newcommand{\prox}[1]{\mathrm{prox}_{#1}}

\newcommand{\zero}{\mathbf{0}}

\usepackage[accepted]{icml2020}

\usepackage{authblk}    %newly added for arxiv
\usepackage[T1]{fontenc}    %newly added for arxiv
\usepackage[utf8]{inputenc}    %newly added for arxiv

\title{\textbf{Momentum with Variance Reduction for Nonconvex Composition Optimization}}
\author[1]{\textit{Ziyi Chen}}
\author[1]{\textit{Yi Zhou}}

\affil[1]{Department of Electrical and Computer Engineering, University of Utah, USA}

\affil[ ]{\small {\{u1276972,yi.zhou\}@utah.edu}}
\date{}

\begin{document}
\twocolumn[ 
\maketitle
\thispagestyle{empty}
%\icmltitle{Momentum with Variance Reduction for Nonconvex Composition Optimization with Near-optimal Sample Complexity}
%\icmltitle{Submission and Formatting Instructions for \\
%           International Conference on Machine Learning (ICML 2020)}



\icmlsetsymbol{equal}{*}

\begin{icmlauthorlist}
%\icmlauthor{Ziyi Chen}{1}
%\icmlauthor{Yi Zhou}{2}
%\icmlauthor{Cieua Vvvvv}{goo}
%\icmlauthor{Iaesut Saoeu}{ed}
%\icmlauthor{Fiuea Rrrr}{to}
%\icmlauthor{Tateu H.~Yasehe}{ed,to,goo}
%\icmlauthor{Aaoeu Iasoh}{goo}
%\icmlauthor{Buiui Eueu}{ed}
%\icmlauthor{Aeuia Zzzz}{ed}
%\icmlauthor{Bieea C.~Yyyy}{to,goo}
%\icmlauthor{Teoau Xxxx}{ed}
%\icmlauthor{Eee Pppp}{ed}
\end{icmlauthorlist}

%1276972@utah.edu
%yi.zhou@utah.edu
%Department of Electrical and Computer Engineering, University of Utah, USA
%\icmlaffiliation{1}{}
%\icmlaffiliation{to}{}
%\icmlaffiliation{goo}{}
%\icmlaffiliation{ed}{}

%\icmlcorrespondingauthor{Ziyi Chen}{1276972@utah.edu}
%\icmlcorrespondingauthor{Yi Zhou}{yi.zhou@utah.edu}

\vskip 0.3in
]




\begin{abstract}
Composition optimization is widely-applied in nonconvex machine learning. Various advanced stochastic algorithms that adopt momentum and variance reduction techniques have been developed for composition optimization. However, these algorithms do not fully exploit both techniques to accelerate the convergence and are lack of convergence guarantee in nonconvex optimization. This paper complements the existing literature by developing various momentum schemes with SPIDER-based variance reduction for nonconvex composition optimization. In particular, our momentum design requires less number of proximal mapping evaluations per-iteration than that required by the existing Katyusha momentum. Furthermore, our algorithm achieves near-optimal sample complexity results in both nonconvex finite-sum and online composition optimization and achieves a linear convergence rate under the gradient dominant condition. Numerical experiments demonstrate that our algorithm converges significantly faster than existing algorithms in nonconvex composition optimization.
\end{abstract}

\section{Introduction}
A variety of machine learning problems naturally have composition structure and can be formulated as
\begin{align}
\min_{ x \in \mathbb{R}^{d}} \Phi(x)=F(x)+r(x), ~\text{where } F(x)=f\big(g(x)\big). \tag{P} \nonumber
\end{align}
In the above problem, $f: \mathbb{R}^p\to \mathbb{R}$ is a differentiable loss function, $g: \mathbb{R}^d\to \mathbb{R}^p$ is a differentiable mapping and $r: \mathbb{R}^d\to \mathbb{R}$ corresponds to a possibly non-smooth regularizer. Such a composition optimization problem covers many important applications including value function approximation in reinforcement learning, risk-averse portfolio optimization \cite{zhang2019stochastic}, stochastic neighborhood embedding \cite{liu2017variance} and sparse additive model \cite{wang2017stochastic}, etc. We further elaborate the composition structure of these problems in \Cref{append: ml example}. 

A simple algorithm to solve the composition problem (P) is the gradient descent algorithm, which, however, induces much computation overhead in the gradient evaluation when the datasize is large. Specifically, consider the following finite-sum formulation of the composition problem (P):
\begin{align}
	&\text{\big($\Sigma^2$ \!\big)}: ~f(y)=\frac{1}{N}\sum_{k=1}^{N} f_k(y), ~g(x)=\frac{1}{n}\sum_{i=1}^{n} g_{i}(x), \nonumber
\end{align}
where we call the above formulation $(\Sigma^2)$ as it involves the finite-sum structure in both $f$ and $g$. 
Due to the composition structure, the gradient of any sample loss $f_i(g(x))$ involves the entire Jacobian matrix $g'(x)$, which is computational burdensome when $n$ is large. Such a challenge for computing composition gradients has motivated the design of various stochastic algorithms to reduce the sample complexity in composition optimization. In particular, \cite{wang2017stochastic} proposed a stochastic composition gradient descent (SCGD) algorithm for solving unregulrarized nonconvex composition problems $(\Sigma^2)$. As the SCGD algorithm suffers from a high stochastic gradient variance and slow convergence rate, more recent works have developed various variance reduction techniques for accelerating stochastic composition optimization. Specifically, the well-known SVRG scheme in \cite{johnson2013accelerating} has been exploited by \cite{lian2016finite} and \cite{huo2018accelerated} to develop variance-reduced algorithms for solving the problem $(\Sigma^2)$ under strong convexity and non-convexity. Other variance reduction schemes such as SAGA \cite{defazio2014saga} and SPIDER \cite{fang2018spider} (a.k.a. SARAH \cite{nguyen2017sarah}) have also been exploited by \cite{zhang2019composite,zhang2019stochastic} to reduce the variance in stochastic composition optimization. In particular, the SPIDER-based stochastic composition optimization algorithm proposed in \cite{zhang2019stochastic} achieves a near-optimal sample complexity in nonconvex optimization. 

Another important and widely-applied technique for accelerating stochastic optimization is momentum, which has also been studied in stochastic composition optimization. For example, \cite{wang2017stochastic} proposed a momentum-accelerated SCGD algorithm that achieves an improved sample complexity in nonconvex composition optimization, and \cite{wang2016accelerating} further generalized it to solve regularized $(\Sigma^2)$ problems that are nonconvex.
Recently, \cite{xu2019katyusha} applied the Katyusha momentum developed in \cite{allen2017katyusha} to accelerate the SVRG-based composition optimization algorithm and achieved improved sample complexities in convex optimization. 
While these works exploit momentum to accelerate the practical convergence of composition optimization,  
their momentum schemes do not provide provable convergence guarantees for {\em nonconvex} composition optimization under {\em variance reduction}, which is widely-applied in large-scale stochastic optimization. Therefore, {the goal of this paper} is to develop a momentum with variance reduction scheme for stochastic nonconvex composition optimization. In particular, the algorithm design is desired to resolve the following issues.
\begin{itemize}[topsep=0pt, noitemsep, leftmargin=*]
	\item The existing Katyusha momentum with SVRG-based variance reduction proposed  in \cite{xu2019katyusha} only applies to convex composition optimization problems and requires two proximal mapping evaluations per-iteration, which induces much computation overhead when the proximal mapping is complex. {\em Can we design a momentum scheme with variance reduction for nonconvex composition optimization that requires less proximal mapping evaluations?}
	\item The existing variance-reduced composition optimization algorithms (without momentum) can achieve a near-optimal sample complexity in nonconvex scenarios. {\em Can we develop a momentum scheme with variance reduction for composition optimization that achieves a near-optimal sample complexity in nonconvex optimization and provides significant acceleration in practice?}  
	\item Momentum has not been developed with variance reduction for {\em online} nonconvex composition optimization problems. {\em Can we develop a momentum scheme with variance reduction that is applicable to online nonconvex composition optimization problems with provable convergence guarantee?}
\end{itemize}

In this paper, we provide positive answers to the questions mentioned above. Our developed momentum \& variance reduction scheme for composition optimization is applicable to both finite-sum and online cases and achieves the state-of-the-art sample complexity results in nonconvex scenarios. We summarize our contributions as follows and compare the sample complexities of all related algorithms in \Cref{table_nonconvex_algs}.

\subsection{Our Contributions}
We first study a special case of the composition optimization problem $(\Sigma^2)$ where $N=1$ and only the mapping $g$ has finite-sum structure. To solve such a nonconvex composition problem (referred to as $(\Sigma^1)$), we propose a stochastic algorithm MVRC-1 that implements both momentum and SpiderBoost-based variance reduction. Our momentum scheme is simpler and computationally lighter than the existing Katyusha momentum studied in \cite{xu2019katyusha}. In specific, our momentum scheme requires only one proximal mapping evaluation per-iteration, whereas the Katyusha momentum requires two proximal mapping evaluations per-iteration. Moreover, the momentum scheme of MVRC can adopt very flexible momentum parameter scheduling as we elaborate below.

Under a diminishing momentum, we show that MVRC-1 achieves a near-optimal sample complexity $\mathcal{O}(n+\sqrt{n}\epsilon^{-2})$ in solving the nonconvex composition problem $(\Sigma^1)$. We further propose a periodic restart scheme to facilitate the practical convergence of MVRC-1 and establish the same near-optimal sample complexity result. Then, under a constant momentum, we show that MVRC-1 also achieves the same near-optimal sample complexity result. Moreover, we establish  a linear convergence rate for MVRC-1 under the gradient dominance condition. With a slight modification of algorithm hyper-parameters, we show  that MVRC-1 also applies to the online version of the problem $(\Sigma^1)$ (referred to as $(\mathbb{E}^1)$) and achieves the state-of-the-art sample complexity in both nonconvex and gradient dominant scenarios under either diminishing or constant momentum. 

We further propose  the algorithm MVRC-2 that generalizes our momentum with variance reduction scheme to solve the more challenging composition problem $(\Sigma^2)$, which has finite-sum structure in both $f$ and $g$. With either diminishing or constant momentum and a normalized learning rate, we show that MVRC-2 achieves a near-optimal sample complexity $\mathcal{O}(N+n+\sqrt{\max\{N,n\}}\epsilon^{-2})$ in nonconvex composition optimization. Furthermore, in the corresponding online case (referred to as $(\mathbb{E}^2)$), we show that MVRC-2 also achieves the state-of-art sample complexity in nonconvex scenario. Please refer to \Cref{table_nonconvex_algs} for a comprehensive comparison between the sample complexities of our algorithm and those of existing composition optimization algorithms.

%accelerate the CIVR algorithm \cite{zhang2019stochastic} and its extension \cite{wang2019spiderboost} using momentum for all of the composition problems, including DE \eqref{doubleE}, DFS \eqref{doublesum}, SCE \eqref{singleCE} and SCFS \eqref{singleCsum}. Experiments on the investment example and Model-agnostic Meta learning \cite{finn2017model} (MAML) problem ?? demonstrate faster convergence of our accelerated algorithm, while the state of the art theoretical complexity is preserved. In addition, the momentum in Spider-boost \cite{wang2019spiderboost} as well as a very similar momentum in \cite{ghadimi2016accelerated} vanishes with iteration, and the momentum magnitude also vanishes in \cite{wang2017stochastic,wang2016accelerating} due to the fast diminishing learning rate. In contrast, we proved the same convergence rate also when both the momentum magnitude and learning rate are constant. Actually, our experiments also show that the constant momentum converges faster than the vanishing momentum, and the larger constant momentum converges faster than the smaller constant momentum??.\\ %To our knowledge, this is the first work that provides theoretical proof and experimental evidence for the fast generalized Nesterov's momentum of constant magnitude. \\
%\indent The complexity of previous non-convex optimization algorithms and convex ones are listed in Table \ref{table_nonconvex_algs} \& \ref{table_convex_algs}, respectively.

\begin{table*}
	\captionsetup{width=.75\textwidth}
	\caption{Comparison of sample complexities of nonconvex composition optimization algorithms. Note that $(\Sigma^1)$ corresponds to the problem $(\Sigma^2)$ with $N=1$. $(\mathbb{E}^1)$ and $(\mathbb{E}^2)$ correspond to the online versions of $(\Sigma^1)$ and $(\Sigma^2)$, respectively.}
	\label{table_nonconvex_algs}
	\vskip 0.1in
	\begin{center}
		\begin{scriptsize}
			%\begin{threeparttable}
				\begin{tabular}{lcccr}
					\toprule
					Problem & Algorithm & Assumption & Momentum & Sample complexity \\ 
					\midrule
					 & CIVR \cite{zhang2019stochastic} & $r$ convex & $\times$ & $\mathcal{O}(\epsilon^{-3})$ \\
					\cline{2-5}
					$(\mathbb{E}^1)$& {\color{blue} \textbf{Our work}} & $r$ convex & { \checkmark} & ${\mathcal{O}(\epsilon^{-3})}$ \\
					\cline{2-5}
					& CIVR \cite{zhang2019stochastic} & $F$ $v$-gradient dominant &   $\times$ &$\mathcal{O}(v\epsilon^{-1}\log\epsilon^{-1})$ \\
					%&   & $r\equiv0$ & & \\
					\cline{2-5}
					& {\color{blue} \textbf{Our work}} & $F$ $v$-gradient dominant &   \checkmark &$\mathcal{O}({v\epsilon^{-1}\log\epsilon^{-1}})$ \\
				%	&   & $r\equiv0$ & & \\
					\midrule
					&	SAGA \cite{zhang2019composite} &   $r$ convex  &  $\times$  &$\mathcal{O}(n+n^{2/3}\epsilon^{-2})$    \\
					\cline{2-5}
					& CIVR \cite{zhang2019stochastic} & $r$ convex & $\times$  & $\mathcal{O}(n+\sqrt{n}\epsilon^{-2})$  \\
					\cline{2-5}
					$(\Sigma^1)$ & {\color{blue} \textbf{Our work}} & $r$ convex & \checkmark & ${\mathcal{O}(n+\sqrt{n}\epsilon^{-2})}$ \\
					\cline{2-5}
					& SAGA \cite{zhang2019composite} & $F$ $v$-gradient dominant & $\times$ &$\mathcal{O}((n+\kappa n^{2/3})\log\epsilon^{-1})$ \\
					%&   & $r\equiv 0$ &  & \\
					\cline{2-5}
					& CIVR \cite{zhang2019stochastic} & $F$ $v$-gradient dominant & $\times$  & $\mathcal{O}((n+v\sqrt{n})\log\epsilon^{-1})$ \\
					%&  & $r\equiv0$ & & \\
					\cline{2-5}
					& {\color{blue} \textbf{Our work}} & $F$ $v$-gradient dominant &   \checkmark &$\mathcal{O}((n+v\sqrt{n})\log\epsilon^{-1})$ \\
					%&   & $r\equiv0$ & & \\
					\midrule
					& Basic SCGD \cite{wang2017stochastic} &	$r\equiv0$ & $\times$ & $\mathcal{O}(\epsilon^{-8})$   \\
					\cline{2-5}
					& Accelerated SCGD \cite{wang2017stochastic} &	$r\equiv0$ & \checkmark & $\mathcal{O}(\epsilon^{-7})$  \\
					\cline{2-5}
					& ASC-PG \cite{wang2016accelerating} & $r\equiv0$ &\checkmark & $\mathcal{O}(\epsilon^{-4.5})$  \\
					\cline{2-5}
					$(\mathbb{E}^2)$ & SARAH-Compositional\cite{yuan2019efficient} & $r\equiv0$ & $\times$ & $\mathcal{O}(\epsilon^{-3})$  \\
					\cline{2-5}
					& Nested-Spider \cite{zhang2019multi} & $r$ convex & $\times$ & $\mathcal{O}(\epsilon^{-3})$  \\
					\cline{2-5}
					& Spider+ADMM\cite{wang2019nonconvex} & $r$ convex & $\times$  & $\mathcal{O}(\epsilon^{-3})$  \\
					\cline{2-5}
					& {\color{blue} \textbf{Our work}} & $r$ convex & \checkmark  & ${\mathcal{O}(\epsilon^{-3})}$ \\
					\midrule
					& VRSC-PG \cite{huo2018accelerated} & $r$ convex & $\times$ &$\mathcal{O}(N+n+(N+n)^{2/3}\epsilon^{-2})$  \\
					\cline{2-5}
					& SAGA \cite{zhang2019composite} & $r$ convex & $\times$  & $\mathcal{O}(N+n+(N+n)^{2/3}\epsilon^{-2})$   \\
					\cline{2-5}
					$(\Sigma^2)$ & SARAH-Compositional\cite{yuan2019efficient} & $r\equiv0$ & $\times$ & {${\mathcal{O}(N+n+\sqrt{N+n}\epsilon^{-2})}$} \\%$\mathcal{O}(n+N+\sqrt{n+N}\epsilon^{-2})$ \\
					\cline{2-5}
					& Nested-Spider \cite{zhang2019multi} & $r$ convex & $\times$ & $\mathcal{O}(N+n+\sqrt{\max(N,n)}\epsilon^{-2})$   \\
					\cline{2-5}
					& Spider+ADMM\cite{wang2019nonconvex} & $r$ convex & $\times$  & {${\mathcal{O}(N+n+\sqrt{N+n}\epsilon^{-2})}$} \\%$\mathcal{O}(n+N+\sqrt{n+N}\epsilon^{-2})$  \\
					\cline{2-5}
					& {\color{blue} \textbf{Our work}} & $r$ convex & \checkmark  & ${\mathcal{O}(N+n+\sqrt{\max(N,n)}\epsilon^{-2})}$ \\
					\bottomrule
				\end{tabular}
		\end{scriptsize}
	\end{center}
	\vskip -0.1in
\end{table*}

\subsection{Related works} 
\textbf{Momentum \& variance reduction:} %The SGD algorithm was studied in \cite{ghadimi2013stochastic,ghadimi2016mini,bottou2018optimization,zhou2018new,ghadimi2016accelerated}, and \cite{ghadimi2016accelerated} accelerated SGD with momentum for the nonconvex case. 
Various variance reduction techniques have been originally developed for accelerating stochastic optimization without the composition structure, e.g., SAG \cite{roux2012stochastic}, SAGA \cite{defazio2014saga,reddi2016proximal}, SVRG \cite{johnson2013accelerating,allen2016variance,reddi2016stochastic,reddi2016proximal,li2018simple}, SCSG \cite{lei2017non}, SNVRG \cite{zhou2018stochastic}, SARAH \cite{nguyen2017sarah,nguyen2017stochastic,nguyen2019finite,pham2019proxsarah} and SPIDER \cite{fang2018spider,wang2019spiderboost}. In particular, the SPIDER scheme achieves a near-optimal sample complexity in nonconvex optimization. Momentum-accelerated versions of these algorithms have also been developed, e.g., momentum-SVRG \cite{li2017convergence}, Katyusha \cite{allen2017katyusha}, Natasha \cite{allen2017natasha,allen2018natasha} and momentum-SpiderBoost \cite{wang2019spiderboost}.  

\noindent\textbf{Stochastic composition optimization:} \cite{wang2017stochastic} developed the SCGD algorithm for stochastic composition optimization, and \cite{wang2016accelerating} further developed its momentum-accelerated version. Variance reduction techniques have been exploited to reduce the sample complexity of composition optimization, including the SVRG-based algorithms \cite{lian2016finite,huo2018accelerated}, SAGA-based algorithm \cite{zhang2019composite}, SPIDER-based algorithms \cite{zhang2019stochastic,zhang2019multi,wang2019nonconvex} and SARAH-based algorithm \cite{yuan2019efficient}. \cite{xu2019katyusha} further applied the Katyusha momentum to accelerate the SVRG-based composition optimization algorithm in convex optimization.

%\textbf{Stochastic algorithms for the DE problem \eqref{doubleE} and the DFS problem \eqref{doublesum}: }\\

\section{Momentum with SpiderBoost for Solving Nonconvex Problems $(\Sigma^1)$ and $(\mathbb{E}^1)$}
In this section, we develop momentum schemes with the Spider-Boost \cite{wang2019spiderboost} variance reduction technique for solving the nonconvex composition problems $(\Sigma^1)$ and $(\mathbb{E}^1)$, which are rewritten below for reference.
\begin{align}
	(\Sigma^1):~ & \min_{ x \in \mathbb{R}^{d}} \Phi(x)=f\Big(\frac{1}{n}\sum_{i=1}^{n} g_{i}(x)\Big)+r(x), \nonumber\\
	(\mathbb{E}^1):~ & \min_{ x \in \mathbb{R}^{d}} \Phi(x)=f\big(\mathbb{E}_{\xi} g_{\xi}(x)\big)+r(x). \nonumber
\end{align}

\subsection{Algorithm Design}
We present our algorithm design in \Cref{alg: 1} and refer to it as MVRC-1, which is short for \textbf{M}omentum with \textbf{V}ariance \textbf{R}eduction for \textbf{C}omposition optimization, and ``1'' stands for the target problems $(\Sigma^1), (\mathbb{E}^1)$.   
In \Cref{alg: 1}, we denote the proximal mapping of function $r$ as: for any $\lambda>0, x\in \mathbb{R}^d$,
\begin{align}
	\prox{\lambda r}(x) :=\argmin_{y\in \mathbb{R}^d} \Big\{r( y)+\frac{1}{2 \lambda} \| y- x\|^{2}\Big\}.
\end{align} 
To elaborate, \Cref{alg: 1} uses the SpiderBoost scheme proposed in \cite{wang2019spiderboost} to construct variance-reduced estimates $\widetilde{g}_t$ and $\widetilde{g}_t'$ of the mapping $g$ and its Jacobian matrix $g'$, respectively, and it further adopts a momentum scheme to facilitate the convergence. 

Our momentum scheme is simpler than the Katyusha momentum developed for composition optimization in \cite{xu2019katyusha}. In particular, our momentum scheme requires only one proximal mapping evaluation per-iteration to update $x_{t+1}$, whereas the Katyusha momentum requires two proximal mapping evaluations per-iteration to update both $x_{t+1}$ and $y_{t+1}$, respectively. Therefore, our momentum scheme saves much computation time when the proximal mapping of the regularizer induces much computation, e.g., nuclear norm regularization, group norm regularization, etc. 

%We first consider the single composition problems \eqref{singleCE} \& \eqref{singleCsum}, for which a state of the art algorithm called CIVR \cite{zhang2019stochastic} with nearly-optimal sample complexity has been recently developed by using the Spider algorithm \cite{fang2018spider} and its boosted version \cite{wang2019spiderboost}. We accelerate the practical speed of the CIVR (Spider) algorithm by adding momentum, as shown in Algorithm \ref{alg_single}.

\begin{algorithm}[h]
	\caption{(MVRC-1): Momentum with variance reduction for solving composition problems $(\Sigma^1), (\mathbb{E}^1)$}
	\label{alg: 1}
	{\bf Input:} $x_{0}\in \mathbb{R}^d$; $T,\tau \in \mathbb{N}$; $\lambda_{t}$, $\beta_{t}>0$ and $\alpha_{t}\in [0,1]$;
	
	{\bf Initialize:} $y_{0}= x_{0}$.
	
	\For{$t=0, 1, \ldots, T-1$}
	{	
		$z_{t}=\left(1-\alpha_{t+1}\right)  y_{t}+\alpha_{t+1}  x_{t}$,
		
		\eIf{$t\mod\tau=0$}
		{
			For $(\Sigma^1):$ Sample  set $\mathcal{A}_{t} = \{1,...,n\}$ \\
			For $(\mathbb{E}^1):$ Sample  set $\mathcal{A}_{t}$ from the distribution of $\xi$\\
			$\widetilde{g}_{t}=\frac{1}{|\mathcal{A}_t|}\sum_{\xi\in\mathcal{A}_{t}} g_{\xi}(z_t), \widetilde{g}_{t}'=\frac{1}{|\mathcal{A}_t|}\sum_{\xi\in\mathcal{A}_{t}} g_{\xi}'(z_t)$
		
		}
		{
			For $(\Sigma^1):$ Sample subset $\mathcal{A}_{t}$ from $\{1,...,n\}$ \\
			For $(\mathbb{E}^1):$ Sample set $\mathcal{A}_{t}$ from the distribution of $\xi$\\
			$\widetilde{g}_{t}= \widetilde{g}_{t-1}+\frac{1}{|\mathcal{A}_t|}\sum_{\xi\in\mathcal{A}_t}\big( g_{\xi}( z_t)- g_{\xi}( z_{t-1})\big),$\\
			$\widetilde{g}_{t}'= \widetilde{g}_{t-1}'+\frac{1}{|\mathcal{A}_t|}\sum_{\xi\in\mathcal{A}_t}\big( g_{\xi}'( z_t)- g_{\xi}'( z_{t-1})\big),$
		}
		Compute ${\widetilde{\nabla}F}( z_{t})= \widetilde{g}_{t}'^{\top} \nabla f(\widetilde{g}_t)$,\\
		$x_{t+1}=\operatorname{prox}_{\lambda_{t}r}(x_{t}-\lambda_{t}{\widetilde{\nabla}F}( z_{t}))$,\\
		$y_{t+1}= z_{t}+\frac{\beta_{t}}{\lambda_{t}}( x_{t+1}- x_{t})$
	}	
	{\textbf{Output:} $z_{\zeta}$, where $\zeta\stackrel{\text{Uniform}}{\sim}\{0,1,\ldots,T-1\}$.}
\end{algorithm}

\subsection{Convergence Analysis in Nonconvex Optimization}\label{subsec: single_nonconvex}
In this subsection, we study the convergence guarantees for MVRC-1 in nonconvex optimization under various choices of the momentum parameter $\alpha_t$%.\Ziyi{, $\beta_t$, $\lambda_t$ (because the theorems say ``momentum parameters $\alpha_t$, $\beta_t$, $\lambda_t$'')}. 

We make the following standard assumptions on the objective function.
\begin{assum}\label{assumption1}
	The objective functions in both the finite-sum problem $(\Sigma^1)$ and the online problem $(\mathbb{E}^1)$ satisfy
	\begin{enumerate}[topsep=0pt, noitemsep, leftmargin=*]
		\item Function $f$ is $l_f$-Lipschitz continuous, and its  gradient $\nabla f$ is $L_f$-Lipschitz continuous;
		\item Every mapping $g_\xi$ is $l_g$-Lipschitz continuous, and its Jacobian matrix $g_\xi'$ is $L_g$-Lipschitz continuous;
		\item Function $r$ is convex and $\Phi^*:= \inf_{x} \Phi(x) > -\infty$.
	\end{enumerate}
\end{assum}
In particular, the above assumption implies that the gradient of $F= f\circ g$ is Lipschitz continuous with parameter $	L_{F}=\ell_{g}^{2} L_{f}+\ell_{f} L_{g}.$
We also make the following assumption on the stochastic variance for the online case. 
\begin{assum}(For the online problem $(\mathbb{E}^1)$)\label{assumption2}
	The online composition optimization problem $(\mathbb{E}^1)$ satisfies that: there exists $\sigma_g, \sigma_{g'}>0$ such that for all $ x\in \mathbb{R}^d$, 
	\begin{equation}
		\mathbb{E}_{\xi} \big[\| g_{\xi}( x)- g( x)\|^2\big]\le\sigma_g^2,~\mathbb{E}_{\xi}\big[\| g_{\xi}'( x)- g'( x)\|^2\big]\le\sigma_{g'}^2. \nonumber
	\end{equation}
\end{assum}

\textbf{Evaluation metric:} To evaluate the convergence in nonconvex optimization, we define the generalized gradient mapping at $x$ with parameter $\lambda>0$  as 
\begin{align}\label{eq_prox_grad}
{\mathcal{G}}_{\lambda}(x):=\frac{1}{\lambda}\Big( x-\operatorname{prox}_{\lambda r}\big( x-\lambda \nabla{F}( x)\big)\Big).
\end{align}
In particular, when $r$ is convex, ${\mathcal{G}}_{\lambda}(x)=\zero$ if and only if  $x$ is a stationary point of the objective function $\Phi=F+r$. Therefore, we say a random point $x$ achieves $\epsilon$-accuracy if it satisfies $\mathbb{E}\|{\mathcal{G}}_{\lambda}(x)\| \le \epsilon$. 

%\subsection*{1) Analysis under diminishing momentum:}
We first study MVRC-1 with the choice of a diminishing momentum coefficient $\alpha_{t} = \frac{2}{t+1}$. We obtain the following complexity results. Throughout the paper, we define $G_0 = 2(l_{g}^{4} L_{f}^{2}+l_{f}^{2} L_{g}^{2})$ and use $\mathcal{O}(\cdot)$ to hide universal constants.

\begin{thm}[Diminishing momentum]\label{thm_alpha_decay_single}
	Let Assumptions \ref{assumption1} and \ref{assumption2} hold.  
	Apply MVRC-1 to solve the problems $(\Sigma^1)$ and $(\mathbb{E}^1)$ with momentum parameters $\alpha_{t}=\frac{2}{t+1}, \beta_t\equiv\beta\le (2\sqrt{16L_F^2+6G_0}+8L_F)^{-1}, \lambda_{t}\in [\beta, (1+\alpha_t)\beta]$. 
	\begin{itemize}[topsep=0pt, noitemsep, leftmargin=*]
		\item For the problem $(\Sigma^1)$, choose parameters $\tau=\lfloor{\sqrt n}\rfloor$, $|\mathcal{A}_t|= n$ whenever $t~\text{\rm mod}~\tau=0$, and $|\mathcal{A}_t|=\lfloor{\sqrt n}\rfloor$ otherwise. Then, the output satisfies
		\begin{equation}\label{eq_conclude_alpha_decay_SCFS_nonconvex}
		\mathbb{E}\|{\mathcal{G}}_{\lambda_\zeta}( z_\zeta)\|^2 \le \mathcal{O}\Big(\frac{\Phi\left(x_{0}\right)-\Phi^*}{T\beta} \Big).
		\end{equation}
		Moreover, to achieve an $\epsilon$-accurate solution, the required sample complexity (number of evaluations of $g, g'$) is $\mathcal{O}(n+\sqrt{n}\epsilon^{-2})$.
		
		\item For the problem $(\mathbb{E}^1)$, choose $\tau=\lfloor \sqrt{2\sigma_{0}^2\epsilon^{-2}} \rfloor$, $|\mathcal{A}_t|= \lceil 2\sigma_{0}^2\epsilon^{-2}\rceil$ whenever $t~\text{\rm mod}~\tau=0$ and $|\mathcal{A}_t|=\lfloor\sqrt{2\sigma_{0}^2\epsilon^{-2}} \rfloor$ otherwise. Then, the output satisfies
		\begin{equation}\label{eq_conclude_alpha_decay_SCE}
		\mathbb{E} \|{\mathcal{G}}_{\lambda_\zeta}( z_\zeta)\|^2\le \mathcal{O}\Big(\epsilon^2+\frac{\Phi\left(x_{0}\right)-\Phi^*}{T\beta}\Big).
		\end{equation}
		Moreover, to achieve an $\epsilon$-accurate solution, the required sample complexity is $\mathcal{O}(\epsilon^{-3})$.
	\end{itemize}
\end{thm}

Therefore, the MVRC-1 algorithm achieves a sublinear $\mathcal{O}(T^{-1})$ convergence rate in both finite-sum and online cases. In particular, the sample complexities of MVRC-1 match the state-of-art near-optimal sample complexities for nonconvex stochastic optimization. 

\textbf{Momentum with periodic restart:} 
We can further use a restart strategy to boost the practical convergence of MVRC-1 under the diminishing momentum scheme. To be specific, consider implementing MVRC-1 $M$ times and denote $\{x_{t,m}, y_{t,m}, z_{t,m}\}_{t=0}^{T-1}$ as the generated variable sequences in the $m$-th run. If we {adopt the initialization scheme $x_{0,m+1}=x_{T-1,m}$ for $m=1, \ldots, M-1$}, then it can be shown that {(see \Cref{subsec: restart1} for the detailed proof)}
\begin{align}\label{eq_conclude_alpha_decay_single_restart}
\text{For $(\Sigma^1)$:}~ &\mathbb{E} \|{\mathcal{G}}_{\lambda_{\zeta}}( z_{\zeta,\delta})\|^2\le \mathcal{O}\Big( \frac{\Phi\left(x_{0,1}\right)-\Phi^*}{MT\beta}\Big), \nonumber\\
\text{For $(\mathbb{E}^1)$:}~ &\mathbb{E} \|{\mathcal{G}}_{\lambda_{\zeta}}( z_{\zeta,\delta})\|^2\le \mathcal{O} \Big(\epsilon^2+\frac{\Phi\left(x_{0,1}\right)-\Phi^*}{MT\beta} \Big), \nonumber
\end{align}
\noindent where $\zeta$ is uniformly sampled from $\{0,\ldots,T-1\}$ and $\delta$ is uniformly sampled from $\{1,\ldots,M\}$. In particular, by choosing $M=T=\mathcal{O}(\epsilon^{-1})$, we can {achieve an $\epsilon$-accurate solution with} the same sample complexities as those specified in \Cref{thm_alpha_decay_single}.

%\begin{thm}\label{thm_alpha_decay_SCE_nonconvex}
%	\indent When implementing algorithm \ref{alg_single} $M$ times  where $x_{t,m}$, $y_{t,m}$, $z_{t,m}$ denote $x_t$, $y_t$, $z_t$ at the $m$-th implementation respectively, and $x_{T-1,m}=y_{T-1,m}=x_{0,m+1}$ for $m=1, \ldots, M-1$ (this is called the restart strategy), then
%		\begin{equation}\label{eq_conclude_alpha_decay_SCE_nonconvex_restart}
%		\mathbb{E}_{\xi,\delta}\|{\mathcal{G}}_{\lambda_{\xi}}( z_{\xi,\delta})\|^2\le 91\epsilon^2+\frac{176}{MT\beta}\left[\Phi\left(x_{0,1}\right)-\Phi^*\right],
%	\end{equation}
%	\noindent where $\xi\mathop\sim\limits^{unif}\{0,\ldots,T-1\}$ and $\delta\mathop\sim\limits^{unif}\{1,\ldots,M\}$. Letting $M=\mathcal{O}(\epsilon^{-1})$ and $T=\mathcal{O}(\epsilon^{-1})$ guarantees $\mathbb{E}_{\xi,\delta}\|{\mathcal{G}}_{\lambda_\xi}( z_{\xi,\delta})\|^2\le C^2\epsilon^2$ for some constant $C>0$, and leads to the same computation complexity as above.\\
%	\indent In both cases above, to achieve $\mathbb{E}_{\xi,\delta}\|{\mathcal{G}}_{\lambda_\xi}( z_{\xi,\delta})\|^2\le \epsilon^2$, just replace $\epsilon$ with $\epsilon/C$ in the hyperparameter choice \eqref{hyperpar_alpha_decay_SCE_nonconvex}, which does not change the computation complexity, nor the magnitude of $T$ and $M$.\\
%\end{thm}

%\subsection*{2) Analysis under constant momentum:}
The momentum scheme of MVRC-1 also allows to adopt a more aggressive constant-level momentum coefficient, i.e., $\alpha_t\equiv \alpha$ for any $\alpha\in (0,1]$. We obtain the following convergence results.
\begin{thm}[Constant momentum]\label{thm_alpha_const_single}
Let Assumptions \ref{assumption1} and \ref{assumption2} hold.  
Apply MVRC-1 to solve the problems $(\Sigma^1)$ and $(\mathbb{E}^1)$ with {momentum} parameters $\alpha_{t}\equiv\alpha\in (0,1], \beta_t\equiv\beta\le (4{\sqrt{{{(1 + \alpha^{-1} )}^2} L_F^2+3{G_0}} + 4(1 + \alpha^{-1} ){L_F}})^{-1}, \lambda_{t}\in [\beta, (1+\alpha)\beta]$. 
\begin{itemize}[topsep=0pt, noitemsep, leftmargin=*]
	\item For the  problem $(\Sigma^1)$, choose the same $\tau$ and $|\mathcal{A}_t|$ as those in item 1 of \Cref{thm_alpha_decay_single}. Then, the output satisfies
	\begin{equation}\label{eq_conclude_alpha_const_SCFS}
	\mathbb{E}\|{\mathcal{G}}_{\lambda_\zeta}( z_\zeta)\|^2 \le \mathcal{O}\Big(\frac{\alpha^{-1}+ 1}{T\beta}\big(\Phi\left(x_{0}\right)-\Phi^*\big)\Big).
	\end{equation}
	Moreover, to achieve an $\epsilon$-accurate solution, the required sample complexity is $\mathcal{O}(n+\sqrt{n}\epsilon^{-2})$.
	
	\item For the problem $(\mathbb{E}^1)$, choose the same  $\tau$ and $|\mathcal{A}_t|$ as those in item 2 of \Cref{thm_alpha_decay_single}. Then, the output  satisfies
	\begin{equation}\label{eq_conclude_alpha_const_SCE}
	\mathbb{E} \|{\mathcal{G}}_{\lambda_\zeta}( z_\zeta)\|^2\le \mathcal{O}\Big(\!(\alpha^{-1}+1)\Big(\epsilon^2\!+\!\frac{\Phi(x_{0})-\Phi^*}{T\beta} \Big)\!\Big). 
	\end{equation}
	Moreover, to achieve an $\epsilon$-accurate solution, the required sample complexity is $\mathcal{O}(\epsilon^{-3})$.
\end{itemize}
\end{thm}

Hence, under the constant momentum, MVRC-1 maintains the near-optimal sample complexities in solving both the finite-sum problem $(\Sigma^1)$ and the online problem $(\mathbb{E}^1)$. We note that under the constant momentum, the proof technique of the previous \Cref{thm_alpha_decay_single} (under diminishing momentum) does not apply, and the proof of \Cref{thm_alpha_const_single} requires novel developments on bounding the variable sequences. To elaborate, we need to develop \Cref{lemma_seq_bound_alpha_const} in \Cref{sec: append: 2} to bound the difference sequences $\{\|y_t -x_t\|,\|z_{t+1} - z_t\| \}_t$ based on the constant momentum scheme. Moreover, \cref{eq_sum_s_bound_alpha_const} in \Cref{sec: append: 2} is developed to bound the accumulated constant momentum parameter involved in the series $\sum_{t=0}^{T-1} \|y_t - x_t\|^2$. These developments are critical for obtaining the convergence guarantees and sample complexities of MVRC-1 under constant momentum. 

\section{Analysis of MVRC-1 under Nonconvex Gradient-dominant Condition}
In this subsection, we study the convergence guarantees of MVRC-1 in solving unregularized nonconvex composition problems that satisfy the following gradient dominant condition. Throughout, we denote $F^*:= \inf_{x\in \mathbb{R}^d} F(x)$.
\begin{assum}[Gradient dominant]\label{assumption_graddom}
	Consider the unregularized problem (P) with $r( x)\equiv 0$. Function $F$ is called gradient dominant with parameter $v>0$ if $~\forall  x \in \mathbb{R}^{d},$
	\begin{equation}\label{eq_graddom}
		F(x)- F^* \leq \frac{v}{2}\left\|\nabla F( x)\right\|^{2}.
	\end{equation}
\end{assum}
%The gradient dominant condition is satisfied by various nonconvex machine learning problems, e.g., deep neural networks \cite{allen-zhu19a}, phase retrieval \cite{Zhou2016}, blind deconvolution \cite{LI2019}, etc. Moreover, compositional optimization problems such as  also satisfy the gradient dominant condition.
The gradient dominant condition is a relaxation of strong convexity and is satisfied by many nonconvex machine learning models.
Next, under the gradient dominant condition, we show that MVRC-1 achieves a linear convergence rate in solving the composition problems $(\Sigma^1)$ and $(\mathbb{E}^1)$. We first consider the case of diminishing momentum.

\begin{thm}[Diminishing momentum]\label{thm_graddom_singleC}
	Let Assumptions \ref{assumption1}, \ref{assumption2} and \ref{assumption_graddom} hold.  
	Apply MVRC-1 to solve the problems $(\Sigma^1)$ and $(\mathbb{E}^1)$ with momentum parameters $\alpha_{t}=\frac{2}{t+1}, \beta_t\equiv\beta\le (2\sqrt{16L_F^2+6G_0}+8L_F)^{-1}, \lambda_{t}\in [\beta, (1+\alpha_t)\beta]$. 
	\begin{itemize}[topsep=0pt, noitemsep, leftmargin=*]
		\item For the  problem $(\Sigma^1)$, choose the same   $\tau$ and $|\mathcal{A}_t|$ as those in item 1 of \Cref{thm_alpha_decay_single}. Then, the output  satisfies
		\begin{equation}\label{eq_conclude_alpha_decay_SCFS_graddom}
		\mathbb{E} F({z_\zeta }) - {F^*} \leq \mathcal{O} \Big(\frac{v(F\left( {{x_0}} \right) - {F^*})}{T\beta }\Big).
		\end{equation}
		
		\item For the  problem $(\mathbb{E}^1)$, choose the same  $\tau$ and $|\mathcal{A}_t|$ as those in item 2 of \Cref{thm_alpha_decay_single}. Then, the output  satisfies
		\begin{equation}\label{eq_conclude_alpha_decay_SCE_graddom}
		\mathbb{E} F({z_\zeta}) - {F^*} \leq \mathcal{O} \Big(v{\epsilon ^2} + \frac{v(F\left( {{x_0}} \right) - {F^*})}{T\beta }\Big).
		\end{equation}
	\end{itemize}
	{Moreover, t}o achieve an $\epsilon$-accurate solution, we restart MVRC-1 $M$ times with $y_{0,m+1}=x_{0,m+1}$ being randomly selected from $\{z_{t,m}\}_{t=0}^{T-1}$ and choose $M=\mathcal{O}(\log \frac{1}{\epsilon}), T=\max\{\mathcal{O}(v),\tau\}$. Then,
	\begin{itemize}[topsep=0pt, noitemsep, leftmargin=*]
		\item For the problem $(\Sigma^1)$, the required sample complexity is $\mathcal{O}((n+\sqrt{n}v)\log\frac{1}{\epsilon})$.
		\item For the problem $(\mathbb{E}^1)$, the required sample complexity is $\mathcal{O}(v\epsilon^{-2}\log\frac{1}{\epsilon})$.
	\end{itemize}

\end{thm}
Therefore, MVRC-1 achieves a linear convergence rate in solving both $(\Sigma^1)$ and $(\mathbb{E}^1)$ under the gradient dominant condition, and the corresponding sample complexities match the best-known existing results. 
Furthermore, our algorithm also allows to adopt a constant momentum scheme under the gradient dominant condition and preserves the convergence guarantee as well as the sample complexity. We obtain the following result.

\begin{thm}[Constant momentum]\label{thm_graddom_const_momentum}
	Let Assumptions \ref{assumption1}, \ref{assumption2} and \ref{assumption_graddom} hold.  
	Apply MVRC-1 to solve the problems $(\Sigma^1)$ and $(\mathbb{E}^1)$ with  parameters $\alpha_{t}\equiv\alpha\in (0,1], \beta_t\equiv\beta\le (4{\sqrt {{{(1 + 1/\alpha )}^2} L_F^2+3{G_0}} + 4(1 + 1/\alpha ){L_F}})^{-1}, \lambda_{t}\in [\beta, (1+\alpha_t)\beta]$. 
	\begin{itemize}[topsep=0pt, noitemsep, leftmargin=*]
		\item For the problem $(\Sigma^1)$, choose the same $\tau$ and $|\mathcal{A}_t|$ as those in item 1 of {\Cref{thm_alpha_decay_single}}. Then, the output satisfies
		\begin{equation}\label{eq_conclude_alpha_const_SCFS_graddom}
		{\mathbb{E} }F({z_\zeta }) - {F^*} \le \mathcal{O} \Big(\frac{v(\alpha^{-1}+1)}{T\beta}\big(F\left(x_{0}\right)-F^*\big) \Big).
		\end{equation}
		
		\item For the problem $(\mathbb{E}^1)$, choose the same $\tau$ and $|\mathcal{A}_t|$ as those in item 2 of {\Cref{thm_alpha_decay_single}}. Then, the output satisfies
		\begin{align*}
		\mathbb{E} F({z_\zeta }) - {F^*} \le \mathcal{O} \Big(\!v(\alpha^{-1}\!+\!1) \Big(\epsilon^2  \!+\!\frac{F\left(x_{0}\right)\!-\!F^*}{T\beta}\Big)\!\Big).
		\end{align*}
	\end{itemize}% $y_{0,m+1}=x_{0,m+1}$ being randomly selected from $\{z_{t,m}\}_{t=0}^{T-1}$
	To achieve an $\epsilon$-accurate solution, we restart MVRC-1 $M$ times with $y_{0,m+1}=x_{0,m+1}$ being randomly selected from $\{z_{t,m}\}_{t=0}^{T-1}$ and choose $M=\mathcal{O}(\log \frac{1}{\epsilon}), T=\max\{\mathcal{O}(v),\tau\}$. Then,
	\begin{itemize}[topsep=0pt, noitemsep, leftmargin=*]
		\item For the problem $(\Sigma^1)$, the required sample complexity is $\mathcal{O}((n+\sqrt{n}v)\log\frac{1}{\epsilon})$.
		\item For the problem $(\mathbb{E}^1)$, the required sample complexity is $\mathcal{O}(v\epsilon^{-2}\log\frac{1}{\epsilon})$.
	\end{itemize}
\end{thm}

%{\color{red} This seems to be a better complexity result?} Furthermore, $\mathbb{E}_{\xi} F( z_{\xi, M})-F^*\le \epsilon$ can also be achieved by replacing $\epsilon$ with $\sqrt{\epsilon/C}$ in the hyperparmeter choices. Then $M=\mathcal{O}(\log\epsilon^{-1})$, $T=\max[\mathcal{O}(v),\tau]$ in all the four cases, while $T\le\mathcal{O}(v\epsilon^{-1/2})$ in cases 1 $\&$ 3, and $T\le\mathcal{O}(v+\sqrt{n})$ in cases 2 $\&$ 4. Therefore, the number of evaluations for function $\nabla f$(also equals the number of solving proximal gradient problem \eqref{eq_updatex}) is at most $\mathcal{O}(v\epsilon^{-1/2}\log\epsilon^{-1})$ in cases 1 $\&$ 3, and at most $\mathcal{O}((\sqrt{n}+v)\log\epsilon^{-1})$ in cases 2 $\&$ 4. The number of  evaluations for functions $ g$ and $ g'$ are also both at most $\mathcal{O}(v\epsilon^{-1}\log\epsilon^{-1})$ in cases 1 $\&$ 3, and at most $\mathcal{O}((n+\sqrt{n}v)\log\epsilon^{-1})$ in cases 2 $\&$ 4.
\section{Momentum with SPIDER for Solving Nonconvex Problems $(\Sigma^2)$ and $(\mathbb{E}^2)$}
In this section, we develop momentum schemes with variance reduction for solving the composition optimization problems $(\Sigma^2)$ and $(\mathbb{E}^2)$ that have double finite-sum and double expectation structures, respectively, which are rewritten below for reference.
\begin{align}
(\Sigma^2):~ & \min_{ x \in \mathbb{R}^{d}} \Phi(x)= \frac{1}{N}\sum_{k=1}^{N}f_k\Big(\frac{1}{n}\sum_{i=1}^{n} g_{i}(x)\Big)+r(x), \nonumber\\
(\mathbb{E}^2):~ & \min_{ x \in \mathbb{R}^{d}} \Phi(x)= \mathbb{E}_\eta f_\eta \big(\mathbb{E}_{\xi} g_{\xi}(x)\big)+r(x). \nonumber
\end{align}

\subsection{Algorithm Design}
The details of the algorithm design are presented in \Cref{alg_double}, which is referred to as MVRC-2.
We note that the MVRC-2 for solving the composition problems $(\Sigma^2), (\mathbb{E}^2)$ are different from the MVRC-1 for solving the simpler problems $(\Sigma^1), (\mathbb{E}^1)$ in several aspects. To elaborate, first, in order to handle the double finite-sum and double expectation structure of  $(\Sigma^2)$ and $(\mathbb{E}^2)$, MVRC-2 requires to sample both the mapping $g$ and the function $f$. In particular, the sampling of $g$ is independent from that of $g'$, which is different from MVRC-1 where they share the same set of samples. Second, MVRC-2 adopts a SPIDER-like variance reduction scheme that uses {an} accuracy-dependent stepsize $\theta_{t}$, whereas MVRC-1 uses the SpiderBoost variance reduction scheme that adopts a constant stepsize. As we present later, such a conservative stepsize leads to theoretical convergence guarantees for MVRC-2 in solving the more challenging problems $(\Sigma^2), (\mathbb{E}^2)$ and help achieve a near-optimal sample complexity result. 

\begin{algorithm}[h]
	\caption{(MVRC-2): Momentum with variance reduction for solving composition problems $(\Sigma^2), (\mathbb{E}^2)$}
	\label{alg_double}
	{\bf Input:} $x_{0}\in\mathbb{R}^d$; $T,\tau \in \mathbb{N}$; $\epsilon, \lambda_{t}$, $\beta_{t}>0$ and $\alpha_{t}\in [0,1]$;
	
	{\bf Initialize:} $y_{0}= x_{0}$.
	
	\For{$t=0, 1, \ldots, T-1$}
	{	
		$z_{t}=\left(1-\alpha_{t+1}\right)  y_{t}+\alpha_{t+1}  x_{t}$,
		
		\eIf{$t\mod\tau=0$}
		{
			For $(\Sigma^2):$ Sample sets $\mathcal{A}_{t}, \mathcal{A}_{t}' = \{1,...,n\}$ and $\mathcal{B}_{t} = \{1,..., N\}$ \\
			For $(\mathbb{E}^2):$ Sample sets $\mathcal{A}_{t}, \mathcal{A}_{t}'$ from the distribution of $\xi$ and {sample set} $\mathcal{B}_{t}$ from the distribution of $\eta$\\
			$\widetilde{g}_{t}=\frac{1}{|\mathcal{A}_t|}\sum_{\xi\in\mathcal{A}_{t}} g_{\xi}(z_t),$ $\widetilde{g}_{t}'=\frac{1}{|\mathcal{A}_t'|}\sum_{\xi\in\mathcal{A}_{t}'} g_{\xi}'(z_t),$\\
			$\widetilde{f}_{t}'=\frac{1}{|\mathcal{B}_t|}\sum_{\eta\in\mathcal{B}_{t}} \nabla f_{\eta}(\widetilde{g}_t).$
			
		}
		{
			For $(\Sigma^2):$ Sample subsets $\mathcal{A}_{t},\mathcal{A}_{t}'$ from $\{1,...,n\}$ and $\mathcal{B}_{t}$ from $\{1,..., N\}$ \\
			For $(\mathbb{E}^2):$ Sample sets $\mathcal{A}_{t},\mathcal{A}_{t}'$ from the distribution of $\xi$ and {sample set} $\mathcal{B}_{t}$ from the distribution of $\eta$\\
			$\widetilde{g}_{t}= \widetilde{g}_{t-1}+\frac{1}{|\mathcal{A}_t|}\sum_{\xi\in\mathcal{A}_t}\big( g_{\xi}( z_t)- g_{\xi}( z_{t-1})\big),$\\
			$\widetilde{g}_{t}'= \widetilde{g}_{t-1}'+\frac{1}{|\mathcal{A}_t'|}\sum_{\xi\in\mathcal{A}_t'}\big( g_{\xi}'( z_t)- g_{\xi}'( z_{t-1})\big),$\\
			$\widetilde{f}_{t}'=\widetilde{f}_{t-1}'+\frac{1}{|\mathcal{B}_t|}\sum_{\eta\in\mathcal{B}_t}\big(\nabla f_{\eta}(\widetilde{g}_t)-\nabla f_{\eta}(\widetilde{g}_{t-1})\big)$
		}
		Compute ${\widetilde{\nabla}F}( z_{t})= \widetilde{g}_{t}'^{\top} \widetilde{f}_{t}'$,\\
		$\widetilde{x}_{t+1}=\operatorname{prox}_{\lambda_{t}r}(x_{t}-\lambda_{t}{\widetilde{\nabla}F}( z_{t}))$,\\
		$x_{t+1}=(1-\theta_t)x_t+\theta_{t} \widetilde{x}_{t+1},
		\theta_t=\min\big\{\frac{\epsilon\lambda_{t}}{\|\widetilde{x}_{t+1}-x_t\|},\frac{1}{2}\big\}$,\\
		$y_{t+1}= z_{t}+\frac{\beta_{t}}{\lambda_{t}}( x_{t+1}- x_{t})$
	}	
	{\textbf{Output:} $z_{\zeta}$, where $\zeta\stackrel{\text{Uniform}}{\sim}\{0,1,\ldots,T-1\}$.}
\end{algorithm}

\subsection{Convergence Analysis in Nonconvex Optimization}

We adopt the following standard assumptions from \cite{zhang2019multi} regarding the problems $(\Sigma^2), (\mathbb{E}^2)$.

\begin{assum}\label{assumption4}
	The objective functions in both the finite-sum problem $(\Sigma^2)$ and online problem $(\mathbb{E}^2)$ satisfy
	\begin{enumerate}[topsep=0pt, noitemsep, leftmargin=*]
		\item Every function $f_\eta$ is $l_f$-Lipschitz continuous, and its gradient $\nabla f_\eta$ is $L_f$-Lipschitz continuous;
		\item Every mapping $g_\xi$ is $l_g$-Lipschitz continuous, and its Jacobian matrix $g_\xi'$ is $L_g$-Lipschitz continuous;
		\item Function $r$ is convex and $\Phi^*:= \inf_{x} \Phi(x) > -\infty$.
	\end{enumerate}
\end{assum}

\begin{assum}(For the online problem $(\mathbb{E}^2)$)\label{assumption5}
	The online  problem $(\mathbb{E}^2)$ satisfies: there exists $\sigma_g, \sigma_{g'}, \sigma_{f'}>0$ such that for all {$x\in \mathbb{R}^p$},$y\in \mathbb{R}^d$, 
	\begin{align}
		&\mathbb{E}_{\xi} \big[\| g_{\xi}( x)- g( x)\|^2\big]\le\sigma_g^2,\nonumber\\
		&\mathbb{E}_{\xi}\big[\| g_{\xi}'( x)- g'( x)\|^2\big]\le\sigma_{g'}^2,\nonumber\\
		&\mathbb{E}_{\eta}\big[\| \nabla f_{\eta}( y)- \nabla f( y)\|^2\big]\le\sigma_{f'}^2. \nonumber
	\end{align}
\end{assum}

%\begin{assum}$\footnote{\cite{zhang2019multi} ignores this assumption but it is necessary to ensure that $\tau/b_1$ and $\tau/b_2$ have constant upper bound in the proof of Theorem 4.4 in \cite{zhang2019multi}}$\label{assumption_DFS_Nn_bound}
%	For the DFS problem \eqref{doublesum}, there exist constants $C_1, C_2>0$ such that $N\le C_1^2n^2$, $n\le C_2^2N^2$. 
%\end{assum}
For the problem $(\Sigma^2)$, we also adopt a mild assumption on the sample sizes that requires $n\le \mathcal{O}(N^2)$ and $N\le \mathcal{O}(n^2)$. We note that such a condition is also implicitly used by the proof of Theorem 4.4 in \cite{zhang2019multi}. Also, we adopt the same evaluation metric $\mathbb{E}\|{\mathcal{G}}_{\lambda}(x)\| \le \epsilon$ as that used in the previous section. We obtain the following results regarding MVRC-2 with diminishing momentum.

%\subsection*{1) Analysis under diminishing momentum:}

%\indent The convergence analysis for double composition problems is similar to that of the single composition problems, with the main difference that the outside function $f$ is also an expectation or finite sum form and thus approximated by the Spider estimator. Specifically, when the momentum magnitude $\alpha_{t}=\frac{2}{t+1}$ is decaying, the following Theorems \ref{thm_alpha_decay_DE_nonconvex} \& \ref{thm_alpha_decay_DFS_nonconvex} can be proved for the DE problem \eqref{doubleE} and the DFS problem \eqref{doublesum} respectively, similar to Theorems \ref{thm_alpha_decay_SCE_nonconvex} \& \ref{thm_alpha_decay_SCFS_nonconvex}; When the momentum magnitude $\alpha_{t}\equiv\alpha$ is constant, the following Theorems \eqref{thm_alpha_const_DE_nonconvex} \& \eqref{thm_alpha_const_DFS_nonconvex} can be proved for the DE problem \eqref{doubleE} and the DFS problem \eqref{doublesum} respectively, similar to Theorems \ref{thm_alpha_const_SCE_nonconvex} \& \ref{thm_alpha_const_SCFS_nonconvex}; When the smooth part $\Phi$ is gradient dominant and $r\equiv 0$, the following Theorem \ref{thm_graddom_double} is similar to Theorem \ref{thm_graddom_singleC}. The proof is shown in the Appendix.

\begin{thm}[Diminishing momentum]\label{thm_alpha_decay_DE_nonconvex}
	Let Assumptions \ref{assumption4} and \ref{assumption5} hold.  
	Apply MVRC-2 to solve the problems $(\Sigma^2)$ and $(\mathbb{E}^2)$ with momentum parameters $\alpha_{t}=\frac{2}{t+1}, \beta_t\equiv\beta=\mathcal{O}(L_F^{-1}), \lambda_{t}\in [\beta, (1+\alpha_t)\beta]$. 
	\begin{itemize}[topsep=0pt, noitemsep, leftmargin=*]
		\item For {the problem} $(\Sigma^2)$, choose $\tau=\lfloor{\sqrt{\max\{N,n\}}}\rfloor$. Set $|\mathcal{A}_t|,|\mathcal{A}_t'|,|\mathcal{B}_t|= n,n,N$, respectively, whenever $t~\text{\rm mod}~\tau=0$ and otherwise set {them to be $\mathcal{O}(\tau)$.} Then, for any {$\epsilon< l_fl_g (\max(N,n))^{-\frac{1}{2}}$}, the output satisfies
		\begin{align}\label{eq_conclude_alpha_decay_DFS}
		\mathbb{E} \|{\mathcal{G}}_{\lambda_\zeta}( z_\zeta)\|\le \mathcal{O}\Big(\epsilon+\frac{\Phi(x_0)-\Phi^*}{\epsilon T}\Big),
		\end{align}
		Moreover, to achieve an $\epsilon$-accurate solution, the required  sample complexity (number of evaluations of $g, g',\nabla f$) is $\mathcal{O}(N+n+\sqrt{\max\{N,n\}}\epsilon^{-2})$.
		
		\item For $(\mathbb{E}^2)$, choose  $\tau=\lfloor {l_fl_g\epsilon^{-1}} \rfloor$. Set $|\mathcal{A}_t|,|\mathcal{A}_t'|,|\mathcal{B}_t|=$ $\mathcal{O}(\lceil L_f^2l_g^2\sigma_g^2\epsilon^{-2}\rceil), \mathcal{O}(\lceil l_f^2\sigma_{g'}^2\epsilon^{-2}\rceil), \mathcal{O}(\lceil l_g^2\sigma_{f'}^2\epsilon^{-2}\rceil)$, respectively,  whenever $t~\text{\rm mod}~\tau=0$ and otherwise set them to be $\mathcal{O}(\lfloor l_fl_g\epsilon^{-1}\rfloor)$. Then, the output  satisfies
		\begin{equation}\label{eq_conclude_alpha_decay_DE}
		\mathbb{E} \|{\mathcal{G}}_{\lambda_\zeta}( z_\zeta)\|\le \mathcal{O}\Big(\epsilon+\frac{\Phi\left(x_{0}\right)-\Phi^*}{\epsilon T}\Big).
		\end{equation}
		Moreover, to achieve an $\epsilon$-accurate solution, the required sample complexity is $\mathcal{O}(\epsilon^{-3})$.
	\end{itemize}
\end{thm}
Therefore, MVRC-2 achieves near-optimal sample complexities in solving the nonconvex composition problems  $(\Sigma^2)$ and $(\mathbb{E}^2)$. Furthermore, under the diminishing momentum scheme, \Cref{alg_double} can implement the same momentum restart scheme as that developed for \Cref{alg: 1} in \Cref{subsec: single_nonconvex} to facilitate the practical convergence and maintain the same complexity results in nonconvex optimization. Due to space limitation, we present these results in \Cref{sec_Mdecay_thm_double}.

%\textbf{2) Constant momentum scheme:}

Next, we establish convergence guarantee for MVRC-2 under a constant momentum scheme in nonconvex optimization. We obtain the following result.
\begin{thm}[Constant momentum]\label{thm_alpha_const_DE_nonconvex}
	Let Assumptions \ref{assumption4} and \ref{assumption5} hold.  
	Apply  MVRC-2 to solve the problems $(\Sigma^2)$ and $(\mathbb{E}^2)$ with momentum parameters $\alpha_{t}\equiv\alpha \in (0,1], \beta_t\equiv\beta=\mathcal{O}(L_F^{-1}), \lambda_{t}\in [\beta, (1+\alpha_t)\beta]$. 
	\begin{itemize}[topsep=0pt, noitemsep, leftmargin=*]
		\item For the problem $(\Sigma^2)$, {choose the same $\tau$, $|\mathcal{A}_t|$,$|\mathcal{A}_t'|$ and $|\mathcal{B}_t|$ as those in item 1 of \Cref{thm_alpha_decay_DE_nonconvex}.} Then, for any {$\epsilon<l_fl_g {(\max(N,n))^{-1/2}}$}, the output satisfies
		\begin{align}\label{eq_conclude_alpha_const_DFS}
		\mathbb{E} \|{\mathcal{G}}_{\lambda_\zeta}( z_\zeta)\|\le \mathcal{O}\Big(\epsilon\alpha^{-1}+\frac{\Phi(x_0)-\Phi^*}{\epsilon T}\Big),
		\end{align}
		Moreover, to achieve an $\epsilon$-accurate solution, the required sample complexity  is $\mathcal{O}(N+n+\sqrt{\max\{N,n\}}\epsilon^{-2})$.
		
		\item For the  problem $(\mathbb{E}^2)$, {choose the same $\tau$, $|\mathcal{A}_t|$,$|\mathcal{A}_t'|$ and $|\mathcal{B}_t|$ as those in item 2 of \Cref{thm_alpha_decay_DE_nonconvex}.} Then, the output  satisfies
		\begin{equation}\label{eq_conclude_alpha_const_DE}
		\mathbb{E} \|{\mathcal{G}}_{\lambda_\zeta}( z_\zeta)\|\le \mathcal{O}\Big(\epsilon\alpha^{-1}+\frac{\Phi\left(x_{0}\right)-\Phi^*}{\epsilon T}\Big).
		\end{equation}
		Moreover, to achieve an $\epsilon$-accurate solution, the required sample complexity is $\mathcal{O}(\epsilon^{-3})$.
	\end{itemize} 
\end{thm}
To summarize, under either the diminishing momentum or the constant momentum, our MVRC-2 has guaranteed convergence in solving {both} the nonconvex problems $(\Sigma^2)$ and $(\mathbb{E}^2)$ with near-optimal sample complexities. Therefore, in practical scenarios, we expect that the momentum scheme can significantly facilitate the convergence of the algorithm, which is further verified in the next section via numerical experiments.

%\begin{thm}\label{thm_graddom_double}
%\end{thm}

\section{Experiments}
In this section, we compare the practical performance of our MVRC with that of other advanced stochastic composition optimization algorithms via two experiments: risk-averse portfolio optimization and nonconvex sparse additive model. The algorithms that we consider include VRSC-PG \cite{huo2018accelerated}, ASC-PG \cite{wang2016accelerating}, CIVR \cite{zhang2019stochastic}, and Katyusha \cite{xu2019katyusha}, which adopt either variance reduction or momentum in their algorithm design. 

\begin{figure*}[h]
	\begin{minipage}{.33\textwidth}
		\centering
		\includegraphics[width=\textwidth]{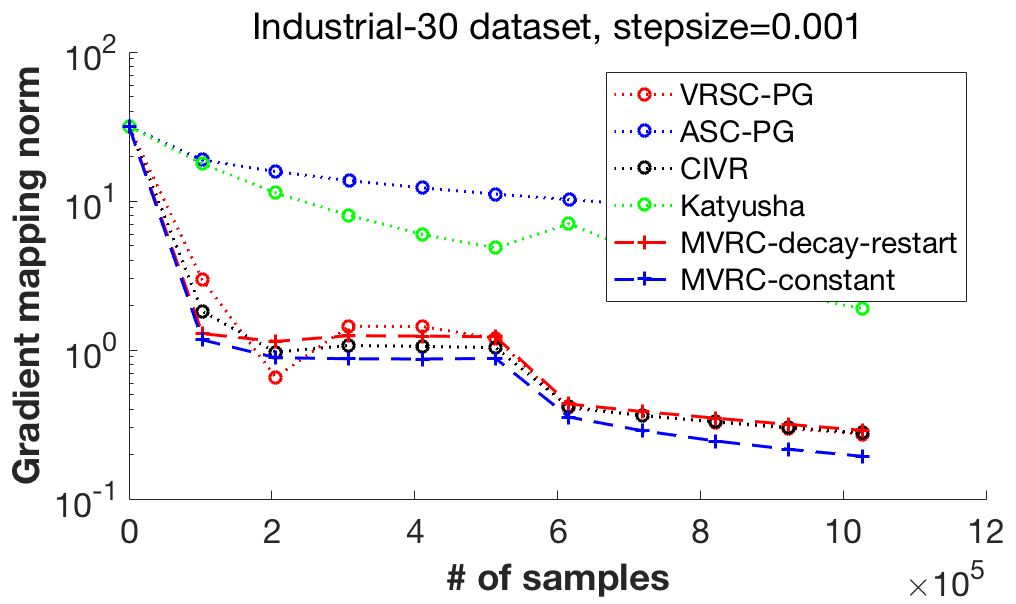}	
		\includegraphics[width=\textwidth]{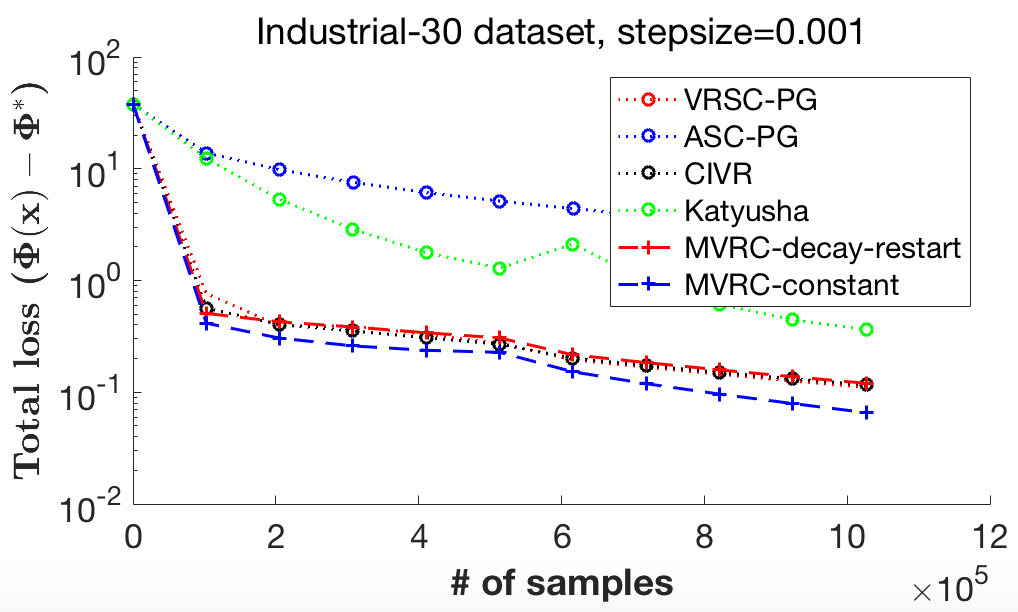}
	\end{minipage} 
	\begin{minipage}{.33\textwidth}
		\centering
		\includegraphics[width=\textwidth]{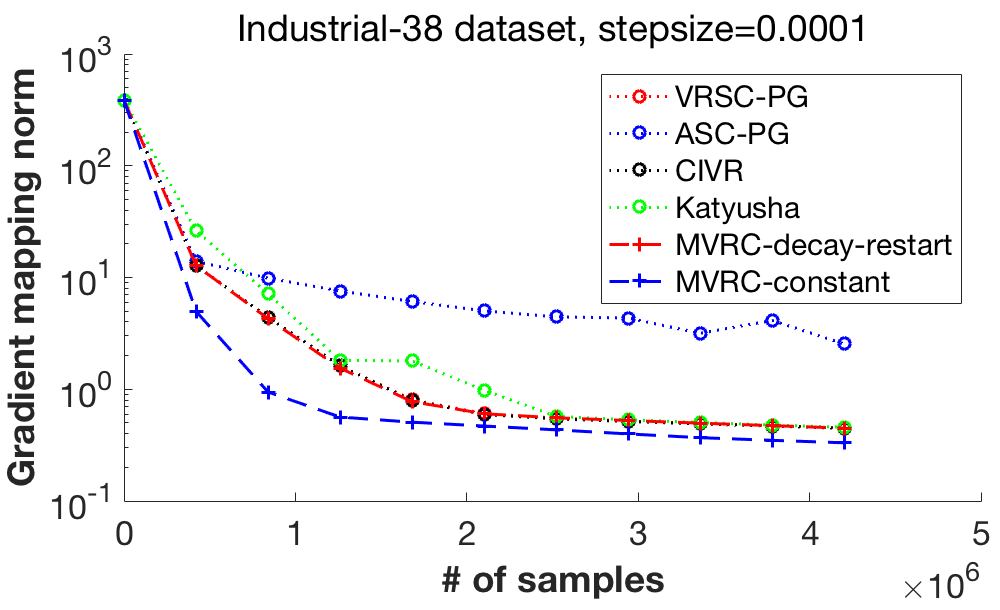}
		\includegraphics[width=\textwidth]{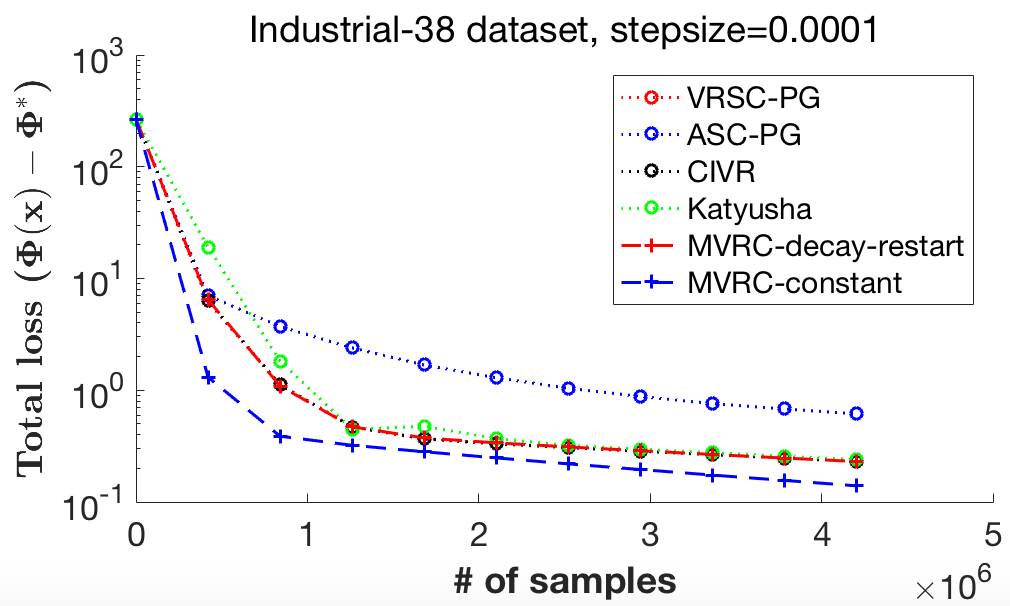}
	\end{minipage} 
	\begin{minipage}{.33\textwidth}
		\centering
		\includegraphics[width=\textwidth]{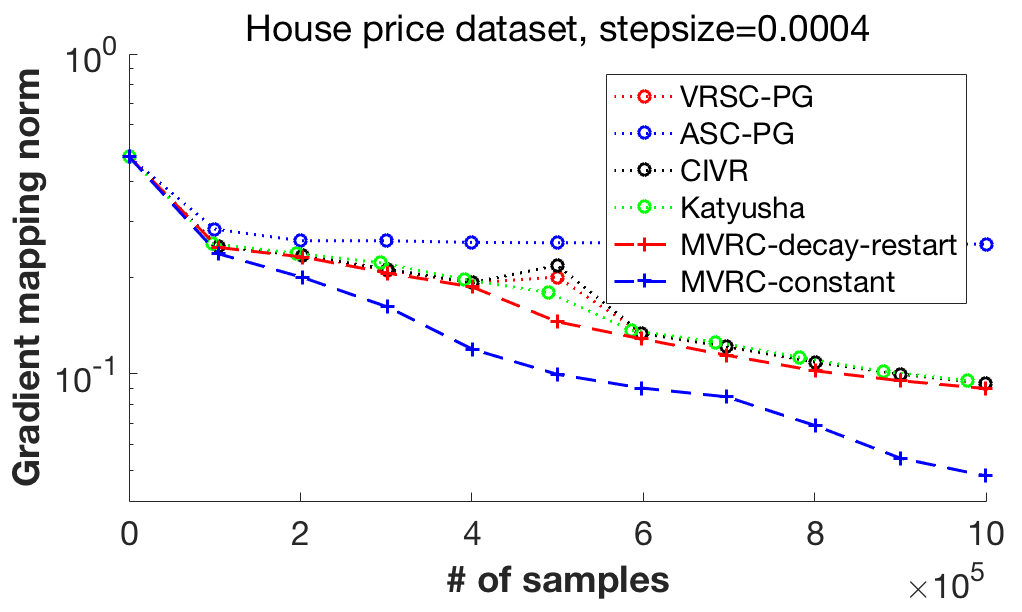}
		\includegraphics[width=\textwidth]{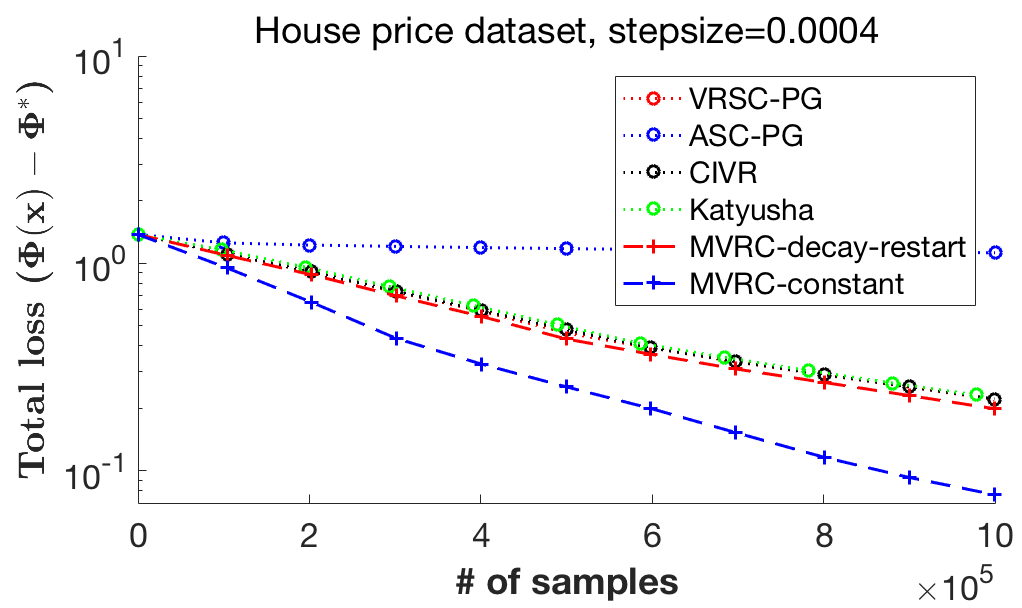}
	\end{minipage}
	\caption{Columns 1-2: Comparison of composition optimization algorithms in solving risk-averse portfolio optimization problems. Column 3: Comparison of composition optimization algorithms in solving nonconvex and nonsmooth sparse additive models.}
	\label{experiment_result}
\end{figure*}

\subsection{Risk-averse Portfolio Optimization}\label{sec_portfolio}
We consider the risk-averse portfolio optimization problem as elaborated in \Cref{append: ml example}. In specific, we set $\lambda=0.2$ in \cref{invest_sumobj} and add an $\ell_1$-{regularizer} $r(x)=0.01\|x\|_1$ to the objective function. We specify the values of $\{R_i\}_{i=1}^n$ using two industrial portfolio datasets from the Keneth R. French Data Library \footnote{\url{http://mba.tuck.dartmouth.edu/pages/faculty/ken.french/data_library.html}}, which have dimension $d=30$ and $d=38$, respectively. As in \cite{zhang2019stochastic}, we select the most recent {$n=24452, 10000$ days from the two datasets, respectively}. We implement all the algorithms using the same initialization point, batch size 256 and learning rate $\eta=10^{-3}, 10^{-4}$ respectively for the two datasets. For variance-reduced algorithms, we set each epoch to include {$J=\lceil n/256\rceil$} inner iterations. For ASC-PG, we set {$\alpha_t=0.0001t^{-5/9}$}, $\beta_t=t^{-4/9}$ as used by \cite{zhang2019stochastic}, whereas for Katyusha \cite{xu2019katyusha} we set $\tau_1=\tau_2=0.4$, $L=1/(3\eta)$, $\alpha=\eta$, {$\theta=1/(12J)$}.
For our MVRC (\Cref{alg: 1}), we consider two settings: 1) diminishing momentum with restart, which chooses $\alpha_t=\frac{2}{t+1}$, $\beta_t\equiv\eta$, $\lambda_t=\frac{t+3}{t+1}\eta$, where $t$ is reset to 0 and $y_t, z_t$ are reset to $x_t$ after each epoch; and 2) constant momentum, which chooses
$\alpha_t\equiv 0.8$, $\beta_t\equiv\eta$, $\lambda_t\equiv 1.8\eta$.
%\begin{figure*}
%	\centering
%	\includegraphics[width=0.6\linewidth]{experiment1_result}
%	\caption{Comparison of composition optimization algorithms in solving risk-averse portfolio optimization problems.}
%	\label{experiment1_result}
%\end{figure*}

\Cref{experiment_result} (Columns 1 \& 2) presents the convergence curves of these algorithms on both datasets with regard to the gradient mapping norm (top row) and function value gap (bottom row). It can be seen that our MVRC with constant momentum (shown as {``}MVRC-constant{''}) achieves the fastest convergence among all {the} algorithms and is significantly faster than the Katyusha composition optimization algorithm. Also, the convergence of our MVRC with diminishing momentum and restart (shown as {``}MVRC-decay-restart{''}) is comparable to that of CIVR and VRSC-PG, and is faster than ASC-PG.

\subsection{Nonconvex \& Nonsmooth Sparse Additive Model}\label{subsec_SPAM}
%\begin{figure}
%	\centering
%	\includegraphics[width=1\linewidth]{experiment2_result}
%	\caption{Comparison of composition optimization algorithms in solving nonconvex and nonsmooth sparse additive models.}
%	\label{experiment2_result}
%\end{figure}^\intercal
We further test these algorithms via solving the sparse additive model as introduced in \Cref{append: ml example}, where we use linear model $h_j(x_j)=\theta_j x_j$ and set $\lambda=1$. In particular, we adopt a modified nonconvex and nonsmooth problem \eqref{SpAM_obj_DE_ours} where the prediction is obtained via the nonlinear model ${\widehat{y}_i} = |\sum_{j=1}^d h_j(x_{ij})|$ and the objective function is further penalized by an $\ell_1$ regularization $r(x)=0.001\|x\|_1$. We use the US house price data\footnote{\url{https://www.kaggle.com/dmvreddy91/usahousing}} that consists of 5000 houses with their prices being approximated via linear combinations of averge income, age, number of rooms and population of a house per area. The price serves as the output and these four features with other 96 random Gaussian features serve as the input. To promote sparsity, all the coefficients of the Gaussian features are set to zero. All the algorithms use the same initialization point, learning rate $\eta=0.004$, batch sizes 10 and 70 for $f$ and $g$, respectively, epoch length $J=140$. For ASC-PG, we set $\alpha_t=0.004t^{-5/9}$, $\beta_t=2t^{-4/9}$. The other hyper-parameters of ASC-PG, Katyusha and our MVRC are the same as those specified in the previous subsection. 

\Cref{experiment_result} (Column 3) presents the convergence curves of these algorithms with regard to the gradient mapping norm ({top right}) and function value gap ({bottom right}). It can be seen that our MVRC with constant momentum {(shown as ``MVRC-constant'')}  converges significantly faster than the other algorithms, in particular, faster than the Katyusha composition optimization algorithm. Also, the convergence of our MVRC with diminishing momentum and restart {(shown as ``MVRC-decay-restart'')}  is comparable to that of CIVR, VRSC-PG and Katyusha. These experiments demonstrate that our MVRC provides significant practical acceleration to nonconvex stochastic composition optimization. 

%We slightly adjust the objective function into the problem $(\Sigma^2)$ with
%\begin{align}
%	g_i(\theta)=&\Big[\Big(y_i-\Big|\sum_{j=1}^{d} \theta_{j}x_{j}\Big|\Big)^{2},(\theta_1 x_1)^2,\ldots,(\theta_d x_d)^2\Big]^{\top},\nonumber \\
%	f_k(z)=&\left\{ \begin{gathered}
%	{z_1};k = 1 \hfill \\
%	\sqrt {{|z_k|}} ; |z_k|\ge 1, 2 \leqslant k \leqslant d + 1 \hfill \\ 
%	1.75z_k^2-0.75z_k^4 ; |z_k|< 1, 2 \leqslant k \leqslant d + 1 \hfill \\ 
%	\end{gathered}  \right.,\nonumber\\
%	r(\theta)=&0.001\|\theta\|_1 \nonumber
%\end{align}
%\noindent where we use linear function $h_j(x_j)=\theta_jx_j$ and add absolute value to the prediction to ensure nonnegative output, as in \cite{zhang2016reshaped}. Compared to Eq. \eqref{SpAM_obj_DE}, $f_k(z)~(2\le k\le d+1)$ at $|z_k|<1$ is changed to avoid gradient explosion when $z_k\to 0$, while still keeping smoothness. We also set $\lambda=1$, and add $l_1$ regularizer. 
%
%
%

\section{Conclusion}
In this paper, we develop momentum with variance reduction schemes for solving both finite-sum and online composition optimization problems and provide a comprehensive sample complexity analysis in nonconvex optimization. Our MVRC achieves the state-of-the-art near-optimal sample complexities and attains a linear convergence under the gradient dominant condition. We empirically demonstrate that MVRC with constant momentum outperforms all the other existing stochastic algorithms for composition optimization. In the future work, it is interesting to study the convergence guarantee for MVRC in solving more complex nonconvex problems such the model-agnostic meta-learning and reinforcement learning.

\newpage
\bibliography{reference}
\bibliographystyle{icml2020}

\onecolumn
\section*{Supplemental Materials}
\appendix
\section{Elaboration of Compositional Structure in ML Applications}\label{append: ml example}
\begin{enumerate}
	\item Risk-averse portfolio optimization \cite{zhang2019stochastic}:\\
	In this problem, we need to decide the amount of investment $x\in \mathbb{R}^d$ that involves $d$ assets to maximize the following total expected return penalized by the risk (i.e., the variance of the return)
	\begin{equation}\label{invest_Eobj}
	\min_{x \in \mathbb{R}^{d}} \mathbb{E}_\xi h_\xi(x)-\lambda \text{var}_\xi h_\xi(x),
	\end{equation}
	where $h_\xi(x)$ is the return at a random time $\xi$. Furthermore, suppose $R_\xi \in \mathbb{R}^{d}$ records the return per unit investment for each of the $d$ assets, such that $h_\xi(x)=R_\xi^{\top} x$, and the time $\xi$ is uniformly obtained from $1,2,\ldots,n$. Then, the objective function \eqref{invest_Eobj} can be reformulated as the composition problem $(\Sigma^1)$ with
	\begin{align}\label{invest_sumobj}
	&f(y_1,y_2)=-y_1-\lambda (y_1^2-y_2)\nonumber \\
	&g_i(x)=[h_i(x), h_i^2(x)]=[R_i^{\top}x, (R_i^{\top}x)^2].
	\end{align}
	\item Linear value function approximation in reinforcement learning \cite{zhang2019stochastic}:\\
	Consider a Markov decision process (MDP) with finite state space $\mathcal{S}=\{1,\ldots,S\}$, transition probability matrix $P^{\pi}$ associated with a certain policy $\pi$, and a random reward $R_{i,j}$ obtained after going from state $i$ to state $j$. Then the value function $V^{\pi}: \mathcal{S}\to\mathbb{R}$ must satisfy the Bellman equation:
	\begin{equation}\label{eq_Bellman}
	V^{\pi}(i)=\sum_{j=1}^{S} P_{i, j}^{\pi}\left(R_{i, j}+\gamma V^{\pi}(j)\right)%=\mathbf{E}_{j | i}\left[R_{i, j}+\gamma V^{\pi}(j)\right]
	\end{equation}
	When the number of states $S$ is large, the value function $V^{\pi}$ can be approximated by some function $\widetilde{V}^{\pi}_i(w)$ parameterized by $w\in \mathbb{R}^d$ to alleviate the computation burden. A simple example is the linear function approximation $\widetilde{V}_i^{\pi}(w)=\Phi_i^{\top}w$, where $\Phi_i\in \mathbb{R}^d$ is fixed. The goal of the value function approximation is to learn the parameter $w$ to minimize the following mean squared error
	\begin{align}\label{MSE_MDP}
	\min_{w\in\mathbb{R}^d} F(w):=&\frac{1}{S}\sum_{i=1}^{S}
	\Big[\widetilde{V}^{\pi}_i(w)- \sum_{j=1}^{S}P_{i,j}^{\pi}(R_{i,j}+\gamma\widetilde{V}^{\pi}_j(w))\Big]^2,
	\end{align}
	\noindent which can be reformulated as the problem $(\mathbb{E}^2)$ with
	\begin{align}\label{MSE_MDP_doubleE}
	&\eta=i\mathop \sim\limits^{\text{uniform}} \{1,\ldots, S\}, \mathbb{P}(\xi=j|\eta=i)=P_{i,j}^{\pi},\nonumber\\
	&g_\xi(w)=[\widetilde{V}_i^{\pi}(w),R_{i,\xi}+\gamma\widetilde{V}^{\pi}_\xi(w)],\nonumber \\
	&f_i(y_1,y_2)=(y_1-y_2)^2, 
	\end{align}
	or the problem $(\mathbb{E}^1)$ with
	\begin{align}\label{MSE_MDP_simpleCE}
	&q_i^{\pi}(w)=\mathbb{E}(w)=\mathbb{E}_\xi[R_{i,\xi}+\gamma\widetilde{V}^{\pi}_\xi(w)|i]~\text{with}~\mathbb{P}(\xi=j|i)=P_{i,j}^{\pi}, \nonumber\\
	&g(w)=[\widetilde{V}_1^{\pi}(w),\ldots,\widetilde{V}_S^{\pi}(w),q_1^{\pi}(w),\ldots,q_S^{\pi}(w)], \nonumber\\
	&f(y_1,\ldots,y_S,z_1,\ldots,z_S)=\frac{1}{S}\sum_{k=1}^S (y_k-z_k)^2. 
	\end{align}
	\item Stochastic neighbor embedding (SNE) \cite{liu2017variance}: \\
	SNE aims at dimension reduction while keeping the original metric among data points as much as possible from a probabilistic view \cite{hinton2003stochastic}. Specifically, for a set of given data points $z_1, \ldots, z_n\in \mathbb{R}^d$, a data point $z_k$ has probability 
	\begin{equation}\label{prob_neighbor_z}
	p_{i|k}=\frac{s(z_i, z_k)}{\sum_{j\ne k} s(z_j, z_k)}
	\end{equation}
	\noindent to select $x_i (i\ne k)$ as its neighbor, where $s(x_i, x_j)$ is the similarity measure between $x_i$ and $x_j$, e.g., the Gaussian form $s(x_i, x_j)=\exp\left(-\|x_i-x_j\|^2/(2\sigma_i^2)\right)$. We want to find a $p$-dim ($p<d$) representation $x_1, \ldots, x_n\in \mathbb{R}^p$ such that its neighbor selection probability
	\begin{equation}\label{prob_neighbor_x}
	q_{i|k}=\frac{s(x_i, x_k)}{\sum_{j\ne k} s(x_j, x_k)}
	\end{equation}
	\noindent is close to the probability \eqref{prob_neighbor_z}. A natural way is to minimize the KL divergence $\text{KL}(p_{i|k}, q_{i|k})=\sum_{i=1}^{n}\sum_{k=1}^{n} p_{i|k}\log(p_{i|k}/q_{i|k})$, which is equivalent to minimize
	\begin{align}\label{SNE_obj}
	\Phi(x)=&-\frac{1}{n^2}\sum_{i=1}^{n}\sum_{k=1}^{n} p_{i|k}\log q_{i|k},\nonumber\\
	=&\frac{1}{n^2}\sum_{i=1}^{n}\sum_{k=1}^{n} p_{i|k} \Big[\log \Big(\frac{1}{n}\sum_{j=1}^{n} s\left(x_{j}, x_{k}\right)-\frac{1}{n}s\left(x_{k}, x_{k}\right)\Big) -\log \Big(\frac{1}{n}s\left(x_{i}, x_{k}\right)\Big)\Big],
	\end{align}
	\noindent where $x=[x_1, \ldots, x_n]\in \mathbb{R}^{np}.$
	This problem can be reformulated as the problem $(\Sigma^2)$ with
	\begin{equation}\label{SNE_DFS_obj}
	\Phi(x)=\frac{1}{n^2}\sum_{i=1}^{n}\sum_{k=1}^{n} f_{i,k}\left[\frac{1}{n}\sum_{j=1}^{n} g_j(x)\right],
	\end{equation}
	\noindent where
	\begin{align}\label{SNE_g_j}
	g_j(x)=&\Big[x_1, x_2, \ldots, x_n, s(x_j, x_1)-\frac{1}{n} s(x_1,x_1),  \ldots, s(x_j, x_n)-\frac{1}{n} s(x_n,x_n)\Big],
	\end{align}
	\begin{equation}\label{SNE_f_ik}
	f_{i,k}(w)=p_{i | k}\Big[\log w_{n+k} - \log \Big(\frac{1}{n}s(w_{i}, w_{k})\Big)\Big].%; i,k=1,\ldots,n. 
	\end{equation}

	%Appendix
	%\item multi-stage stochastic programming [Shapiro et al.]
	%	\item Minimax game \cite{wang2017stochastic}\\
	%	The minimax game has the following objective function
	%	\begin{equation}
	%		\min_{x} \mathbb{E}_v \max_{1\le i\le I} \mathbb{E}_w [g_w^{i}(x)|v],
	%	\end{equation}
	%	which can be viewed as the DE problem \eqref{doubleE}
	\item Sparse additive model (SpAM) \cite{wang2017stochastic}:\\
	\indent SpAM is an important model for nonparametric regression where the input vector $x_i=(x_{i1}, \ldots, x_{id})^{\top}\in\mathbb{R}^d$ is high-dimensional. The true output is $y_i\in \mathbb{R}$ and the corresponding model prediction is $\widehat{y}_i=\sum_{j=1}^{d} h_j(x_{ij})$, where $h_j:\mathbb{R}\to\mathbb{R}$ is selected from a certain function family $\mathcal{H}_j$. The object is to minimize the following penalized mean squared error
	\begin{align}\label{SpAM_obj}
	\min _{h_{j} \in \mathcal{H}_{j}, j=1 \ldots, d} 
	&\frac{1}{n}\sum_{i=1}^{n}\Big[y_i-\sum_{j=1}^{d} h_{j}\left(x_{ij}\right)\Big]^{2}+\lambda \sum_{j=1}^{d} \sqrt{\frac{1}{n}\sum_{i=1}^{n}h_{j}^{2}\left(x_{ij}\right)},
%	\min _{h_{j} \in \mathcal{H}_{j}, j=1 \ldots, d} 
%	&\mathbf{E}\Big[Y-\sum_{j=1}^{d} h_{j}\left(X_{j}\right)\Big]^{2}+\lambda \sum_{j=1}^{d} \sqrt{\mathbf{E}\big[h_{j}^{2}\left(X_{j}\right)\big]},
	\end{align}
	\noindent \cite{huang2010variable} showed some good statistical properties of SpAM. To use an efficient stochastic optimizer, the objective function \eqref{SpAM_obj} is further reformulated as the problem $(\Sigma^2)$ with
	\begin{align}\label{SpAM_obj_DE}
	%&\xi=[Y,X_1,\ldots,X_d],\nonumber\\
	%&\eta \mathop \sim\limits^{\text{unif}} \{1,\ldots, d+1\}\nonumber\\
	&g_i(h_1,\ldots,h_d)=\Big[\Big(y_i-\sum_{j=1}^{d} h_{j}\left(x_{ij}\right)\Big)^{2},h_1(x_{i1})^2,\ldots,h_d(x_{id})^2\Big]^{\top},\nonumber \\
	&f_k(y)=\left\{ \begin{gathered}
	{y_1};k = 1 \hfill \\
	\sqrt {|{y_k}|} ;2 \leqslant k \leqslant d + 1 \hfill \\ 
	\end{gathered}  \right..
	\end{align}
	In our experiment in \Cref{subsec_SPAM}, we adopt a nonconvex and nonsmooth modification of the SpAM, which is also formulated as the problem $(\Sigma^2)$ with
	\begin{align}\label{SpAM_obj_DE_ours}
	&g_i(h_1,\ldots,h_d)=\Big[\Big(y_i-\Big|\sum_{j=1}^{d} h_{j}\left(x_{ij}\right)\Big|\Big)^{2},h_1(x_{i1})^2,\ldots,h_d(x_{id})^2\Big]^{\top},\nonumber \\
	&f_1(y)=y_1,\nonumber\\
	&f_k(y)=\left\{ \begin{gathered}
	\sqrt {|{y_k}|} ; |y_k|\ge 1\hfill \\ 
	1.75y_k^2-0.75y_k^4; |y_k|<1 \hfill \\
	\end{gathered}  \right.; 2 \leqslant k \leqslant d + 1\\
	&r(x)=\mu\|x\|_1,
	\end{align}
	where the prediction ${\widehat{y}_i} = |\sum_{j=1}^d h_j(x_{ij})|$ is nonlinear and nonnegative, and the square root function in $[-1,1]$ is modified to ensure bounded derivative.
	%\item Adaptive simulation \cite{hu2014model}
	%\item Model Agnostic Meta Learning: Current MAMLs does not make use of double finite sum
\end{enumerate}
\section{Auxiliary Lemmas for Proving \Cref{thm_alpha_decay_single}}
%\subsection*{A. Lemmas for Theorems \ref{thm_alpha_decay_SCE_nonconvex} $\&$ \ref{thm_alpha_decay_SCFS_nonconvex}}
The following \Cref{lemma_approx_grad} originates from the Lemma 1 of \cite{zhang2019stochastic}, and we present it for completeness.
\begin{lemma}\label{lemma_approx_grad}
	Let Assumptions \ref{assumption1} and \ref{assumption2} hold and apply \Cref{alg: 1} to solve the problems $(\Sigma^1)$ and $(\mathbb{E}^1)$. Then, the variance of the stochastic gradient satisfies
	\begin{align}\label{eq_approx_grad}
		\mathbb{E}&\left\|\widetilde{\nabla} F\left(z_{t}\right)-\nabla F\left(z_{t}\right)\right\|^{2} \leq \frac{\sigma_{0}^{2}}{|\mathcal{A}_{\tau \left\lfloor t/\tau \right\rfloor}|}+G_{0} \sum_{s=\tau \lfloor t / \tau\rfloor+1}^{t} \frac{1}{|\mathcal{A}_{s}|} \mathbb{E}\left\|z_{s}-z_{s-1}\right\|^{2},
	\end{align}
	\noindent where
	\begin{align}\label{eq_G0_sigma0}
		G_{0}:=&2\left(l_{g}^{4} L_{f}^{2}+l_{f}^{2} L_{g}^{2}\right),\nonumber\\ 
		\sigma_{0}^{2}:=&\left\{ \begin{gathered}
			0; \text{For the problem } (\Sigma^1) \hfill \\
			2\left(l_{g}^{2} L_{f}^{2} \sigma_{g}^{2}+l_{f}^{2} \sigma_{g^{\prime}}^{2}\right); \text{For the problem }(\mathbb{E}^1) \hfill \\ 
		\end{gathered}  \right.
	\end{align}
Note that when $t=m\tau$ for some positive integer $m$, the summation becomes $\sum_{s=m\tau+1}^{m\tau}$, which is 0 by default. 
\end{lemma}

\begin{lemma}\label{lemma_seq_bound_alpha_decay}
	Implement \Cref{alg: 1} with $\alpha_{t}=\frac{2}{t+1}, \beta_t\equiv\beta, \beta\le\lambda_{t}\le(1+\alpha_t)\beta$ to solve the problems $(\Sigma^1)$ and $(\mathbb{E}^1)$. Then, the generated sequences $\{ x_{t}, y_{t}, z_{t}\}$ satisfy
	\begin{equation}\label{eq_xydiff}
		y_{t}-x_{t}=-\Gamma_{t} \sum_{s=1}^{t} \frac{\lambda_{s-1}-\beta}{\Gamma_{s}\lambda_{s-1}} (x_s-x_{s-1}),
	\end{equation}
	\begin{equation}\label{eq_xydiff_bound}
		\left\|y_{t}-x_{t}\right\|^{2} \leq \Gamma_{t} \sum_{s=1}^{t} \frac{(\lambda_{s-1}-\beta)^2}{\alpha_{s} \Gamma_{s}\lambda_{s-1}^2}\left\|x_s-x_{s-1}\right\|^{2},
	\end{equation}
	\begin{equation}\label{eq_zdiff_bound}
		\left\|z_{t+1}-z_{t}\right\|^{2} \leq \frac{2 \beta_{t}^{2}}{\lambda_{t}^2}\left\|x_{t+1}-x_{t}\right\|^{2}+2 \alpha_{t+2}^{2} \Gamma_{t+1} \sum_{s=1}^{t+1} \frac{\left(\lambda_{s-1}-\beta_{s-1}\right)^{2}}{\alpha_{s} \Gamma_{s}\lambda_{s-1}^2}\left\|x_s-x_{s-1}\right\|^{2},
	\end{equation}
 where $\Gamma_t=\frac{2}{t(t+1)}$. When $t=0$, the summation $\sum_{t=1}^{0}$ is 0 by default.
\end{lemma}
%\textbf{Note:} This lemma and its proof below are highly similar to those in \cite{zhou2019momentum}. 

\begin{proof}
	Based on the update rules in Algorithm \ref{alg: 1},
	\begin{align}
		 y_{t+1}- x_{t+1}=& z_{t}+\frac{\beta_{t}}{\lambda_{t}}( x_{t+1}- x_{t})- x_{t+1} \nonumber \\
		=&\left(1-\alpha_{t+1}\right) y_{t}+\alpha_{t+1} x_{t} +\frac{\beta_{t}}{\lambda_{t}}( x_{t+1}- x_{t})- x_{t+1} \nonumber \\
		=&\left(1-\alpha_{t+1}\right) ( y_{t}- x_{t}) +\left(\frac{\beta_{t}}{\lambda_{t}}-1\right)( x_{t+1}- x_{t}). \nonumber 
	\end{align}
	\indent Taking the above equality as a difference equation about $ y_{t}- x_{t}$ with initial condition $ y_0- x_0= 0$, it can be verified that Eq. \eqref{eq_xydiff} is the solution. Hence,
	\begin{align}
	\| y_{t}- x_t\|^2=&\left\|\Gamma_{t}\sum_{s=1}^{t} \frac{\lambda_{s-1}-\beta}{\Gamma_{s}\lambda_{s-1}} (x_s-x_{s-1}) \right\|^2 \nonumber \\
	=&\left\|\sum_{s=1}^{t} \frac{\Gamma_{t}\alpha_{s}}{\Gamma_{s}} \frac{\lambda_{s-1}-\beta}{\alpha_{s}\lambda_{s-1}} (x_s-x_{s-1}) \right\|^2 \nonumber \\
	\le&\sum_{s=1}^{t} \frac{\Gamma_{t}\alpha_{s}}{\Gamma_{s}} \left\|\frac{\lambda_{s-1}-\beta}{\alpha_{s}\lambda_{s-1}} (x_s-x_{s-1}) \right\|^2 \nonumber \\
	=&\sum_{s=1}^{t} \frac{\Gamma_{t}(\lambda_{s-1}-\beta)^2}{\Gamma_{s}\alpha_{s}\lambda_{s-1}^2} \left\| x_s-x_{s-1} \right\|^2, \nonumber 
	\end{align}
	\noindent where $\le$ applies Jensen's inequality to the convex function $\|\bullet\|^2$ with $\sum_{s=1}^{t} \frac{\Gamma_{t}\alpha_{s}}{\Gamma_{s}}=1$. This proves Eq. \eqref{eq_xydiff_bound}.
	
	\indent It can be derived from the update rules in Algorithm \ref{alg: 1} that $z_{t+1}- z_t= \beta_{t}(x_{t+1}-x_t)/\lambda_{t}-\alpha_{t+2}( y_{t+1}- x_{t+1})$. Hence,
	\begin{align}
		\| z_{t+1}- z_t\|^2=&\left\| \frac{\beta_t}{\lambda_t}(x_{t+1}- x_t)-\alpha_{t+2}( y_{t+1}- x_{t+1}) \right\|^2 \nonumber \\
		\le&2\left\| \frac{\beta_t}{\lambda_t}(x_{t+1}- x_t) \right\|^2+2\alpha_{t+2}^2\left\|y_{t+1}- x_{t+1}\right\|^2 \nonumber \\
		\le&\frac{2\beta_t^2}{\lambda_t^2}\left\|x_{t+1}- x_t \right\|^2 +2\alpha_{t+2}^2\Gamma_{t+1} \sum_{s=1}^{t+1} \frac{(\lambda_{s-1}-\beta)^2}{\alpha_{s} \Gamma_{s}\lambda_{s-1}^2}\left\|x_s-x_{s-1}\right\|^{2}, \nonumber 
	\end{align}
	\noindent which proves Eq. \eqref{eq_zdiff_bound}. 
\end{proof}

\section{Proof of \Cref{thm_alpha_decay_single}}\label{sec_proof_thm_alpha_decay_single}
\indent Since $ x_{t+1}$ is the minimizer of the function $\widetilde{r}( x):=r( x)+\frac{1}{2\lambda_t}\| x- x_{t}+\lambda_{t}{\widetilde{\nabla}F}( z_{t})\|^2$, which is $\lambda_{t}^{-1}$-strongly convex as $r$ is convex, we have
\begin{align}\label{eq_r_convex}
	&\widetilde{r}( x_t)\ge\widetilde{r}( x_{t+1})+\frac{1}{2\lambda_t}\| x_{t+1}- x_t\|^2\nonumber\\
	\Rightarrow&r( x_t)+\frac{\lambda_t}{2}\|{\widetilde{\nabla}F}( z_{t})\|^2 \ge r( x_{t+1})+\frac{1}{2\lambda_t}\| x_{t+1}- x_{t}+\lambda_{t}{\widetilde{\nabla}F}( z_{t})\|^2+\frac{1}{2\lambda_t}\| x_{t+1}- x_t\|^2\nonumber\\
	\Rightarrow&r( x_{t+1})-r( x_t)\le-\frac{1}{\lambda_t}\| x_{t+1}- x_t\|^2-\left<{\widetilde{\nabla}F}( z_{t}),  x_{t+1}- x_{t}\right>. %\numberthis
\end{align}
\indent Since $ \nabla F$ is $L_F$-Lipschitz, we obtain that
\begin{equation}\label{eq_f_Taylor}
	F( x_{t+1})-F( x_t)\le\left< \nabla F( x_{t}), x_{t+1}- x_t\right>+\frac{L_F}{2}\| x_{t+1}- x_t\|^2.
\end{equation}
\indent Adding up Eqs. \eqref{eq_r_convex}$\&$\eqref{eq_f_Taylor} and taking expectation, we further obtain that
\begin{align}\label{eq_adjacent_Phi_diff_bound}
\mathbb{E}\Phi( x_{t+1})-\mathbb{E}\Phi( x_t)\le& \mathbb{E}\left< \nabla F( x_{t})-{\widetilde{\nabla}F}( z_{t}), x_{t+1}- x_t\right>+\left(\frac{L_F}{2}-\frac{1}{\lambda_t}\right)\mathbb{E}\| x_{t+1}- x_t\|^2 \nonumber \\
=& \mathbb{E}\left< \nabla F( z_{t})-{\widetilde{\nabla}F}( z_{t}), x_{t+1}- x_t\right>+\mathbb{E}\left< \nabla F( x_{t})- \nabla F( z_{t}), x_{t+1}- x_t\right> \nonumber \\
&+\left(\frac{L_F}{2}-\frac{1}{\lambda_t}\right)\mathbb{E}\| x_{t+1}- x_t\|^2 \nonumber \\
\le& \mathbb{E}\left(\left\| \nabla F(z_{t})-{\widetilde{\nabla}F}(z_{t})\right\| \left\|x_{t+1}-x_t\right\|\right)+\mathbb{E}\left( \left\|\nabla F( x_{t})- \nabla F( z_{t})\right\| \left\|x_{t+1}- x_t\right\|\right) \nonumber \\
&+\left(\frac{L_F}{2}-\frac{1}{\lambda_t}\right)\mathbb{E}\| x_{t+1}- x_t\|^2 \nonumber \\
\le& \frac{\lambda_t}{2}\mathbb{E}\left\| \nabla F(z_{t})-{\widetilde{\nabla}F}(z_{t})\right\|^2+ \frac{1}{2\lambda_t}\mathbb{E}\left\|x_{t+1}-x_t\right\|^2 \nonumber \\
&+L_F\mathbb{E}\left( \left\|x_{t}-z_{t}\right\| \left\|x_{t+1}- x_t\right\|\right)+\left(\frac{L_F}{2}-\frac{1}{\lambda_t}\right)\mathbb{E}\| x_{t+1}- x_t\|^2 \nonumber \\
\le& \frac{\lambda_t}{2}\mathbb{E}\left\| \nabla F(z_{t})-{\widetilde{\nabla}F}(z_{t})\right\|^2+ \frac{1}{2\lambda_t}\mathbb{E}\left\|x_{t+1}-x_t\right\|^2 \nonumber \\
&+L_F(1-\alpha_{t+1})\mathbb{E}\left( \left\|y_{t}-x_{t}\right\| \left\|x_{t+1}- x_t\right\|\right)+\left(\frac{L_F}{2}-\frac{1}{\lambda_t}\right)\mathbb{E}\| x_{t+1}- x_t\|^2 \nonumber \\
\le& \frac{\lambda_t}{2}\mathbb{E}\left\| \nabla F(z_{t})-{\widetilde{\nabla}F}(z_{t})\right\|^2+ \frac{1}{2\lambda_t}\mathbb{E}\left\|x_{t+1}-x_t\right\|^2 \nonumber \\
&+\frac{L_F}{2}\mathbb{E} \left\|y_{t}-x_{t}\right\|^2 +\frac{L_F}{2}\mathbb{E}\left\|x_{t+1}- x_t\right\|^2+\left(\frac{L_F}{2}-\frac{1}{\lambda_t}\right)\mathbb{E}\| x_{t+1}- x_t\|^2 \nonumber \\
\le& \frac{\lambda_t}{2}\mathbb{E}\left\| \nabla F(z_{t})-{\widetilde{\nabla}F}(z_{t})\right\|^2+\left(L_F-\frac{1}{2\lambda_t}\right)\mathbb{E}\| x_{t+1}- x_t\|^2 \nonumber  \nonumber \\
&+\frac{L_F\Gamma_{t}}{2} \sum_{s=1}^{t} \frac{(\lambda_{s-1}-\beta)^2}{\alpha_{s} \Gamma_{s}\lambda_{s-1}^2}\left\|x_s-x_{s-1}\right\|^{2},%\numberthis
\end{align}
\noindent where the last inequality uses Eq. \eqref{eq_xydiff_bound}.
Telescoping the above inequality over $t$ from $0$ to $T-1$ yields that
\begin{align}\label{eq_Phi_Taylor}
\mathbb{E}\Phi(x_{T})-\Phi(x_0)\le& \sum_{t=0}^{T-1}  \frac{\lambda_t}{2}\mathbb{E}\left\| \nabla F(z_{t})-{\widetilde{\nabla}F}(z_{t})\right\|^2 + \sum_{t=0}^{T-1} \left(L_F-\frac{1}{2\lambda_t}\right)\mathbb{E}\| x_{t+1}- x_t\|^2 \nonumber \\
&+\sum_{t=0}^{T-1} \frac{L_F\Gamma_{t}}{2} \sum_{s=1}^{t} \frac{(\lambda_{s-1}-\beta)^2}{\alpha_{s} \Gamma_{s}\lambda_{s-1}^2}\left\|x_s-x_{s-1}\right\|^{2} \nonumber \\
\stackrel{(i)}{\le}& \sum_{t=0}^{T-1}  \frac{\lambda_t}{2}\left[G_{0} \sum_{s=\tau\left\lfloor t/\tau\right\rfloor+1}^{t} \frac{1}{|\mathcal{A}_{s}|} \mathbb{E}\left\| z_{s}- z_{s-1}\right\|^{2}+\frac{\sigma_{0}^{2}}{|\mathcal{A}_{\tau\left\lfloor t/\tau\right\rfloor}|}\right] \nonumber \\
&+\sum_{t=0}^{T-1} \left(L_F-\frac{1}{2\lambda_t}\right)\mathbb{E}\| x_{t+1}- x_t\|^2 \nonumber \\
&+\frac{L_F}{2}\sum_{s=1}^{T-1} \frac{(\lambda_{s-1}-\beta)^2}{\alpha_{s} \Gamma_{s}\lambda_{s-1}^2}\left(\sum_{t=s}^{T-1}\Gamma_{t}\right)\mathbb{E}\left\|x_s-x_{s-1}\right\|^{2} \nonumber \\
\stackrel{(ii)}{=}& \frac{G_0}{2} \sum_{s=1}^{T-1} \frac{1}{|\mathcal{A}_{s}|} \mathbb{E}\left\|z_{s}- z_{s-1}\right\|^{2} \sum_{t=s}^{\min\left(\tau\left\lfloor \frac{s-1}{\tau}\right\rfloor+\tau-1, T-1\right)} \lambda_t+\frac{\sigma_{0}^{2}}{2}\sum_{t=0}^{T-1}\frac{\lambda_t}{|\mathcal{A}_{\tau\left\lfloor t/\tau\right\rfloor}|} \nonumber \\
&+\sum_{t=0}^{T-1} \left(L_F-\frac{1}{2\lambda_t}\right)\mathbb{E}\| x_{t+1}- x_t\|^2 \nonumber \\
&+L_F\sum_{s=1}^{T-1} \frac{(\lambda_{s-1}-\beta)^2}{\alpha_{s} \Gamma_{s}\lambda_{s-1}^2}\sum_{t=s}^{T-1}\left( \frac{1}{t}-\frac{1}{t+1}\right)\mathbb{E}\left\|x_s-x_{s-1}\right\|^{2} \nonumber \\
\stackrel{(iii)}{\le}& \frac{G_0}{2}\sum_{s=1}^{T-1}\frac{1}{\tau}\mathbb{E}\|z_s-z_{s-1}\|^2 \left(\tau\left\lfloor \frac{s-1}{\tau}\right\rfloor+\tau-1-(s-1)\right)(2\beta) \nonumber \\
&+\frac{T\beta\sigma_{0}^2}{|\mathcal{A}_0|}+\sum_{t=0}^{T-1}\left(L_F-\frac{1}{4\beta}\right)\mathbb{E}\|x_{t+1}-x_t\|^2 \nonumber \\
&+\frac{L_F}{4}\sum_{s=1}^{T-1}s(s+1)^2\left(\frac{\alpha_{s-1}}{1+\alpha_{s-1}}\right)^2\mathbb{E}\|x_s-x_{s-1}\|^2\left(\frac{1}{s}-\frac{1}{T}\right) \nonumber \\
\stackrel{(iv)}{\le}& \beta G_0\sum_{s=1}^{T-1}\left[
\frac{2\beta^2}{\lambda_{s-1}^2}\mathbb{E}\left\|x_s- x_{s-1} \right\|^2+2\alpha_{s+1}^2 \Gamma_{s} \sum_{t=1}^{s} \frac{(\lambda_{t-1}-\beta)^2}{\alpha_{t} \Gamma_{t}\lambda_{t-1}^2}\mathbb{E}\left\|x_t-x_{t-1}\right\|^{2}
\right] \nonumber \\
&+\frac{T\beta\sigma_{0}^2}{|\mathcal{A}_0|}+\sum_{t=0}^{T-1}\left(L_F-\frac{1}{4\beta}\right)\mathbb{E}\|x_{t+1}-x_t\|^2 \nonumber \\
&+\frac{L_F}{4}\sum_{s=1}^{T-1}s(s+1)^2\frac{4}{(s+2)^2}\mathbb{E}\|x_s-x_{s-1}\|^2\left(\frac{1}{s}-\frac{1}{T}\right) \nonumber \\
\stackrel{(v)}{\le}& 2\beta G_0\sum_{s=1}^{T-1} \mathbb{E}\left\|x_s- x_{s-1} \right\|^2 \nonumber \\
&+2\beta G_0\sum_{t=1}^{T-1} \sum_{s=t}^{T-1} \frac{4}{(s+2)^2}\frac{2}{s(s+1)}\frac{t+1}{2}\frac{t(t+1)}{2}\left(\frac{\alpha_{t-1}}{1+\alpha_{t-1}}\right)^2
\mathbb{E}\left\|x_t-x_{t-1}\right\|^{2} \nonumber \\
&+\frac{T\beta\sigma_{0}^2}{|\mathcal{A}_0|}+\sum_{t=0}^{T-1}\left(L_F-\frac{1}{4\beta}\right)\mathbb{E}\|x_{t+1}-x_t\|^2+L_F\sum_{t=0}^{T-2}\mathbb{E}\|x_{t+1}-x_{t}\|^2 \nonumber \\
\stackrel{(vi)}{\le}& 2\beta G_0\sum_{t=0}^{T-2} \mathbb{E}\left\|x_{t+1}-x_t\right\|^2+\beta G_0\sum_{t=1}^{T-1} (t+1)\frac{4}{(t+2)^2}
\mathbb{E}\left\|x_t-x_{t-1}\right\|^{2} \nonumber \\
&+\frac{T\beta\sigma_{0}^2}{|\mathcal{A}_0|}+\sum_{t=0}^{T-1}\left(2L_F-\frac{1}{4\beta}\right)\mathbb{E}\|x_{t+1}-x_t\|^2 \nonumber \\
\le& \frac{T\beta\sigma_{0}^2}{|\mathcal{A}_0|}+\left(3\beta G_0+2L_F-\frac{1}{4\beta}\right)\sum_{t=0}^{T-1}\mathbb{E}\|x_{t+1}-x_t\|^2,%\numberthis
\end{align}
\noindent where (i) follows from Lemma \ref{lemma_approx_grad} and switches the last two summations, (ii) switches the first two summations, (iii) uses the hyperparameter choices that $\beta\le\lambda_{t}\le(1+\alpha_{t})\beta\le 2\beta$, $|\mathcal{A}_t|\ge\tau$ and $|\mathcal{A}_{m\tau}|\equiv |\mathcal{A}_0| (m\in\mathbb{N})$, (iv) uses Eq. \eqref{eq_zdiff_bound}, $\frac{1}{\tau}\left(\tau\left\lfloor \frac{s-1}{\tau}\right\rfloor+\tau-1-(s-1)\right)\le 1$ and $\beta\le\lambda_{t}\le(1+\alpha_{t})\beta\le 2\beta$, (v) switches the first two summations and replaces $s$ with $t+1$ in the last summation, (vi) uses the following inequality
\begin{align}\label{eq_sum_s_bound}
	\sum_{s=t}^{T-1}\frac{1}{s(s+1)(s+2)^2}=&\frac{1}{2}\sum_{s=t}^{T-1}\left(\frac{1}{s}-\frac{1}{s+2}\right)\left(\frac{1}{s+1}-\frac{1}{s+2}\right) \nonumber \\
	=&\frac{1}{2}\sum_{s=t}^{T-1}\left[\left(\frac{1}{s}-\frac{1}{s+1}\right)-\left(\frac{1}{s+1}-\frac{1}{s+2}\right)-\frac{1}{2}\left(\frac{1}{s}-\frac{1}{s+2}\right)+\frac{1}{(s+2)^2}\right] \nonumber \\
	=&\frac{1}{2}\left[\left(\frac{1}{t}-\frac{1}{T}\right)-\left(\frac{1}{t+1}-\frac{1}{T+1}\right)-\frac{1}{2}\left(\frac{1}{t}+\frac{1}{t+1}-\frac{1}{T}-\frac{1}{T+1}\right)\right. \nonumber \\
	&+\left.\int_{t-1}^{T-1}\frac{ds}{(s+2)^2}\right] \nonumber \\
	=&\frac{1}{4t}-\frac{3}{4(t+1)}-\frac{1}{4T}+\frac{3}{4(T+1)}+\frac{1}{2}\left(\frac{1}{t+1}-\frac{1}{T+1}\right) \nonumber \\
	\le&\frac{1}{4t(t+1)}.%\numberthis
\end{align} 
%\stackrel{(i)}{ \le}
\noindent Let $3\beta G_0+2L_F-\frac{1}{4\beta}\le-\frac{1}{8\beta}$, which together with $\beta>0$ implies $0<\beta\le\frac{1}{2\sqrt{16L_F^2+6G_0}+8L_F}$. Then Eq. \eqref{eq_Phi_Taylor} further indicates that
\begin{equation}\label{sum_xdiff_bound}
	\sum_{t=0}^{T-1}\mathbb{E}\|x_{t+1}-x_t\|^2 \le \frac{8T\beta^2\sigma_{0}^2}{|\mathcal{A}_0|}+8\beta\left[\Phi\left(x_{0}\right)-\mathbb{E}\Phi(x_T)\right].
\end{equation}
\noindent Therefore, let $\zeta$ be sampled from $\{0,...,T-1 \}$ uniformly at random, we obtain that
\begin{align}\label{eq_Egrad_bound2}
	\mathbb{E} \|{\mathcal{G}}_{\lambda_\zeta}( z_\zeta)\|^2=&\frac{1}{T}\sum_{t=0}^{T-1} \mathbb{E}\|{\mathcal{G}}_{\lambda_t}( z_t)\|^2 \nonumber \\
	\stackrel{(i)}{=}&\frac{1}{T}\sum_{t=0}^{T-1} \lambda_t^{-2}\mathbb{E}\left\| z_t-\operatorname{prox}_{\lambda_t r}\left[ z_t-\lambda_t\nabla F( z_t)\right]\right\|^2 \nonumber \\
	\stackrel{(ii)}{ =}&\frac{9}{T}\sum_{t=0}^{T-1} \lambda_t^{-2}\mathbb{E}\left\|\frac{1}{3}(z_t-x_t)+\frac{1}{3}( x_t- x_{t+1})\right. \nonumber \\
	&+\left.\frac{1}{3}\left\{\operatorname{prox}_{\lambda_{t}r}[ x_{t}-\lambda_{t}{\widetilde{\nabla}F}( z_{t})]-\operatorname{prox}_{\lambda_t r}\left[ z_t-\lambda_t\nabla F( z_t)\right]\right\}\right\|^2 \nonumber \\
	\stackrel{(iii)}{\le}&\frac{3}{T}\sum_{t=0}^{T-1} \lambda_t^{-2}\mathbb{E}\| z_t- x_t\|^2+\frac{3}{T}\sum_{t=0}^{T-1} \lambda_t^{-2}\mathbb{E}\| x_{t+1}- x_t\|^2 \nonumber \\
	&+\frac{3}{T}\sum_{t=0}^{T-1} \lambda_t^{-2}\mathbb{E}\left\|( x_{t}- z_t)+\lambda_{t}\left[\nabla F( z_t)-{\widetilde{\nabla}F}( z_{t})]\right]\right\|^2 \nonumber \\
	\stackrel{(iv)}{\le}&\frac{9}{T}\sum_{t=0}^{T-1} \lambda_t^{-2}(1-\alpha_{t+1})^2\mathbb{E} \| y_t- x_t\|^2+\frac{3}{T}\sum_{t=0}^{T-1} \lambda_t^{-2}\mathbb{E}\| x_{t+1}- x_t\|^2+\frac{6}{T}\sum_{t=0}^{T-1} \mathbb{E}\|\nabla F( z_t)-{\widetilde{\nabla}F}( z_{t})\|^2 \nonumber \\
	\stackrel{(v)}{\le}&\frac{9}{T}\sum_{t=1}^{T-1} \lambda_t^{-2}(1-\alpha_{t+1})^2 \Gamma_{t} \sum_{s=1}^{t} \frac{(\lambda_{s-1}-\beta)^2}{\alpha_{s} \Gamma_{s}\lambda_{s-1}^2}\mathbb{E}\left\|x_s-x_{s-1}\right\|^{2} + \frac{3}{T}\sum_{t=1}^{T} \lambda_{t-1}^{-2}\mathbb{E}\| x_{t}- x_{t-1}\|^2 \nonumber \\
	&+\frac{6}{T}\sum_{t=0}^{T-1} \left[G_{0} \sum_{s=\tau\left\lfloor t/\tau\right\rfloor+1}^{t} \frac{1}{|\mathcal{A}_s|} \mathbb{E}\left\| z_{s}- z_{s-1}\right\|^{2}+\frac{\sigma_{0}^{2}}{|\mathcal{A}_{\tau\left\lfloor t/\tau\right\rfloor}|}\right] \nonumber \\
	\stackrel{(vi)}{\le}&\frac{9}{T}\sum_{s=1}^{T-1} \sum_{t=s}^{T-1} \beta^{-2} \frac{2}{t(t+1)} \frac{s(s+1)^2}{4} \left(\frac{\alpha_{s-1}}{1+\alpha_{s-1}}\right)^2 \mathbb{E}\left\| x_s- x_{s-1}\right\|^2 \nonumber \\
	&+\frac{3}{T}\sum_{t=1}^{T} \beta^{-2}\mathbb{E}\|x_{t}- x_{t-1}\|^2 \nonumber \\
	&+\frac{6G_0}{T} \sum_{s=1}^{T-1} \frac{1}{\tau} \mathbb{E}\left\|z_{s}- z_{s-1}\right\|^{2} \sum_{t=s}^{\min\left(\tau\left\lfloor \frac{s-1}{\tau}\right\rfloor+\tau-1, T-1\right)} 1+\frac{6\sigma_{0}^{2}}{|\mathcal{A}_0|} \nonumber \\
	\le& \frac{9}{2T\beta^2}\sum_{s=1}^{T-1} s(s+1)^2 \left(\frac{2/s}{1+2/s}\right)^2 \mathbb{E}\left\| x_s- x_{s-1}\right\|^2 \sum_{t=s}^{T-1}\left(\frac{1}{t}-\frac{1}{t+1}\right) \nonumber \\
	&+\frac{3}{T}\sum_{t=1}^{T} \beta^{-2}\mathbb{E}\|x_{t}- x_{t-1}\|^2+ \frac{6G_0}{T} \sum_{s=1}^{T-1} \frac{\tau-1}{\tau} \mathbb{E}\left\|z_{s}- z_{s-1}\right\|^{2}+\frac{6\sigma_{0}^{2}}{|\mathcal{A}_0|} \nonumber \\
	\stackrel{(vii)}{\le}& \frac{18}{T\beta^2}\sum_{s=1}^{T-1} \frac{s(s+1)^2}{(s+2)^2} \mathbb{E}\left\| x_s- x_{s-1}\right\|^2 \left(\frac{1}{s}-\frac{1}{T}\right)+\frac{3}{T\beta^2}\sum_{t=1}^{T} \mathbb{E}\|x_{t}- x_{t-1}\|^2+\frac{6\sigma_{0}^{2}}{|\mathcal{A}_0|} \nonumber \\
	&+\frac{6G_0}{T} \sum_{s=1}^{T-1} \left[\frac{2 \beta^{2}}{\lambda_{s-1}^2}\mathbb{E}\left\|x_{s}-x_{s-1}\right\|^{2}+2 \alpha_{s+1}^{2} \Gamma_{s} \sum_{t=1}^{s} \frac{\left(\lambda_{t-1}-\beta\right)^{2}}{\alpha_{t} \Gamma_{t}\lambda_{t-1}^2}\mathbb{E}\left\|x_t-x_{t-1}\right\|^{2}\right] \nonumber \\
	\stackrel{(viii)}{\le}& \frac{18}{T\beta^2}\sum_{s=1}^{T-1}  \mathbb{E}\left\| x_s- x_{s-1}\right\|^2 + \frac{3}{T\beta^2}\sum_{s=1}^{T} \mathbb{E}\|x_{s}- x_{s-1}\|^2+\frac{6\sigma_{0}^{2}}{|\mathcal{A}_0|} +\frac{12G_0}{T}\sum_{s=1}^{T-1}\mathbb{E}\left\|x_{s}-x_{s-1}\right\|^{2} \nonumber \\ &+\frac{12G_0}{T}\sum_{t=1}^{T-1}\sum_{s=t}^{T-1}\frac{4}{(s+2)^2}\frac{2}{s(s+1)}\frac{t+1}{2}\frac{t(t+1)}{2}\left(\frac{\alpha_{t-1}}{1+\alpha_{t-1}}\right)^2\mathbb{E}\left\|x_t-x_{t-1}\right\|^{2} \nonumber \\
	\stackrel{(ix)}{\le}& \left(\frac{21}{T\beta^2}+\frac{12G_0}{T}\right) \sum_{s=1}^{T} \mathbb{E}\|x_{s}-x_{s-1}\|^2+\frac{6\sigma_{0}^{2}}{|\mathcal{A}_0|} \nonumber \\
	&+\frac{24G_0}{T}\sum_{t=1}^{T-1} \frac{1}{4t(t+1)}
	t(t+1)^2\left(\frac{2}{2+t}\right)^2\mathbb{E}\left\|x_t-x_{t-1}\right\|^{2} \nonumber \\
	\le& \left(\frac{21}{T\beta^2}+\frac{12G_0}{T}\right) \sum_{s=1}^{T} \mathbb{E}\|x_{s}-x_{s-1}\|^2+\frac{6\sigma_{0}^{2}}{|\mathcal{A}_0|}+\frac{6G_0}{T}
	\sum_{t=1}^{T-1}\mathbb{E}\left\|x_t-x_{t-1}\right\|^{2} \nonumber \\
	\stackrel{(x)}{\le}& \left(\frac{21}{T\beta^2}+\frac{20G_0}{T}\right) \left[\frac{8T\beta^2\sigma_{0}^2}{|\mathcal{A}_0|}+8\beta\left[\Phi\left(x_{0}\right)-\mathbb{E}\Phi(x_{T})\right]\right]+\frac{6\sigma_{0}^{2}}{|\mathcal{A}_0|} \nonumber \\
	\stackrel{(xi)}{\le}&\frac{22}{T\beta^2} \left[\frac{8T\beta^2\sigma_{0}^2}{|\mathcal{A}_0|}+8\beta\left[\Phi\left(x_{0}\right)-\mathbb{E}\Phi(x_{T})\right]\right]+\frac{6\sigma_{0}^{2}}{|\mathcal{A}_0|} \nonumber \\
	\le&\frac{182\sigma_{0}^2}{|\mathcal{A}_0|}+\frac{176}{T\beta}\left[\Phi\left(x_{0}\right)-\mathbb{E}\Phi(x_{T})\right], %\numberthis
\end{align}
\noindent where (i) follows from Eq. \eqref{eq_prox_grad}, (ii) uses $x_{t+1}=\operatorname{prox}_{\lambda_{t}r}[ x_{t}-\lambda_{t}{\widetilde{\nabla}F}( z_{t})]$ in Algorithm \ref{alg: 1}, (iii) uses Jensen's inequality and the non-expansive property of proximal operator (See Section 31 of \cite{pryce1973r} for detail) and the fact that $ y_0= x_0$, (iv) uses the inequality $\| u+ v\|^2\le 2\| u\|^2+2\| v\|^2$ and then $z_t-x_t=(1-\alpha_{t+1})(y_t-x_t)$, (v) uses Eq. \eqref{eq_xydiff_bound} and Lemma \ref{lemma_approx_grad}, (vi) swaps the order of the first two summations as well as that of the last two summations and uses hyperparameter choices that $\beta\le\lambda_{t}\le(1+\alpha_{t})\beta\le 2\beta$, $|\mathcal{A}_t|\ge\tau$ and $|\mathcal{A}_{m\tau}|\equiv |\mathcal{A}_0| (m\in\mathbb{N})$, (vii) uses Eq. \eqref{eq_zdiff_bound}, (viii) swaps the last two summations and uses $\lambda_{s-1}\ge\beta>0$, (ix) uses Eq. \eqref{eq_sum_s_bound}, (x) uses Eq. \eqref{sum_xdiff_bound}, (xi) uses the following inequality $$\beta\le\frac{1}{2\sqrt{16L_F^2+6G_0}+8L_F}\le\frac{1}{2\sqrt{6G_0}}\Rightarrow G_0\le\frac{1}{24\beta^2}.$$
\indent In \Cref{thm_alpha_decay_single}, by substituting the hyperparameter choices specified in item 2 of \Cref{thm_alpha_decay_single} for the problem $(\mathbb{E}^1)$ into Eq. \eqref{eq_Egrad_bound2} and using $\Phi(x_T)\ge\Phi^*=\inf_{x}\Phi(x)$, we obtain Eq. \eqref{eq_conclude_alpha_decay_SCE}, which implies $\mathbb{E}_{\zeta}\|{\mathcal{G}}_{\lambda_\zeta}( z_\zeta)\|^2\le \mathcal{O}(\epsilon^2)$ and thus $\mathbb{E}_{\zeta}\|{\mathcal{G}}_{\lambda_\zeta}( z_\zeta)\|\le \mathcal{O}(\epsilon)$ when $T=\mathcal{O}(\epsilon^{-2})$. Then, the number of evaluations of the proximal operator is at most $T=\mathcal{O}(\epsilon^{-2})$. The sample complexity (number of evaluations of $g$, $g'$) is $\mathcal{O}\left(\sum_{t=0}^{T-1} |\mathcal{A}_t|\right)$, where
$$
\sum_{t=0}^{T-1} |\mathcal{A}_t|=\left\lfloor \frac{T-1}{\tau} \right\rfloor\left\lceil\frac{2\sigma_{0}^2}{\epsilon^2}\right\rceil+\left(T-\left\lfloor \frac{T-1}{\tau} \right\rfloor\right)\left\lfloor\sqrt{\frac{2\sigma_{0}^2}{\epsilon^2}}\right\rfloor=\mathcal{O}(\epsilon^{-3}),
$$
\noindent in which $\left\lfloor \frac{T-1}{\tau} \right\rfloor$ is the number of $t$ values that are exactly divisible by $\tau$ and we use $T=\mathcal{O}(\epsilon^{-2})$ and $\tau=\left\lfloor\sqrt{\frac{2\sigma_{0}^2}{\epsilon^2}}\right\rfloor=\mathcal{O}(\epsilon^{-1})$. 

\subsection{Proof of Convergence under Periodic Restart}\label{subsec: restart1}

When using the restart strategy in solving the problem $(\mathbb{E}^1)$, the result in Eq. \eqref{eq_Egrad_bound2} implies that for the $m$-th restart period,
$$\frac{1}{T}\sum_{t=0}^{T-1}\|{\mathcal{G}}_{\lambda_{t}}( z_{t,m})\|^2\le\frac{182\sigma_{0}^2}{|\mathcal{A}_0|}+\frac{176}{T\beta}\left[\mathbb{E}\Phi\left(x_{0,m}\right)-\mathbb{E}\Phi(x_{T,m})\right].$$
Hence,
$$\mathbb{E} \|{\mathcal{G}}_{\lambda_{\zeta, \delta}}( z_{\zeta,\delta})\|^2=\frac{1}{MT}\sum_{m=1}^{M}\sum_{t=0}^{T-1}\|{\mathcal{G}}_{\lambda_{t}}( z_{t,m})\|^2\le\frac{182\sigma_{0}^2}{|\mathcal{A}_0|}+\frac{176}{MT\beta}\left[\Phi\left(x_{0,1}\right)-\mathbb{E}\Phi(x_{T,M})\right]\le \mathcal{O} \Big(\epsilon^2+\frac{\Phi\left(x_{0,1}\right)-\Phi^*}{MT\beta} \Big)
,$$
where we have used the restart strategy $x_{T-1,m}=x_{0,m+1}$ and substitutes in the hyperparameter choices specified in item 2 of \Cref{thm_alpha_decay_single} for the problem $(\mathbb{E}^1)$. By letting $M=\mathcal{O}(\epsilon^{-1})$ and $T=\mathcal{O}(\epsilon^{-1})$, we obtain that  $\mathbb{E}_{\zeta}\|{\mathcal{G}}_{\lambda_\zeta}( z_\zeta)\|^2\le \mathcal{O}(\epsilon^2)$ and thus $\mathbb{E}_{\zeta}\|{\mathcal{G}}_{\lambda_\zeta}( z_\zeta)\|\le \mathcal{O}(\epsilon)$. The sample complexity (number of evaluations of $g$, $g'$) is $\mathcal{O}\left(M\sum_{t=0}^{T-1} |\mathcal{A}_t|\right)$, where
$$
M\sum_{t=0}^{T-1} |\mathcal{A}_t|=M\left\lfloor \frac{T-1}{\tau} \right\rfloor\left\lceil\frac{2\sigma_{0}^2}{\epsilon^2}\right\rceil+M\left(T-\left\lfloor \frac{T-1}{\tau} \right\rfloor\right)\left\lfloor\sqrt{\frac{2\sigma_{0}^2}{\epsilon^2}}\right\rfloor=\mathcal{O}(\epsilon^{-3}),
$$
which gives the same computation complexity as that of \Cref{alg: 1} (without restart) with $T=\mathcal{O}(\epsilon^{-2})$ iterations.

Similarly, for the problem $(\Sigma^1)$, $\sigma_{0}^2=0$ and thus Eq. \eqref{eq_conclude_alpha_decay_SCFS_nonconvex} holds. Then, the sample complexity (number of evaluations of $g$, $g'$) is $\mathcal{O}\left(\sum_{t=0}^{T-1} |\mathcal{A}_t|\right)$, where \indent 
$$
\sum_{t=0}^{T-1} |\mathcal{A}_t|=\left\lfloor \frac{T-1}{\tau} \right\rfloor n+\left(T-\left\lfloor \frac{T-1}{\tau} \right\rfloor\right)\left\lfloor{\sqrt n}\right\rfloor=\mathcal{O}(n+\sqrt{n}\epsilon^{-2}),
$$
\noindent which uses the hyperparameter choices specified in item 1 of \Cref{thm_alpha_decay_single} for the problem $(\Sigma^1)$. When using the restart strategy, following a similar proof as that of the previous proof, we can prove that 
$$\mathbb{E} \|{\mathcal{G}}_{\lambda_{\zeta,\delta}}( z_{\zeta,\delta})\|^2\le \mathcal{O}\Big( \frac{\Phi\left(x_{0,1}\right)-\Phi^*}{MT\beta}\Big),$$
and the complexity is the same as that of \Cref{alg: 1} without restart for solving the problem $(\Sigma^1)$. 
%\indent As to the assumptions, Assumption \ref{assumption1} is always necessary, while Assumption \ref{assumption2} is only used for the problem $(\mathbb{E}^1)$ as indicated in Lemma \ref{lemma_approx_grad}.

\section{Auxiliary Lemmas for Proving Theorem \ref{thm_alpha_const_single}}\label{sec: append: 2}
%The main difference between Theorems \ref{thm_alpha_decay_single} and \ref{thm_alpha_const_single} is that the former one uses the diminishing momentum $\alpha_{t}=\frac{2}{t+1}$ as in \cite{wang2019spiderboost}, while the latter uses the more aggressive constant-level momentum coefficient $\alpha_{t}\equiv\alpha$. Similar to Lemma \ref{lemma_seq_bound_alpha_decay} for Theorem \ref{thm_alpha_decay_single}, we will prove the following lemma for Theorem \ref{thm_alpha_const_single}.
\begin{lemma}\label{lemma_seq_bound_alpha_const}
	Implement algorithm \ref{alg: 1} with $\alpha_{t}\equiv\alpha\in(0,1], \beta_t\equiv\beta, \beta\le\lambda_{t}\le(1+\alpha)\beta$. Then, the generated sequences $\{x_{t}, y_{t}, z_{t}\}$ satisfy the following conditions:
	\begin{equation}\label{eq_xydiff_alpha_const}
		y_{t}-x_{t}=-\sum_{s=1}^{t} (1-\alpha)^{t-s}\frac{\beta-\lambda_{s-1}}{\lambda_{s-1}} (x_s-x_{s-1})
	\end{equation}
	\begin{equation}\label{eq_xydiff_bound_alpha_const}
		\left\|y_{t}-x_{t}\right\|^{2} \leq \frac{t}{t+1}\sum_{s=1}^{t} (1-\alpha)^{2(t-s)}(t-s+1)(t-s+2)\frac{(\beta-\lambda_{s-1})^2}{\lambda_{s-1}^2} \|x_s-x_{s-1}\|^2,
	\end{equation}
	\begin{align}\label{eq_zdiff_bound_alpha_const}
		\left\|z_{t+1}-z_{t}\right\|^{2} \le&\frac{2\beta^2}{\lambda_t^2}\left\|x_{t+1}- x_t \right\|^2+2\alpha^2 \frac{t+1}{t+2}\nonumber\\
		&\sum_{s=1}^{t+1} (1-\alpha)^{2(t-s+1)}(t-s+2)(t-s+3)\frac{(\beta-\lambda_{s-1})^2}{\lambda_{s-1}^2} \|x_s-x_{s-1}\|^2.
	\end{align}
	\noindent Note that when $t=0$, the summation $\sum_{t=1}^{0}$ is 0 by default.
\end{lemma}

\begin{proof}
	Based on the update rules in Algorithm \ref{alg: 1},
	\begin{align}
	y_{t+1}- x_{t+1}=& z_{t}+\frac{\beta}{\lambda_{t}}( x_{t+1}- x_{t})- x_{t+1} \nonumber \\
	=&\left(1-\alpha\right) y_{t}+\alpha  x_{t}+\frac{\beta}{\lambda_{t}}( x_{t+1}- x_{t})- x_{t+1} \nonumber \\
	=&\left(1-\alpha\right) ( y_{t}- x_{t})+\left(\frac{\beta}{\lambda_{t}}-1\right)( x_{t+1}- x_{t}). \nonumber 
	\end{align}
	\indent Taking the above equality as a difference equation about $ y_{t}- x_{t}$ with initial condition $ y_0- x_0= 0$, it can be verified that Eq. \eqref{eq_xydiff_alpha_const} is the solution. Hence,
	\begin{align}
	\| y_{t}- x_t\|^2=&\left\|\sum_{s=1}^{t} (1-\alpha)^{t-s}\frac{\beta-\lambda_{s-1}}{\lambda_{s-1}} (x_s-x_{s-1}) \right\|^2 \nonumber \\
	=&\left\|\sum_{s=1}^{t} \frac{t+1}{t(t-s+1)(t-s+2)} (1-\alpha)^{t-s}\frac{t(t-s+1)(t-s+2)}{t+1}\frac{\beta-\lambda_{s-1}}{\lambda_{s-1}} (x_s-x_{s-1}) \right\|^2 \nonumber \\
	\le&\sum_{s=1}^{t} \frac{t+1}{t(t-s+1)(t-s+2)}\left\| (1-\alpha)^{t-s}\frac{t(t-s+1)(t-s+2)}{t+1}\frac{\beta-\lambda_{s-1}}{\lambda_{s-1}} (x_s-x_{s-1}) \right\|^2 \nonumber \\
	=&\frac{t}{t+1}\sum_{s=1}^{t} (1-\alpha)^{2(t-s)}(t-s+1)(t-s+2)\frac{(\beta-\lambda_{s-1})^2}{\lambda_{s-1}^2} \|x_s-x_{s-1}\|^2, \nonumber 
	\end{align}
	\noindent where $\le$ applies Jensen's inequality to the convex function $\|\bullet\|^2$ with $\sum_{s=1}^{t} \frac{t+1}{t(t-s+1)(t-s+2)}=1$. The coefficient $\frac{t+1}{t(t-s+1)(t-s+2)}$ is different from that in the proof of Lemma \ref{lemma_seq_bound_alpha_decay} to accomodate the constant momentum. This proves Eq. \eqref{eq_xydiff_bound_alpha_const}. 
	
	\indent It can be derived from the updating rules in Algorithm \ref{alg: 1} that $z_{t+1}- z_t= \beta_{t}(x_{t+1}-x_t)/\lambda_{t}-\alpha_{t+2}( y_{t+1}- x_{t+1})$. Hence,
	\begin{align}
	\| z_{t+1}- z_t\|^2=&\left\| \frac{\beta_t}{\lambda_t}(x_{t+1}- x_t)-\alpha_{t+2}( y_{t+1}- x_{t+1}) \right\|^2 \nonumber \\
	\le&2\left\| \frac{\beta}{\lambda_t}(x_{t+1}- x_t) \right\|^2+2\alpha^2\left\|y_{t+1}- x_{t+1}\right\|^2 \nonumber \\
	\le&\frac{2\beta^2}{\lambda_t^2}\left\|x_{t+1}- x_t \right\|^2+2\alpha^2 \frac{t+1}{t+2} \nonumber \\
	&\sum_{s=1}^{t+1} (1-\alpha)^{2(t-s+1)}(t-s+2)(t-s+3)\frac{(\beta-\lambda_{s-1})^2}{\lambda_{s-1}^2} \|x_s-x_{s-1}\|^2, \nonumber 
	\end{align}
	\noindent which proves Eq. \eqref{eq_zdiff_bound_alpha_const}. 
\end{proof}

\section{Proof of Theorem \ref{thm_alpha_const_single}}
Notice that the second last inequality of Eq. \eqref{eq_adjacent_Phi_diff_bound} still holds, that is,
\begin{align}
\mathbb{E}\Phi( x_{t+1})-\mathbb{E}\Phi( x_t)\le& \frac{\lambda_t}{2}\mathbb{E}\left\| \nabla F(z_{t})-{\widetilde{\nabla}F}(z_{t})\right\|^2+\frac{L_F}{2}\mathbb{E} \left\|y_{t}-x_{t}\right\|^2+\left(L_F-\frac{1}{2\lambda_t}\right)\mathbb{E}\| x_{t+1}-x_t\|^2 \nonumber 
\end{align}
\indent By telescoping the above inequality over $t$ from $0$ to $T-1$, we obtain that
\begin{align}\label{eq_Phi_Taylor_alpha_const}
\mathbb{E}\Phi(x_{T})-\Phi(x_0)\le& \sum_{t=0}^{T-1}  \frac{\lambda_t}{2}\mathbb{E}\left\| \nabla F(z_{t})-{\widetilde{\nabla}F}(z_{t})\right\|^2 + \sum_{t=0}^{T-1} \left(L_F-\frac{1}{2\lambda_t}\right)\mathbb{E}\| x_{t+1}- x_t\|^2 \nonumber \\
&+\frac{L_F}{2} \sum_{t=0}^{T-1} \mathbb{E}\| y_t-x_t\|^2 \nonumber \\
\stackrel{(i)}{\le}& \sum_{t=0}^{T-1}  \frac{\lambda_t}{2}\left[G_{0} \sum_{s=\tau\left\lfloor t/\tau\right\rfloor+1}^{t} \frac{1}{|\mathcal{A}_s|} \mathbb{E}\left\| z_{s}- z_{s-1}\right\|^{2}+\frac{\sigma_{0}^{2}}{|\mathcal{A}_{\tau\left\lfloor t/\tau\right\rfloor}|}\right] \nonumber \\
&+\sum_{t=0}^{T-1} \left(L_F-\frac{1}{2\lambda_t}\right)\mathbb{E}\| x_{t+1}- x_t\|^2 \nonumber \\ 
&+\frac{L_F}{2}\sum_{t=0}^{T-1} \frac{t}{t+1}\sum_{s=1}^{t}(1-\alpha)^{2(t-s)}(t-s+1)(t-s+2)\frac{(\beta-\lambda_{s-1})^2}{\lambda_{s-1}^2} \mathbb{E}\|x_s-x_{s-1}\|^2 \nonumber \\
\stackrel{(ii)}{\le}& \frac{G_0}{2} \sum_{s=1}^{T-1} \frac{1}{|\mathcal{A}_s|} \mathbb{E}\left\|z_{s}- z_{s-1}\right\|^{2} \sum_{t=s}^{\min\left(\tau\left\lfloor \frac{s-1}{\tau}\right\rfloor+\tau-1, T-1\right)} \lambda_t+\frac{\sigma_{0}^{2}}{2}\sum_{t=0}^{T-1}\frac{\lambda_t}{|\mathcal{A}_{\tau\left\lfloor t/\tau\right\rfloor}|} \nonumber \\
&+\sum_{t=0}^{T-1} \left(L_F-\frac{1}{2\lambda_t}\right)\mathbb{E}\| x_{t+1}- x_t\|^2+\frac{L_F}{2}\sum_{s=1}^{T-1} \frac{(\lambda_{s-1}-\beta)^2}{\lambda_{s-1}^2} \nonumber \\
&\mathbb{E}\left\|x_s-x_{s-1}\right\|^{2}\sum_{t=s}^{T-1}(1-\alpha)^{2(t-s)}(t-s+1)(t-s+2) \nonumber \\
\stackrel{(iii)}{\le}& \frac{G_0}{2}\sum_{s=1}^{T-1}\frac{1}{\tau}\mathbb{E}\|z_s-z_{s-1}\|^2 \left(\tau\left\lfloor \frac{s-1}{\tau}\right\rfloor+\tau-1-(s-1)\right)(2\beta) \nonumber \\
&+\frac{T\beta\sigma_{0}^2}{|\mathcal{A}_0|}+\sum_{t=0}^{T-1}\left(L_F-\frac{1}{4\beta}\right)\mathbb{E}\|x_{t+1}-x_t\|^2 \nonumber \\
&+\frac{L_F}{\alpha^3(2-\alpha)^3}\sum_{s=1}^{T-1} \frac{\alpha^2}{(1+\alpha)^2}\mathbb{E}\left\|x_s-x_{s-1}\right\|^{2} \nonumber \\
\stackrel{(iv)}{\le}& \beta G_0\sum_{s=1}^{T-1}\left[\frac{2\beta^2}{\lambda_{s-1}^2}\mathbb{E}\left\|x_{s}- x_{s-1} \right\|^2+2\alpha^2 \frac{s}{s+1}\right. \nonumber \\
&\left.\sum_{t=1}^{s} (1-\alpha)^{2(s-t)}(s-t+1)(s-t+2)\frac{(\beta-\lambda_{t-1})^2}{\lambda_{t-1}^2} \mathbb{E}\|x_t-x_{t-1}\|^2\right] \nonumber \\
&+\frac{T\beta\sigma_{0}^2}{|\mathcal{A}_0|}+\sum_{t=0}^{T-1}\left(L_F-\frac{1}{4\beta}\right)\mathbb{E}\|x_{t+1}-x_t\|^2+\frac{L_F}{\alpha}\sum_{t=0}^{T-2} \mathbb{E}\left\|x_{t+1}-x_{t}\right\|^{2} \nonumber \\
\stackrel{(v)}{\le}& 2\beta G_0\sum_{s=1}^{T-1} \mathbb{E}\left\|x_s-x_{s-1} \right\|^2+2\alpha^2\beta G_0 \nonumber \\
&\sum_{t=1}^{T-1} \sum_{s=t}^{T-1} (1-\alpha)^{2(s-t)}(s-t+1)(s-t+2)\left(\frac{\alpha}{1+\alpha}\right)^2
\mathbb{E}\left\|x_t-x_{t-1}\right\|^{2} \nonumber \\
&+\frac{T\beta\sigma_{0}^2}{|\mathcal{A}_0|}+\sum_{t=0}^{T-1}\left(L_F+\frac{L_F}{\alpha}-\frac{1}{4\beta}\right)\mathbb{E}\|x_{t+1}-x_t\|^2 \nonumber \\
\stackrel{(vi)}{\le}& 2\beta G_0\sum_{t=0}^{T-2} \mathbb{E}\left\|x_{t+1}-x_t\right\|^2+ \frac{4\alpha^2\beta G_0}{\alpha^3(2-\alpha)^3}\sum_{t=1}^{T-1} \left(\frac{\alpha}{1+\alpha}\right)^2
\mathbb{E}\left\|x_t-x_{t-1}\right\|^{2} \nonumber \\
&+\frac{T\beta\sigma_{0}^2}{|\mathcal{A}_0|}+\sum_{t=0}^{T-1}\left(L_F+\frac{L_F}{\alpha}-\frac{1}{4\beta}\right)\mathbb{E}\|x_{t+1}-x_t\|^2 \nonumber \\
\stackrel{(vii)}{\le}& \frac{T\beta\sigma_{0}^2}{|\mathcal{A}_0|}+ \left(6\beta G_0+L_F+\frac{L_F}{\alpha}-\frac{1}{4\beta}\right)\sum_{t=0}^{T-1}\mathbb{E}\|x_{t+1}-x_t\|^2,
\end{align}
\noindent where (i) uses Lemma \ref{lemma_approx_grad}(still holds) and Eq. \eqref{eq_xydiff_bound_alpha_const}, (ii) switches the order of the first two summations as well as that of the last two summations, (iii) uses the following inequality \eqref{eq_sum_s_bound_alpha_const} that is not used in the proof of \Cref{thm_alpha_decay_single} for diminishing momentum, and also uses the facts that $\beta\le\lambda_{t}\le(1+\alpha_{t})\beta\le 2\beta$, $|\mathcal{A}_t|\ge\tau$ and $|\mathcal{A}_{m\tau}|\equiv |\mathcal{A}_0| (m\in\mathbb{N})$ for the hyperparameter choices in Theorem \ref{thm_alpha_const_single}, (iv) uses Eq. \eqref{eq_zdiff_bound_alpha_const} and $\alpha\in(0,1]$, (v) switches the first two summations and uses $\beta\le\lambda_{t}\le(1+\alpha_{t})\beta\le 2\beta$, (vi) uses the following Eq. \eqref{eq_sum_s_bound_alpha_const} again, and (vii) uses $\alpha\in(0,1]$.
\begin{align}\label{eq_sum_s_bound_alpha_const}
	&\sum_{t=s}^{T-1}(1-\alpha)^{2(t-s)}(t-s+1)(t-s+2) \nonumber \\
	\le& \sum_{k=0}^{\infty}(1-\alpha)^{2k}(k+1)(k+2)  \nonumber \\
	=&\left.\left(\frac{d^2}{dv^2}\sum_{k=0}^{\infty}v^{k+2}\right)\right|_{v=(1-\alpha)^2} \nonumber \\
	=&\left. \frac{2}{{{{\left( {1 - v} \right)}^3}}} \right|_{v=(1-\alpha)^2} \nonumber \\
	=&\frac{2}{\alpha^3(2-\alpha)^3}.
\end{align}
\noindent Let $6\beta G_0+L_F+\frac{L_F}{\alpha}-\frac{1}{4\beta}\le-\frac{1}{8\beta}$, which together with $\beta>0$ implies 
$$0<\beta\le \frac{1}{{4{\sqrt {{{(1 + 1/\alpha )}^2}L_F^2+3{G_0}} + 4(1 + 1/\alpha ){L_F}}}}.$$ 
Then Eq. \eqref{eq_Phi_Taylor_alpha_const} further indicates that
\begin{equation}\label{sum_xdiff_bound_alpha_const}
\sum_{t=0}^{T-1}\mathbb{E}\|x_{t+1}-x_t\|^2 \le \frac{8T\beta^2\sigma_{0}^2}{|\mathcal{A}_0|}+8\beta\left[\Phi\left(x_{0}\right)-\mathbb{E}\Phi(x_T)\right].
\end{equation}
On the other hand, one can check that (iv) of Eq. \eqref{eq_Egrad_bound2} still holds, and we obtain that
\begin{align}\label{eq_Egrad_bound2_alpha_const}
\mathbb{E} \|{\mathcal{G}}_{\lambda_\zeta}( z_\zeta)\|^2 \le&\frac{9}{T}\sum_{t=0}^{T-1} \lambda_t^{-2}(1-\alpha)^2\mathbb{E} \| y_t- x_t\|^2+\frac{3}{T}\sum_{t=0}^{T-1} \lambda_t^{-2}\mathbb{E}\| x_{t+1}- x_t\|^2 \nonumber \\
&+\frac{6}{T}\sum_{t=0}^{T-1} \mathbb{E}\|\nabla F( z_t)-{\widetilde{\nabla}F}( z_{t})\|^2 \nonumber \\
\stackrel{(i)}{\le}&\frac{9}{T}\sum_{t=1}^{T-1}
\frac{t\lambda_t^{-2}(1-\alpha)^2}{t+1}\sum_{s=1}^{t} (1-\alpha)^{2(t-s)}(t-s+1)(t-s+2)\frac{(\beta-\lambda_{s-1})^2}{\lambda_{s-1}^2} \mathbb{E}\|x_s-x_{s-1}\|^2 \nonumber \\
&+\frac{3}{T}\sum_{t=1}^{T} \lambda_{t-1}^{-2}\mathbb{E}\| x_{t}- x_{t-1}\|^2 +\frac{6}{T}\sum_{t=0}^{T-1} \left[G_{0} \sum_{s=\tau\left\lfloor t/\tau\right\rfloor+1}^{t} \frac{1}{|\mathcal{A}_s|} \mathbb{E}\left\| z_{s}- z_{s-1}\right\|^{2}+\frac{\sigma_{0}^{2}}{|\mathcal{A}_{\tau\left\lfloor t/\tau\right\rfloor}|}\right] \nonumber \\
\stackrel{(ii)}{\le}&\frac{9}{T}\beta^{-2}(1-\alpha)^2\frac{\alpha^2}{(1+\alpha)^2}\sum_{s=1}^{T-1} \mathbb{E}\|x_s-x_{s-1}\|^2 \sum_{t=s}^{T-1}  (1-\alpha)^{2(t-s)}(t-s+1)(t-s+2) \nonumber \\
&+\frac{3}{T}\sum_{t=1}^{T} \beta^{-2}\mathbb{E}\|x_{t}- x_{t-1}\|^2 \nonumber \\
&+\frac{6G_0}{T} \sum_{s=1}^{T-1} \frac{1}{\tau} \mathbb{E}\left\|z_{s}- z_{s-1}\right\|^{2} \sum_{t=s}^{\min\left(\tau\left\lfloor \frac{s-1}{\tau}\right\rfloor+\tau-1, T-1\right)} 1+\frac{6\sigma_{0}^{2}}{|\mathcal{A}_0|} \nonumber \\
\stackrel{(iii)}{\le}& \frac{18\beta^{-2}(1-\alpha)^2}{T\alpha^3(2-\alpha)^3}\frac{\alpha^2}{(1+\alpha)^2}\sum_{s=1}^{T-1} \mathbb{E}\|x_s-x_{s-1}\|^2+\frac{3}{T}\sum_{t=1}^{T} \beta^{-2}\mathbb{E}\|x_{t}- x_{t-1}\|^2 \nonumber \\
&+\frac{6G_0}{T} \sum_{s=1}^{T-1} \frac{\tau-1}{\tau} \mathbb{E} \left[\frac{2\beta^2}{\lambda_{s-1}^2}\left\|x_{s}- x_{s-1} \right\|^2+2\alpha^2 \frac{s}{s+1}\right. \nonumber \\
&\left.\sum_{t=1}^{s} (1-\alpha)^{2(s-t)}(s-t+1)(s-t+2)\frac{(\beta-\lambda_{t-1})^2}{\lambda_{t-1}^2}  \mathbb{E}\|x_t-x_{t-1}\|^2\right] +\frac{6\sigma_{0}^{2}}{|\mathcal{A}_0|} \nonumber \\
\stackrel{(iv)}{\le}& \frac{18}{\alpha T\beta^2}\sum_{s=1}^{T-1} \mathbb{E}\left\|x_s- x_{s-1}\right\|^2  +\frac{3}{T\beta^2}\sum_{t=1}^{T} \mathbb{E}\|x_{t}- x_{t-1}\|^2+\frac{12G_0}{T} \sum_{s=1}^{T-1} \mathbb{E}\left\|x_{s}-x_{s-1}\right\|^{2}  \nonumber \\
&+\frac{12\alpha^2G_0}{T}\frac{\alpha^2}{(1+\alpha)^2}\sum_{t=1}^{T-1} \mathbb{E}\|x_{t}- x_{t-1}\|^2 \sum_{s=t}^{T-1} (1-\alpha)^{2(s-t)}(s-t+1)(s-t+2)+\frac{6\sigma_{0}^{2}}{|\mathcal{A}_0|} \nonumber \\
\stackrel{(v)}{\le}& \left(\frac{18/\alpha+3}{T\beta^2}+\frac{12G_0}{T}\right)\sum_{s=1}^{T}  \mathbb{E}\left\| x_s- x_{s-1}\right\|^2 \nonumber \\
&+\frac{12\alpha^2G_0}{T}\frac{\alpha^2}{(1+\alpha)^2}\frac{2}{\alpha^3(2-\alpha)^3}\sum_{t=1}^{T-1} \mathbb{E}\|x_{t}- x_{t-1}\|^2+\frac{6\sigma_{0}^{2}}{|\mathcal{A}_0|}  \nonumber \\
\stackrel{(vi)}{\le}& \left(\frac{18/\alpha+3}{T\beta^2}+\frac{24G_0}{T}\right)\sum_{t=0}^{T-1} \mathbb{E}\left\| x_{t+1}- x_{t}\right\|^2+\frac{6\sigma_{0}^{2}}{|\mathcal{A}_0|} \nonumber \\
\stackrel{(vii)}{\le}& \frac{18/\alpha+7/2}{T\beta^2} \left[\frac{8T\beta^2\sigma_{0}^2}{|\mathcal{A}_0|}+8\beta\left[\Phi\left(x_{0}\right)-\mathbb{E}\Phi(x_T)\right]\right] +\frac{6\sigma_{0}^{2}}{|\mathcal{A}_0|} \nonumber \\
\le& \frac{(144/\alpha+34)\sigma_{0}^2}{|\mathcal{A}_0|}+\frac{144/\alpha+28}{T\beta}\left[\Phi\left(x_{0}\right)-\mathbb{E}\Phi(x_T)\right], 
\end{align}
\noindent where (i) uses Eqs. \eqref{eq_xydiff_bound_alpha_const} and Lemma \ref{lemma_approx_grad}, (ii) swaps the order of the first two summations as wells as that of the last two summations and uses the parameter choices that $\beta\le\lambda_{t}\le(1+\alpha_{t})\beta$, $|\mathcal{A}_t|\ge\tau$ and $|\mathcal{A}_{m\tau}|\equiv |\mathcal{A}_0| (m\in\mathbb{N})$, (iii) follows from Eqs.  \eqref{eq_zdiff_bound_alpha_const} $\&$ \eqref{eq_sum_s_bound_alpha_const}, (iv) uses the facts that $\alpha\in(0,1]$, $\beta\le\lambda_{t}\le(1+\alpha_{t})\beta$ and swaps the last two summations, (v) uses Eq. \eqref{eq_sum_s_bound_alpha_const}, (vi) uses $\alpha\in(0,1]$ and replaces $s$ and $t$ with $t+1$ in the two summations, (vii) uses Eq. \eqref{sum_xdiff_bound_alpha_const} and the following inequality
$$\beta\le \frac{1}{{4{\sqrt {{{(1 + 1/\alpha )}^2}L_F^2+3{G_0}} + 4(1 + 1/\alpha ){L_F}}}}\le\frac{1}{4\sqrt{3G_0}}\Rightarrow G_0\le\frac{1}{48\beta^2}.$$
\indent In Theorem \ref{thm_alpha_const_single}, by substituting the hyperparameter choices specified in item 1 into Eq. \eqref{eq_Egrad_bound2_alpha_const} and noting that $\Phi(x_T)\ge\Phi^*=\inf_{x}\Phi(x)$ and $\sigma_{0}^2=0$, we obtain Eq. \eqref{eq_conclude_alpha_const_SCFS}, which implies that $\mathbb{E}_{\zeta}\|{\mathcal{G}}_{\lambda_\zeta}( z_\zeta)\|^2\le \epsilon^2$ when $T=\mathcal{O}(\epsilon^{-2})$. Similarly, by substituting the hyperparameter choices specified in item 2, we obtain Eq. \eqref{eq_conclude_alpha_const_SCE}, which implies that $\mathbb{E}_{\zeta}\|{\mathcal{G}}_{\lambda_\zeta}( z_\zeta)\|^2\le \mathcal{O}(\epsilon^2)$ when $T=\mathcal{O}(\epsilon^{-2})$. Notice that Theorems \ref{thm_alpha_decay_single} $\&$ \ref{thm_alpha_const_single} share the same assumptions, very similar hyperparameter choices and the same order of upper bounds for $\mathbb{E}_{\zeta}\|{\mathcal{G}}_{\lambda_\zeta}( z_\zeta)\|^2$. Hence, the same sample complexity result for solving both problems $(\Sigma^1)$ and $(\mathbb{E}^1)$ in Theorem \ref{thm_alpha_const_single} can be proved in the same way as that in Theorem \ref{thm_alpha_decay_single}. 
%\\ \indent In theorem \ref{thm_alpha_const_SCFS_nonconvex}, by substituting the hyperparameter choices and $\sigma_{0}^2=0$ (for finite sum problem \eqref{singleCsum}) \eqref{hyperpar_alpha_const_SCFS_nonconvex} into Eq. \eqref{eq_Egrad_bound2_alpha_const} and using $\Phi(x_T)\ge\Phi^*=\inf_{x}\Phi(x)$, we obtain Eq. \eqref{eq_conclude_alpha_const_SCFS_nonconvex}, which implies $\mathbb{E}_{\xi}\|{\mathcal{G}}_{\lambda_\xi}( z_\xi)\|^2\le \epsilon^2$ when $T=\mathcal{O}(\epsilon^{-2})$. Notice that theorems \ref{thm_alpha_decay_SCFS_nonconvex} $\&$ \ref{thm_alpha_const_SCFS_nonconvex} share the same assumptions, and very similar hyperparameter choices and upper bounds of $\mathbb{E}_{\xi}\|{\mathcal{G}}_{\lambda_\xi}( z_\xi)\|^2$ except for some different constants. Hence, Eq. \eqref{eq_conclude_alpha_const_SCFS_nonconvex_restart} as well as the time complexity (for both the case of implementing algorithm \ref{alg_single} once and the case of restart strategy) in theorem \ref{thm_alpha_const_SCFS_nonconvex} can be proved in almost the same way as theorem \ref{thm_alpha_decay_SCFS_nonconvex}.

\section{Proof of Theorems \ref{thm_graddom_singleC} \& \ref{thm_graddom_const_momentum}} 
\indent The two items in \Cref{thm_graddom_singleC} and those in \Cref{thm_graddom_const_momentum} can be proved in the same way, so we will only focus on the proof of the item 2 of \Cref{thm_graddom_singleC}.

\indent Since $r\equiv 0$, it holds that
\begin{equation}
	{\mathcal{G}}_{\lambda}(x)=\nabla F(x).
\end{equation}
\indent Hence, the gradient dominant condition \eqref{eq_graddom} implies that
\begin{equation}\label{eq_graddom_E}
	\mathbb{E} F( z_{\zeta})-F^* \leq \frac{v}{2}\mathbb{E}\left\|{\mathcal{G}}_{\lambda}( z_{\zeta})\right\|^{2}.
\end{equation}
\indent In the item 2 of Theorem \ref{thm_graddom_singleC} for solving the online problem $(\mathbb{E}^1)$, Eq. \eqref{eq_conclude_alpha_decay_SCE} still holds because the hyperparameter choices are the same as the item 2 of \Cref{thm_alpha_decay_single} (Notice $F=\Phi$ now). Hence, Eq. \eqref{eq_conclude_alpha_decay_SCE_graddom} can be obtained by substituting Eq. \eqref{eq_conclude_alpha_decay_SCE} into Eq. \eqref{eq_graddom_E}. When restarting \Cref{alg: 1} $M$ times with restart strategy, Eq. \eqref{eq_conclude_alpha_decay_SCE_graddom} becomes
\begin{equation}\label{tmp1}
	\frac{1}{T}\sum_{t=0}^{T-1} F({z_{t,m+1}}) - {F^*} \leqslant Cv{\epsilon ^2} + \frac{Cv}{T\beta}\left[ {F\left( {{x_{0,m+1}}} \right) - {F^*}} \right],
\end{equation}
\noindent for $m=0,1,\ldots,M-1$ and a constant $C>0$. Since $x_{0,m+1}$ is randomly selected from $\{z_{t,m}\}_{t=0}^{T-1}$, $\mathbb{E} F(x_{0,m+1})=\frac{1}{T} \sum_{t=0}^{T-1} \mathbb{E}F(z_{t,m})$. Hence, by taking expectation of \eqref{tmp1} and rearranging the equation, we obtain
\begin{align}\label{tmp2}
	&\frac{1}{T}\sum_{t=0}^{T-1} \mathbb{E}F({z_{t,m+1}}) - {F^*} \leqslant Cv{\epsilon ^2} + \frac{Cv}{T\beta}\left[ \frac{1}{T} \sum_{t=0}^{T-1} \mathbb{E}F(z_{t,m}) - {F^*} \right]\nonumber\\
	\Rightarrow&\frac{1}{T}\sum_{t=0}^{T-1} \mathbb{E}F({z_{t,m+1}}) - {F^*}- W\epsilon ^2 \leqslant\frac{{Cv}}{T\beta}\left[ \frac{1}{T} \sum_{t=0}^{T-1} \mathbb{E}F(z_{t,m}) - {F^*}- W\epsilon ^2 \right]\nonumber\\
	\Rightarrow&\frac{1}{T}\sum_{t=0}^{T-1} \mathbb{E}F({z_{t,M}}) - {F^*}- W\epsilon ^2 \leqslant \left[\frac{{Cv}}{T\beta}\right]^{M-1} \left[\frac{1}{T} \sum_{t=0}^{T-1} \mathbb{E}F(z_{t,1}) - {F^*}- W\epsilon ^2 \right]\nonumber\\
	\Rightarrow&\frac{1}{T}\sum_{t=0}^{T-1} \mathbb{E}F({z_{t,M}}) - {F^*}- W\epsilon ^2 \leqslant \left[\frac{{Cv}}{T\beta}\right]^{M} \left(F(x_{0,1}) - {F^*}- W\epsilon ^2 \right),
\end{align}
\noindent where we denote
$$W:=\frac{Cv}{1-\frac{Cv}{T\beta}}.$$
\indent Taking $M=\lceil\log_2(\epsilon^{-2})\rceil=\mathcal{O}(\log\frac{1}{\epsilon})$ and $T=\max\left(\lceil Cv/(2\beta)\rceil, \tau\right)=\max\left(\mathcal{O}(v), \tau\right)\le\mathcal{O}(v/\epsilon)$ ($\max$ is added to ensure $T\ge \tau$), Eq. \eqref{tmp2} becomes
\begin{align}
	\mathbb{E}_{\xi} F({z_{\xi,M}}) - {F^*}- W\epsilon ^2 \leq \frac{1}{2^M} \left( F(x_{0,1}) - {F^*}- W\epsilon ^2 \right) \leq  \epsilon^2 \left( F(x_{0,1}) - {F^*}\right),\nonumber
\end{align}
which implies $\mathbb{E}_{\xi} F({z_{\xi,M}}) - {F^*} \le \mathcal{O}(\epsilon^2)$ since $W=\mathcal{O}(v)$ for $T=\max\left(\lceil Cv/(2\beta)\rceil, \tau\right)$. Hence, the required sample complexity is $\mathcal{O}(M\sum_{t=0}^{T-1} A_t)$, where
$$
	M\sum_{t=0}^{T-1} A_t=M\left\lfloor \frac{T-1}{\tau} \right\rfloor\left\lceil\frac{2\sigma_{0}^2}{\epsilon^2}\right\rceil+M\left(T-\left\lfloor \frac{T-1}{\tau} \right\rfloor\right)\left\lfloor\sqrt{\frac{2\sigma_{0}^2}{\epsilon^2}}\right\rfloor=\mathcal{O}(v\epsilon^{-2}\log\epsilon^{-1}).
$$

%\noindent It is easy to verify that $M=\mathcal{O}(\log\frac{1}{\epsilon})$, $T=\max\left(\mathcal{O}(v), \tau\right)\le\mathcal{O}(v/\epsilon)$ and the same time complexity also work for case 3. However, for case 2 and 4, the values of $\tau$ and $A_t$ changes to 
%\begin{equation}\label{tmp4}
%	\tau=\left\lfloor{\sqrt n}\right\rfloor, A_t=\left\{ \begin{gathered} n; t \in\{\widetilde{\tau}_m\}  \hfill \\
%	\left\lfloor{\sqrt n}\right\rfloor=\tau; \text{Otherwise} \hfill \\ 
%	\end{gathered} \right.
%\end{equation}
%respectively, and thus $T=\max\left(\mathcal{O}(v), \tau\right)\le\mathcal{O}(v+\sqrt{n})$ while $M=\mathcal{O}(\log\epsilon^{-1})$ still holds. By substituing Eq. \eqref{tmp4} into $M\sum_{t=0}^{T-1} A_t$, the number of evalutions for function $f$ and that of proximal problem \eqref{eq_updatex} become $MT\le\mathcal{O}((\sqrt{n}+v)\log\epsilon^{-1})$, and the number of evaluations for functions $g$ and that of $g'$ become $\mathcal{O}(n+\sqrt{n}v)\log\epsilon^{-1}$.

\section{Auxiliary Lemmas for Proving Theorem \ref{thm_alpha_decay_DE_nonconvex}}
\begin{lemma}\label{lemma_seq_bound_alpha_decay_double}
	Implement algorithm \ref{alg_double} with
	$$\alpha_{t}=\frac{2}{t+1}, \beta_t\equiv\beta, \beta\le\lambda_{t}\le(1+\alpha_t)\beta $$
%	\noindent which is shared by by both hyperparameter choices \eqref{hyperpar_alpha_decay_DE_nonconvex} (in theorem \ref{thm_alpha_decay_DE_nonconvex}) and \eqref{hyperpar_alpha_decay_DFS_nonconvex} (in theorem \ref{thm_alpha_decay_DFS_nonconvex}).
	
	The generated sequences $\{x_{t}, y_{t}, z_{t}\}$ satisfy the following conditions:
	\begin{equation}\label{eq_xdiff_bound_double}
		\left\|x_{t+1}-x_{t}\right\| \leq \epsilon\lambda_{t} \le 2\beta\epsilon
	\end{equation}
	\begin{equation}\label{eq_xydiff_bound_double}
		\left\|y_{t}-x_{t}\right\|^{2} \leq 4\beta^{2}\epsilon^2,
	\end{equation}
	\begin{equation}\label{eq_xzdiff_bound_double}
		\left\|z_{t}-x_{t}\right\|^{2} \leq 4\beta^{2}\epsilon^2,
	\end{equation}
	\begin{equation}\label{eq_zdiff_bound_double}
		\left\|z_{t+1}-z_{t}\right\|^{2} \leq \frac{104}{9}\beta^2\epsilon^2
	\end{equation}
	\noindent where $\Gamma_t=\frac{2}{t(t+1)}$
	When $t=0$, the summation $\sum_{t=1}^{0}$ is 0 by default.
\end{lemma}

\begin{proof}
	\indent Eq. \eqref{eq_xdiff_bound_double} can be directly derived from the following equation in \Cref{alg_double}. 
	\begin{align}
		x_{t+1}=(1-\theta_t)x_t+\theta_{t} \widetilde{x}_{t+1},~
		\theta_t=\min\left\{\frac{\epsilon\lambda_{t}}{\|\widetilde{x}_{t+1}-x_t\|},\frac{1}{2}\right\}.\nonumber
	\end{align}
	\indent Eqs. \eqref{eq_xydiff}-\eqref{eq_zdiff_bound} in Lemma \ref{lemma_seq_bound_alpha_decay} still hold because they are derived from $z_{t}=\left(1-\alpha_{t+1}\right)  y_{t}+\alpha_{t+1}x_{t}$ and $y_{t+1}= z_{t}+\frac{\beta_{t}}{\lambda_{t}}(x_{t+1}-x_{t})$ that are shared by both Algorithms \ref{alg: 1} \& \ref{alg_double}. Hence, it can be derived from Eqs. \eqref{eq_xydiff_bound} \& \eqref{eq_xdiff_bound_double} that
	%they are derived from Eqs. \eqref{eq_updatez}$\&$\eqref{eq_updatey} in Algorithm \ref{alg_single}, the same as Eqs. \eqref{eq_updatez_double}$\&$\eqref{eq_updatey_double} in Algorithm \ref{alg_double}. 
	\begin{align}
		\left\|y_{t}-x_{t}\right\|^{2} \leq& \Gamma_{t} \sum_{s=1}^{t} \frac{(\lambda_{s-1}-\beta)^2}{\alpha_{s} \Gamma_{s}\lambda_{s-1}^2}\left\|x_s-x_{s-1}\right\|^{2}\nonumber \\
		\le& \frac{2}{t(t+1)}\sum_{s=1}^{t} \frac{s+1}{2}\frac{s(s+1)}{2}\left(1-\frac{\beta}{\lambda_{s-1}}\right)^2(4\beta^2\epsilon^2)\nonumber \\
		\le& \frac{2\beta^2\epsilon^2}{t(t+1)}\sum_{s=1}^{t} s(s+1)^2 \left(1-\frac{\beta}{(1+\alpha_{s-1})\beta}\right)^2\nonumber \\
		\le& \frac{2\beta^2\epsilon^2}{t(t+1)}\sum_{s=1}^{t} s(s+1)^2 \left(\frac{2}{s+2}\right)^2\nonumber \\
		\le& \frac{8\beta^2\epsilon^2}{t(t+1)}\sum_{s=1}^{t} s\nonumber\\
		\le& 4\beta^2\epsilon^2,\nonumber
	\end{align}
	\noindent which proves Eq. \eqref{eq_xydiff_bound_double}. Then, it can be derived Eq. \eqref{eq_xydiff_bound_double}, $z_{t}=\left(1-\alpha_{t+1}\right) y_{t}+\alpha_{t+1} x_{t}$ and $\alpha_{t+1}\in(0, 1]$ that $\left\|z_{t}-x_{t}\right\|^{2}=(1-\alpha_{t+1})^2\|y_t-x_t\|^2 \le 4\beta^{2}\epsilon^2$, which proves Eq. \eqref{eq_xzdiff_bound_double}. Finally, it can be derived from Eqs. \eqref{eq_zdiff_bound} \& \eqref{eq_xdiff_bound_double} that
	\begin{align}
		\left\|z_{t+1}-z_{t}\right\|^{2} \leq& \frac{2 \beta^{2}}{\lambda_{t}^2}\left\|x_{t+1}-x_{t}\right\|^{2}+2 \alpha_{t+2}^{2} \Gamma_{t+1} \sum_{s=1}^{t+1} \frac{\left(\lambda_{s-1}-\beta\right)^{2}}{\alpha_{s} \Gamma_{s}\lambda_{s-1}^2}\left\|x_s-x_{s-1}\right\|^{2}\nonumber \\
		\le& 2(4\beta^2\epsilon^2)+2\frac{4}{(t+3)^2}\frac{2}{(t+1)(t+2)} \sum_{s=1}^{t+1} \frac{s+1}{2}\frac{s(s+1)}{2}\left(1-\frac{\beta}{\lambda_{s-1}}\right)^2(4\beta^2\epsilon^2)\nonumber \\
		\le& 8\beta^2\epsilon^2+\frac{16\beta^2\epsilon^2}{(t+1)(t+2)(t+3)^2} \sum_{s=1}^{t+1} s(s+1)^2\left(1-\frac{\beta}{(1+\alpha_{s-1})\beta}\right)^2\nonumber \\
		\le& 8\beta^2\epsilon^2+\frac{16\beta^2\epsilon^2}{(t+1)(t+2)(t+3)^2} \sum_{s=1}^{t+1} s(s+1)^2\left(\frac{2}{s+2}\right)^2\nonumber \\
		\le& 8\beta^2\epsilon^2+\frac{64\beta^2\epsilon^2}{(t+1)(t+2)(t+3)^2} \sum_{s=1}^{t+1} s\nonumber \\
		\le& 8\beta^2\epsilon^2+\frac{32\beta^2\epsilon^2}{(t+3)^2}\nonumber\\
		\le&\frac{104}{9}\beta^2\epsilon^2,\nonumber
	\end{align}
	\noindent which proves Eq. \eqref{eq_zdiff_bound_double}.
\end{proof}

\begin{lemma}\label{lemma_approx_grad_double}
	Let Assumptions \ref{assumption4} and \ref{assumption5} hold and apply \Cref{alg_double} to solve the problems $(\Sigma^2)$ and $(\mathbb{E}^2)$, with the hyperparameter choices in items 1 and 2 of \Cref{thm_alpha_decay_DE_nonconvex} respectively. Then, the variance of the stochastic gradient satisfies
	\begin{equation}\label{eq_approx_grad_double}
		\mathbb{E}\left\|\widetilde{\nabla} F\left(z_{t}\right)-\nabla F\left(z_{t}\right)\right\|^{2} \le\epsilon^2.
	\end{equation}
\end{lemma}

\begin{proof}
	\indent The proof is similar to that of Lemma 4.2 in \cite{zhang2019multi}. 
	
	\indent In \Cref{alg_double}, $\widetilde{g}_{t}= \widetilde{g}_{t-1}+\frac{1}{|\mathcal{A}_t|}\sum_{\xi\in\mathcal{A}_t}\big( g_{\xi}( z_t)- g_{\xi}( z_{t-1})\big)$ for $t\mod\tau\ne0$. Hence, we get the following two inequalities.
	\begin{equation}\label{eq_inner_g_diff_double}
		\|\widetilde{g}_{t}-\widetilde{g}_{t-1}\| \le  \frac{1}{|\mathcal{A}_t|}\sum_{\xi\in\mathcal{A}_t} \|g_{\xi}(z_t)- g_{\xi}(z_{t-1})\| \le \frac{1}{|\mathcal{A}_t|}\sum_{\xi\in\mathcal{A}_t} l_g\|z_t-z_{t-1}\| \le \frac{2\sqrt{26}}{3}l_g\beta\epsilon,
	\end{equation}
	\noindent where the last step uses Eq. \eqref{eq_zdiff_bound_double}, and
	\begin{align}\label{eq_inner_g_err_double_short}
		\mathbb{E}\|\widetilde{g}_{t}-g(z_t)\|^2=&\mathbb{E}\left\|\widetilde{g}_{t-1}+\frac{1}{|\mathcal{A}_t|}\sum_{\xi\in\mathcal{A}_t}[ g_{\xi}( z_t)- g_{\xi}( z_{t-1})]-g(z_t)+g(z_{t-1})-g(z_{t-1})\right\|^2 \nonumber \\
		=&\mathbb{E}\left\|\widetilde{g}_{t-1}-g(z_{t-1})\right\|^2+\mathbb{E}\left\|\frac{1}{|\mathcal{A}_t|}\sum_{\xi\in\mathcal{A}_t}[g_{\xi}( z_t)- g_{\xi}( z_{t-1})]-[g(z_t)-g(z_{t-1})]\right\|^2\nonumber \\
		&+2\mathbb{E}\left< \widetilde{g}_{t-1}-g(z_{t-1}), \frac{1}{|\mathcal{A}_t|}\sum_{\xi\in\mathcal{A}_t}[g_{\xi}( z_t)- g_{\xi}( z_{t-1})]-[g(z_t)-g(z_{t-1})] \right>\nonumber \\
		\stackrel{(i)}{=}&\mathbb{E}\left\|\widetilde{g}_{t-1}-g(z_{t-1})\right\|^2+\frac{1}{|\mathcal{A}_t|^2}\sum_{\xi\in\mathcal{A}_t}\mathbb{E}\left\|[g_{\xi}( z_t)- g_{\xi}( z_{t-1})]-[g(z_t)-g(z_{t-1})]\right\|^2\nonumber \\
		\stackrel{(ii)}{\le}&\mathbb{E}\left\|\widetilde{g}_{t-1}-g(z_{t-1})\right\|^2+\frac{1}{|\mathcal{A}_t|^2}\sum_{\xi\in\mathcal{A}_t}\mathbb{E}\left\|g_{\xi}( z_t)- g_{\xi}( z_{t-1})\right\|^2\nonumber \\
		\le&\mathbb{E}\left\|\widetilde{g}_{t-1}-g(z_{t-1})\right\|^2+ \frac{1}{|\mathcal{A}_t|^2}\sum_{\xi\in\mathcal{A}_t}\mathbb{E}\left(l_g^2\|z_t-z_{t-1}\|^2\right)\nonumber \\
		\stackrel{(iii)}{\le}&\mathbb{E}\left\|\widetilde{g}_{t-1}-g(z_{t-1})\right\|^2+ \frac{104l_g^2}{9|\mathcal{A}_t|}\beta^2\epsilon^2,
	\end{align}	
	\noindent where (i) uses the facts that $\mathbb{E}_{\xi}\left\{\frac{1}{A_t}\sum_{\xi\in\mathcal{A}_t}[g_{\xi}( z_t)-g_{\xi}(z_{t-1})]\right\}=g(z_t)-g(z_{t-1})$ and that different $\xi, \xi'\in \mathcal{A}_t$ are independent; (ii) uses the inequality $\mathbb{E}\|X-EX\|^2\le \mathbb{E}\|X\|^2$ for any random vector $X$; (iii) uses Eq. \eqref{eq_zdiff_bound_double}. By telescoping Eq. \eqref{eq_inner_g_err_double_short}, we obtain
	\begin{equation}\label{eq_inner_g_err_double}
		\mathbb{E}\|\widetilde{g}_{t}-g(z_t)\|^2 \le \mathbb{E}\|\widetilde{g}_{\tau\lfloor t/\tau\rfloor}-g(z_{\tau\lfloor t/\tau\rfloor})\|^2+\frac{104l_g^2}{9}\beta^2\epsilon^2 \sum_{s=\tau\lfloor t/\tau\rfloor+1}^{t} \frac{1}{|\mathcal{A}_s|}
	\end{equation}
	\indent Similarly, since $\widetilde{g}_{t}'= \widetilde{g}_{t-1}'+\frac{1}{|\mathcal{A}_t|}\sum_{\xi\in\mathcal{A}_t}\big(g_{\xi}'(z_t)-g_{\xi}'( z_{t-1})\big)$ for $t\mod\tau\ne0$ in \Cref{alg_double}, 
	\begin{equation}\label{eq_inner_ggrad_diff_double}
		\|\widetilde{g}_{t}'-\widetilde{g}_{t-1}'\| \le \frac{2\sqrt{26}}{3}L_g\beta\epsilon.
	\end{equation}
	and that
	\begin{equation}\label{eq_inner_ggrad_err_double}
		\mathbb{E}\|\widetilde{g}_{t}'-g'(z_t)\|^2 \le \mathbb{E}\|\widetilde{g}_{\tau\lfloor t/\tau\rfloor}'-g'(z_{\tau\lfloor t/\tau\rfloor})\|^2+\frac{104L_g^2}{9}\beta^2\epsilon^2 \sum_{s=\tau\lfloor t/\tau\rfloor+1}^{t} \frac{1}{|\mathcal{A}_s'|}.
	\end{equation}
	\indent Since $\widetilde{f}_{t}'=\widetilde{f}_{t-1}'+\frac{1}{|\mathcal{B}_t|}\sum_{\eta\in\mathcal{B}_t}\big(\nabla f_{\eta}(\widetilde{g}_t)-\nabla f_{\eta}(\widetilde{g}_{t-1})\big)$ for $t\mod\tau\ne0$ in \Cref{alg_double}, it can be derived in a similar way that
	\begin{align}\label{eq_inner_fgrad_err_double}
		\mathbb{E}\|\widetilde{f}_{t}'-\nabla f(\widetilde{g}_t)\|^2 \le& \mathbb{E}\|\widetilde{f}_{\tau\lfloor t/\tau\rfloor}'-\nabla f(\widetilde{g}_{\tau\lfloor t/\tau\rfloor})\|^2 + \sum_{s=\tau\lfloor t/\tau\rfloor+1}^{t} \frac{1}{|\mathcal{B}_s|^2}\sum_{\eta\in\mathcal{B}_s}\mathbb{E}\left\|\nabla f_{\eta}(\widetilde{g}_s)-\nabla f_{\eta}(\widetilde{g}_{s-1})\right\|^2\nonumber \\
		\le& \mathbb{E}\|\widetilde{f}_{\tau\lfloor t/\tau\rfloor}'-\nabla f(\widetilde{g}_{\tau\lfloor t/\tau\rfloor})\|^2 + \sum_{s=\tau\lfloor t/\tau\rfloor+1}^{t} \frac{L_f^2}{|\mathcal{B}_s|}\mathbb{E}\left\|\widetilde{g}_s-\widetilde{g}_{s-1}\right\|^2\nonumber \\
		\le& \mathbb{E}\|\widetilde{f}_{\tau\lfloor t/\tau\rfloor}'-\nabla f(\widetilde{g}_{\tau\lfloor t/\tau\rfloor})\|^2 + \frac{104}{9}L_f^2 l_g^2 \beta^2 \epsilon^2 \sum_{s=\tau\lfloor t/\tau\rfloor+1}^{t} \frac{1}{|\mathcal{B}_s|}. 
	\end{align}
	\indent Furthermore, %it can be obtained by telescoping Eq. \eqref{eq_inner_ggrad_err_double} that
	\begin{align} \label{eq_inner_ggrad_bound_double}
		\|\widetilde{g}_{t}'\| \le& \|\widetilde{g}_{\tau\lfloor t/\tau\rfloor}'\|+\sum_{s=\tau\lfloor t/\tau\rfloor+1}^{t} \|\widetilde{g}_{s}'-\widetilde{g}_{s-1}'\|\nonumber \\ 
		\le& \left\|\frac{1}{|\mathcal{A}_{\tau\lfloor t/\tau\rfloor}'|}\sum_{\xi\in\mathcal{A}_{\tau\lfloor t/\tau\rfloor}'} g_{\xi}'(z_{\tau\lfloor t/\tau\rfloor})\right\|+\frac{2\sqrt{26}}{3}L_g\beta\epsilon(t-\tau\lfloor t/\tau\rfloor)\nonumber \\
		\le& l_g+\frac{2\sqrt{26}}{3}\tau L_g\beta\epsilon, 
	\end{align}
	\noindent where the second $\le$ uses Eq. \eqref{eq_inner_ggrad_diff_double} and $\widetilde{g}_{t}'=\frac{1}{|\mathcal{A}_t'|}\sum_{\xi\in\mathcal{A}_{t}'} g_{\xi}'(z_t)$ for $t\mod\tau=0$ in \Cref{alg_double}.
	
	\indent Therefore, 
	\begin{align}\label{Fgrad_err_derive_double}
		&\mathbb{E}\|\widetilde{\nabla} F\left(z_{t}\right)-\nabla F\left(z_{t}\right)\|^2\nonumber \\
		=&\mathbb{E}\left\|\widetilde{g}_{t}'^{\top} \widetilde{f}_{t}'-g'(z_t)^{\top}\nabla f[g(z_t)]\right\|^2\nonumber \\
		=& \mathbb{E}\left\|\widetilde{g}_{t}'^{\top}\left[\widetilde{f}_{t}'-\nabla f(\widetilde{g}_t)+\nabla f(\widetilde{g}_t)-\nabla f[g(z_t)]\right]+[\widetilde{g}_{t}'-g'(z_t)]^{\top}\nabla f[g(z_t)]\right\|^2\nonumber \\
		\le& 3\mathbb{E}\left\|\widetilde{g}_{t}'^{\top}\left[\widetilde{f}_{t}'-\nabla f(\widetilde{g}_t)\right]\right\|^2 +3\mathbb{E}\left\|\widetilde{g}_{t}'^{\top}\left[\nabla f(\widetilde{g}_t)-\nabla f[g(z_t)]\right]\right\|^2\nonumber \\ &+3\mathbb{E}\left\|[\widetilde{g}_{t}'-g'(z_t)]^{\top}\nabla f[g(z_t)]\right\|^2\nonumber \\
		\le& 3\mathbb{E}\left[\left\|\widetilde{g}_{t}'\right\|^2
		\left(\left\|\widetilde{f}_{t}'-\nabla f(\widetilde{g}_t)\right\|^2 +\left\|\nabla f(\widetilde{g}_t)-\nabla f[g(z_t)]\right\|^2\right)\right]\nonumber \\ &+3\mathbb{E}\left[\left\|\widetilde{g}_{t}'-g'(z_t)\right\|^2\left\|\nabla f[g(z_t)]\right\|^2\right]\nonumber \\
		\le& 3\left(l_g+\frac{2\sqrt{26}}{3}\tau L_g\beta\epsilon\right)^2\left(\mathbb{E}\|\widetilde{f}_{\tau\lfloor t/\tau\rfloor}'-\nabla f(\widetilde{g}_{\tau\lfloor t/\tau\rfloor})\|^2\right.\nonumber \\
		&\left.+\frac{104}{9}L_f^2 l_g^2 \beta^2 \epsilon^2 \sum_{s=\tau\lfloor t/\tau\rfloor+1}^{t} \frac{1}{|\mathcal{B}_s|}+L_f^2\mathbb{E}\|\widetilde{g}_{t}-g(z_t)\|^2\right)\nonumber \\
		+&3l_f^2\left(\mathbb{E}\|\widetilde{g}_{\tau\lfloor t/\tau\rfloor}'-g'(z_{\tau\lfloor t/\tau\rfloor})\|^2+\frac{104}{9}L_g^2\beta^2\epsilon^2 \sum_{s=\tau\lfloor t/\tau\rfloor+1}^{t} \frac{1}{|\mathcal{A}_s'|}\right)\nonumber \\
		\le& \left(6l_g^2+\frac{208}{3}\tau^2 L_g^2\beta^2\epsilon^2\right)\left(\mathbb{E}\|\widetilde{f}_{\tau\lfloor t/\tau\rfloor}'-\nabla f(\widetilde{g}_{\tau\lfloor t/\tau\rfloor})\|^2\right.\nonumber \\
		&\left.+\frac{104}{9}L_f^2 l_g^2 \beta^2 \epsilon^2 \sum_{s=\tau\lfloor t/\tau\rfloor+1}^{t} \frac{1}{|\mathcal{B}_s|}+ L_f^2\mathbb{E}\|\widetilde{g}_{\tau\lfloor t/\tau\rfloor}-g(z_{\tau\lfloor t/\tau\rfloor})\|^2+\frac{104}{9}L_f^2l_g^2\beta^2\epsilon^2 \sum_{s=\tau\lfloor t/\tau\rfloor+1}^{t} \frac{1}{|\mathcal{A}_s|}
		\right)\nonumber \\
		&+3l_f^2\left(\mathbb{E}\|\widetilde{g}_{\tau\lfloor t/\tau\rfloor}'-g'(z_{\tau\lfloor t/\tau\rfloor})\|^2+\frac{104}{9}L_g^2\beta^2\epsilon^2 \sum_{s=\tau\lfloor t/\tau\rfloor+1}^{t} \frac{1}{|\mathcal{A}_s'|}\right). 
	\end{align}
	\indent For the problem $(\Sigma^2)$, we have
	\begin{equation}\label{eq_Err_sq_DFS}
	\widetilde{g}_{\tau\lfloor t/\tau\rfloor}=g(z_{\tau\lfloor t/\tau\rfloor}), \widetilde{g}_{\tau\lfloor t/\tau\rfloor}'=g'(z_{\tau\lfloor t/\tau\rfloor}), \widetilde{f}_{\tau\lfloor t/\tau\rfloor}'=\nabla f(\widetilde{g}_{\tau\lfloor t/\tau\rfloor}). 
	\end{equation}
	Use the following hyperparameter choice which fits the item 1 of \Cref{thm_alpha_decay_DE_nonconvex}
	\begin{align}\label{hyperpar_alpha_decay_DFS_nonconvex}
	&\alpha_{t}=\frac{2}{t+1}, \beta_t\equiv\beta\equiv \frac{\sqrt{3}}{2\sqrt{26(5C_1+4C_2+1)}L_F},  \nonumber\\
	&\tau=\lfloor\sqrt{\max(N,n)}\rfloor, \nonumber\\ 
	&|\mathcal{A}_t|=|\mathcal{A}_t'|=\left\{ \begin{gathered}
	n; t ~\text{\rm mod}~\tau=0  \hfill \\
	\lceil\tau/C_1\rceil; \text{\rm Otherwise} \hfill \\ 
	\end{gathered}  \right., \nonumber\\ 
	&|\mathcal{B}_t|=\left\{ \begin{gathered}
	N; t ~\text{\rm mod}~\tau=0  \hfill \\
	\lceil\tau/C_2\rceil; \text{\rm Otherwise} \hfill \\ 
	\end{gathered}  \right., %\numberthis
	\end{align}
	\noindent where $C_1, C_2$ are the constant upper bounds of $\sqrt{N}/n$ and $\sqrt{n}/N$ respectively ($C_1, C_2$ are constant since $N\le\mathcal{O}(n^2)$, $n\le\mathcal{O}(N^2)$ are assumed for the problem $(\Sigma^2)$). We can let $C_1, C_2>1$ such that $|\mathcal{A}_t|, |\mathcal{A}_t'|<n, |\mathcal{B}_t|<N$. Then, by substituting Eqs. \eqref{eq_Err_sq_DFS} \& \eqref{hyperpar_alpha_decay_DFS_nonconvex} into Eq. \eqref{Fgrad_err_derive_double}, we have
	\begin{align}
		\mathbb{E}\|\widetilde{\nabla} F\left(z_{t}\right)-\nabla F\left(z_{t}\right)\|^2 \stackrel{(i)}{\le} & \left[6l_g^2+\frac{208}{3}\tau^2 L_g^2\frac{3}{104L_F^2}\left(\frac{l_fl_g}{\tau}\right)^2\right]\left(\frac{104}{9}L_f^2l_g^2\beta^2\epsilon^2 \right)(C_1+C_2)+\frac{104}{3}l_f^2L_g^2\beta^2\epsilon^2 C_1\nonumber\\
		\stackrel{(ii)}{\le}& \left(6l_g^2+2L_g^2\frac{1}{l_f^2L_g^2}l_f^2l_g^2\right)\left(\frac{104}{9}L_f^2l_g^2\epsilon^2 \right)\frac{3}{104(5C_1+4C_2+1)L_F^2}(C_1+C_2)\nonumber\\
		&+\frac{104}{3}l_f^2L_g^2\frac{3}{104(5C_1+4C_2+1)L_F^2}\epsilon^2 C_1\nonumber\\
		\stackrel{(iii)}{\le}& 8l_g^2\left(\frac{1}{3}L_f^2l_g^2\epsilon^2 \right) \frac{1}{(5C_1+4C_2+1)l_g^4L_f^2}(C_1+C_2)+
		\frac{l_f^2L_g^2\epsilon^2 C_1}{(5C_1+4C_2+1)l_f^2L_g^2}\nonumber\\
		=&\frac{11C_1+8C_2}{3(5C_1+4C_2+1)}\epsilon^2\le\epsilon^2\nonumber
	\end{align}
	\noindent where (i) uses $\epsilon\le l_fl_g/\sqrt{\max(N,n)}\le l_fl_g/\tau$ and $\beta\le\sqrt{3}/(2\sqrt{26}L_F)$, (ii) uses $L_F=l_fL_g+l_g^2L_f\ge l_fL_g$, (iii) uses $L_F\ge l_fL_g$ and $L_F=l_fL_g+l_g^2L_f\ge l_g^2L_f$. This proves Eq. \eqref{eq_approx_grad_double} for the problem $(\Sigma^2)$. 
	
	For the problem $(\mathbb{E}^2)$, it can derived from \Cref{assumption5} that 
	\begin{align}\label{eq_Err_sq_DE}
		\mathbb{E}\|\widetilde{g}_{\tau\lfloor t/\tau\rfloor}-g(z_{\tau\lfloor t/\tau\rfloor})\|^2=\mathbb{E}\left\|\frac{1}{|\mathcal{A}_{\tau\lfloor t/\tau\rfloor}|}\sum_{\xi\in\mathcal{A}_{\tau\lfloor t/\tau\rfloor}} g_{\xi}(z_{\tau\lfloor t/\tau\rfloor})-g(z_{\tau\lfloor t/\tau\rfloor})\right\|^2 \le& \frac{\sigma_{g}^2}{|\mathcal{A}_{\tau\lfloor t/\tau\rfloor}|}\nonumber \\
		\mathbb{E}\|\widetilde{g}_{\tau\lfloor t/\tau\rfloor}'-g'(z_{\tau\lfloor t/\tau\rfloor})\|^2=\mathbb{E}\left\|\frac{1}{|\mathcal{A}_{\tau\lfloor t/\tau\rfloor}'|}\sum_{\xi\in\mathcal{A}_{\tau\lfloor t/\tau\rfloor}'} g_{\xi}'(z_{\tau\lfloor t/\tau\rfloor})-g'(z_{\tau\lfloor t/\tau\rfloor})\right\|^2 \le& \frac{\sigma_{g'}^2}{|\mathcal{A}_{\tau\lfloor t/\tau\rfloor}'|}\nonumber \\
		\mathbb{E}\|\widetilde{f}_{\tau\lfloor t/\tau\rfloor}'-\nabla f(\widetilde{g}_{\tau\lfloor t/\tau\rfloor})\|^2=\mathbb{E}\left\|\frac{1}{|\mathcal{B}_{\tau\lfloor t/\tau\rfloor}|}\sum_{\xi\in\mathcal{B}_{\tau\lfloor t/\tau\rfloor}} \nabla f_{\xi}(\widetilde{g}_{\tau\lfloor t/\tau\rfloor})-\nabla f(\widetilde{g}_{\tau\lfloor t/\tau\rfloor})\right\|^2 \le& \frac{\sigma_{f'}^2}{|\mathcal{B}_{\tau\lfloor t/\tau\rfloor}|},
	\end{align}
	Using the following hyperparameter choice which fits the item 2 of \Cref{thm_alpha_decay_DE_nonconvex} and substituting Eq. \eqref{eq_Err_sq_DE} into Eq. \eqref{Fgrad_err_derive_double}, 
	\begin{align}\label{hyperpar_alpha_decay_DE_nonconvex}
	&\alpha_{t}=\frac{2}{t+1},  \beta_t\equiv\beta=\frac{3}{2\sqrt{26}L_F}, \nonumber\\
	&\beta\le\lambda_{t}\le(1+\alpha_t)\beta, \tau=\left\lfloor\frac{l_fl_g}{\epsilon}\right\rfloor, \nonumber\\
	&|\mathcal{A}_t|=\left\{ \begin{gathered}
	\left\lceil\frac{54L_f^2l_g^2\sigma_g^2}{\epsilon^2}\right\rceil; t ~\text{\rm mod}~\tau=0  \hfill \\
	54\tau=54\left\lfloor\frac{l_fl_g}{\epsilon}\right\rfloor; \text{\rm Otherwise} \hfill \\ 
	\end{gathered}  \right., \nonumber\\ 
	&|\mathcal{A}_t'|=\left\{ \begin{gathered}
	\left\lceil\frac{54l_f^2\sigma_{g'}^2}{\epsilon^2}\right\rceil; t ~\text{\rm mod}~\tau=0  \hfill \\
	54\tau=54\left\lfloor\frac{l_fl_g}{\epsilon}\right\rfloor; \text{\rm Otherwise} \hfill \\ 
	\end{gathered}  \right., \nonumber\\
	&|\mathcal{B}_t|=\left\{ \begin{gathered}
	\left\lceil\frac{54l_g^2\sigma_{f'}^2}{\epsilon^2}\right\rceil; t ~\text{\rm mod}~\tau=0  \hfill \\
	54\tau=54\left\lfloor\frac{l_fl_g}{\epsilon}\right\rfloor; \text{\rm Otherwise} \hfill \\ 
	\end{gathered}  \right.  %\numberthis
	\end{align}
	we have 
	\begin{align}
	\mathbb{E}\|\widetilde{\nabla} F\left(z_{t}\right)-\nabla F\left(z_{t}\right)\|^2 \le & \left[6l_g^2+\frac{208}{3}\left(\frac{l_fl_g}{\epsilon}\right)^2 L_g^2\frac{9}{104L_F^2}\epsilon^2\right]\left[\frac{\sigma_{f'}^2}{|\mathcal{B}_{\tau\lfloor t/\tau\rfloor}|}+\left(\frac{104}{9}L_f^2l_g^2\frac{9}{104L_F^2}\epsilon^2 \right)\left(\frac{1}{54}+\frac{1}{54}\right)\right.\nonumber\\
	&\left.+L_f^2\frac{\sigma_{g}^2}{|\mathcal{A}_{\tau\lfloor t/\tau\rfloor}|}\right] +3l_f^2\left(\frac{\sigma_{g'}^2}{|\mathcal{A}_{\tau\lfloor t/\tau\rfloor}'|}+\frac{104}{9}l_f^2L_g^2\frac{9}{104L_F^2}\epsilon^2 \frac{1}{54}\right)\nonumber\\
	\stackrel{(i)}{\le}& 12l_g^2\left(\frac{\epsilon^2\sigma_{f'}^2}{54l_g^2\sigma_{f'}^2}+\frac{\epsilon^2}{27l_g^2}+L_f^2\frac{\epsilon^2\sigma_{g}^2}{54L_f^2l_g^2\sigma_{g}^2}\right) + 3l_f^2\left(\frac{\epsilon^2\sigma_{g'}^2}{54l_f^2\sigma_{g'}^2} + \frac{\epsilon^2}{54}\right)\nonumber\\
	=&\epsilon^2,\nonumber
	\end{align}
	where (i) uses $L_F=l_fL_g+l_g^2L_f\ge l_fL_g$ and $L_F\ge l_g^2L_f$. This proves Eq. \eqref{eq_approx_grad_double} for the problem $(\mathbb{E}^2)$.
\end{proof}

\section{Proof of Theorem \ref{thm_alpha_decay_DE_nonconvex}}\label{sec_proof_decay_thm_double}
%The proof is similar to that of Theorem \ref{thm_alpha_const_SCE_nonconvex} \& \ref{thm_alpha_decay_SCFS_nonconvex}.\\
\indent By using the convexity of $r$ and $x_{t+1}=\left(1-\theta_{t}\right)x_t+\theta_{t}\widetilde{x}_{t+1}$ in \Cref{alg_double}, we obtain 
\begin{align}\label{eq_r_convex_double}
r(x_{t+1})-r(x_t)=&r\left[\left(1-\theta_{t}\right)x_t+\theta_{t}\widetilde{x}_{t+1}\right]-r(x_t)\nonumber\\
\le& \left(1-\theta_{t}\right) r(x_t) + \theta_{t}r(\widetilde{x}_{t+1})-r(x_t)\nonumber\\
=& \theta_{t} [r(\widetilde{x}_{t+1})-r(x_t)]\nonumber\\
\stackrel{(i)}{\le}& -\theta_{t} \left[\frac{1}{\lambda_t}\| \widetilde{x}_{t+1}- x_t\|^2 + \left<{\widetilde{\nabla}F}( z_{t}), \widetilde{x}_{t+1}- x_{t}\right>\right].
\end{align}
\noindent where (i) uses slightly modified Eq. \eqref{eq_r_convex} where $x_{t+1}$ is replaced with $\widetilde{x}_{t+1}$, since now $\widetilde{x}_{t+1}$ is the minimizer of the function $\widetilde{r}(x):=r(x)+\frac{1}{2\lambda_t}\|x- x_{t}+\lambda_{t}{\widetilde{\nabla}F}(z_{t})\|^2$. Eq. \eqref{eq_f_Taylor} still holds since $\nabla F$ is still $L_F$-Lipschitz. By adding up Eqs. \eqref{eq_r_convex_double} \& \eqref{eq_f_Taylor} and taking expectation, it can be derived that
\begin{align}\label{eq_adjacent_Phi_diff_bound_double}
	\mathbb{E}\Phi( x_{t+1})-\mathbb{E}\Phi( x_t)\le& \mathbb{E}\left< \nabla F( x_{t}),  x_{t+1}-x_t\right> + \frac{L_F}{2} \mathbb{E}\| x_{t+1}- x_t\|^2\nonumber\\
	& -\mathbb{E}\left[\frac{\theta_{t}}{\lambda_t}\| \widetilde{x}_{t+1}- x_t\|^2 + \theta_{t}\left<{\widetilde{\nabla}F}( z_{t}), \widetilde{x}_{t+1}- x_{t}\right>\right]\nonumber\\
	\le& \mathbb{E}\left[\theta_{t}\left<\nabla F(x_{t}), \widetilde{x}_{t+1}-x_t\right>\right] + \frac{L_F}{2} \mathbb{E}(\theta_{t}^2\| \widetilde{x}_{t+1}- x_t\|^2)\nonumber\\
	& -\mathbb{E}\left[\frac{\theta_{t}}{\lambda_t}\| \widetilde{x}_{t+1}- x_t\|^2 + \theta_{t}\left<{\widetilde{\nabla}F}( z_{t}), \widetilde{x}_{t+1}- x_{t}\right>\right]\nonumber\\
	\le& \mathbb{E}\left(\theta_{t}\left< \nabla F( x_{t})-{\widetilde{\nabla}F}( z_{t}), \widetilde{x}_{t+1}- x_t\right>\right)+\mathbb{E}\left[\left(\frac{L_F\theta_{t}^2}{2}-\frac{\theta_{t}}{\lambda_t}\right)\| \widetilde{x}_{t+1}- x_t\|^2\right]\nonumber\\
	\le& \mathbb{E}\left< L_F^{-1/2}[\nabla F( x_{t})-{\widetilde{\nabla}F}(z_{t})], \theta_{t}\sqrt{L_F}(\widetilde{x}_{t+1}- x_t)\right>\nonumber\\
	&+\mathbb{E}\left[\left(\frac{L_F\theta_{t}^2}{2}-\frac{\theta_{t}}{\lambda_t}\right)\| \widetilde{x}_{t+1}- x_t\|^2\right]\nonumber\\
	\stackrel{(i)}{\le}& \frac{1}{2L_F} \mathbb{E}\|\nabla F( x_{t})-{\widetilde{\nabla}F}(z_{t})\|^2 +\mathbb{E}\left[\left(L_F\theta_{t}^2-\frac{\theta_{t}}{\lambda_t}\right)\| \widetilde{x}_{t+1}- x_t\|^2\right]\nonumber\\
	\stackrel{(ii)}{\le}& \frac{1}{L_F} \mathbb{E}\|\nabla F( x_{t})-\nabla F(z_{t})\|^2 + \frac{1}{L_F} \mathbb{E}\|\nabla F( z_{t})-{\widetilde{\nabla}F}(z_{t})\|^2 \nonumber\\
	&+\mathbb{E}\left[\left(L_F\theta_t^2-\frac{5}{3}L_F\theta_t\right)\| \widetilde{x}_{t+1}- x_t\|^2\right]\nonumber\\
	\stackrel{(iii)}{\le}& L_F\mathbb{E}\|z_t-x_t\|^2+\frac{\epsilon^2}{L_F}+\mathbb{E}\left[\max\left(\frac{1}{4}L_F-\frac{5}{6}L_F,\right.\right.\nonumber\\
	&\left.\left.\frac{L_F\epsilon^2\lambda_{t}^2}{\|\widetilde{x}_{t+1}-x_t\|^2}-\frac{5L_F}{3}\frac{\epsilon\lambda_{t}}{\|\widetilde{x}_{t+1}-x_t\|}\right) \| \widetilde{x}_{t+1}- x_t\|^2\right],\nonumber\\
	\stackrel{(iv)}{\le}& 4L_F\beta^2\epsilon^2+\frac{\epsilon^2}{L_F}+L_F\mathbb{E}\left[\max\left(-\frac{7}{12}\| \widetilde{x}_{t+1}- x_t\|^2, \epsilon^2\lambda_{t}^2-\frac{5}{3}\epsilon\lambda_{t}\|\widetilde{x}_{t+1}-x_t\|\right)\right]\nonumber\\
	\stackrel{(v)}{\le}& \frac{9\epsilon^2}{25L_F}+\frac{\epsilon^2}{L_F}+L_F\mathbb{E}\left(\frac{25}{21}\epsilon^2\lambda_{t}^2-\frac{5}{3}\epsilon\lambda_{t}\|\widetilde{x}_{t+1}-x_t\|\right)\nonumber\\
	\stackrel{(vi)}{\le}&\frac{34\epsilon^2}{25L_F}+L_F\frac{25\epsilon^2}{21}\frac{9}{25L_F^2}-\frac{5}{3} L_F\beta\epsilon\mathbb{E}\|\widetilde{x}_{t+1}-x_t\|\nonumber\\
	=&\frac{313\epsilon^2}{175L_F}-\frac{5}{3} L_F\beta\epsilon\mathbb{E}\|\widetilde{x}_{t+1}-x_t\|,
\end{align}
\noindent where (i) uses the inequality $<a,b>\le\|a\| \|b\|\le\frac{\|a\|^2+\|b\|^2}{2}$, (ii) uses the inequality $\|a+b\|^2\le 2\|a\|^2+2\|b\|^2$ and the inequality $\lambda_t\le 2\beta\le \frac{3}{\sqrt{26}L_F}\le \frac{3}{5L_F}$ that holds for both hyperparameter choices \eqref{hyperpar_alpha_decay_DFS_nonconvex}\&\eqref{hyperpar_alpha_decay_DE_nonconvex}, (iii) uses Lemma \ref{lemma_approx_grad_double} and  $\theta_t=\min\left(\frac{\epsilon\lambda_{t}}{\|\widetilde{x}_{t+1}-x_t\|}, \frac{1}{2}\right)$ from \Cref{alg_double}, (iv) uses Eq. \eqref{eq_xzdiff_bound_double}, (v) uses $\beta\le \frac{3}{2\sqrt{26}L_F}\le \frac{3}{10L_F}$, and the following two inequalities where $U=\|\widetilde{x}_{t+1}-x_t\|$, $V=\epsilon\lambda_t$, $\footnote{This technique to remove max is inspired by \cite{zhang2019multi}.}$ (vi) uses $\beta\le\lambda_t\le \frac{3}{5L_F}$ that holds for both hyperparameter choices \eqref{hyperpar_alpha_decay_DFS_nonconvex}\&\eqref{hyperpar_alpha_decay_DE_nonconvex}.  \\
$$\frac{25}{21}V^2-\frac{5}{3}UV+\frac{7}{12}U^2=\frac{1}{84}(7U-10V)^2\ge 0 \Rightarrow -\frac{7}{12}U^2\le \frac{25}{21}V^2-\frac{5}{3}UV,$$ $$V^2-\frac{5}{3}UV\le \frac{25}{21}V^2-\frac{5}{3}UV.$$

\indent By telescoping Eq. \eqref{eq_adjacent_Phi_diff_bound_double}, we obtain
\begin{align}\label{eq_adjacent_Phi_bound_double}
	&\mathbb{E}\Phi(x_{T})-\mathbb{E}\Phi(x_0)\le \frac{313\epsilon^2T}{175L_F}-\frac{5}{3} L_F\beta\epsilon\sum_{t=0}^{T-1}\mathbb{E}\|\widetilde{x}_{t+1}-x_t\|\nonumber\\
	\Rightarrow& \frac{1}{T}\sum_{t=0}^{T-1}\mathbb{E}\|\widetilde{x}_{t+1}-x_t\| \le \frac{1}{L_F\beta}\left[\frac{939\epsilon}{875L_F}+\frac{3}{5\epsilon T}\mathbb{E}[\Phi(x_0)-\Phi(x_T)]\right], 
\end{align}

As a result,
\begin{align}\label{eq_Egrad_bound_double}
	\mathbb{E}_{\xi}\|{\mathcal{G}}_{\lambda_\xi}( z_\xi)\|=&\frac{1}{T}\sum_{t=0}^{T-1} \mathbb{E}\|{\mathcal{G}}_{\lambda_t}(z_t)\|\nonumber\\
	\stackrel{(i)}{=}&\frac{1}{T}\sum_{t=0}^{T-1} \lambda_t^{-1}\mathbb{E}\left\| z_t-\operatorname{prox}_{\lambda_t r}\left[ z_t-\lambda_t\nabla F( z_t)\right]\right\|\nonumber\\
	\stackrel{(ii)}{ =}&\frac{1}{T}\sum_{t=0}^{T-1} \lambda_t^{-1}\mathbb{E}\left\|(z_t-x_t)+(x_t- \widetilde{x}_{t+1})\right.\nonumber\\
	&+\left.\left\{\operatorname{prox}_{\lambda_{t}r}[ x_{t}-\lambda_{t}{\widetilde{\nabla}F}( z_{t})]-\operatorname{prox}_{\lambda_t r}\left[ z_t-\lambda_t\nabla F( z_t)\right]\right\}\right\|\nonumber\\
	\stackrel{(iii)}{\le}&\frac{1}{T}\sum_{t=0}^{T-1} \lambda_t^{-1}\mathbb{E}\| z_t- x_t\|+\frac{1}{T}\sum_{t=0}^{T-1} \lambda_t^{-1}\mathbb{E}\| \widetilde{x}_{t+1}- x_t\|\nonumber\\
	&+\frac{1}{T}\sum_{t=0}^{T-1} \lambda_t^{-1}\mathbb{E}\left\|( x_{t}- z_t)+\lambda_{t}\left[\nabla F( z_t)-{\widetilde{\nabla}F}( z_{t})]\right]\right\|\nonumber\\
	\stackrel{(iv)}{\le}&\frac{2}{T}\sum_{t=0}^{T-1} \lambda_t^{-1} \mathbb{E}\| z_t- x_t\|+\frac{1}{T}\sum_{t=0}^{T-1} \lambda_t^{-1}\mathbb{E}\| \widetilde{x}_{t+1}- x_t\|+\frac{1}{T}\sum_{t=0}^{T-1} \mathbb{E}\left\|\nabla F(z_t)-{\widetilde{\nabla}F}( z_{t})]\right\|\nonumber\\
	\stackrel{(v)}{\le}& \frac{2}{T} T\beta^{-1}(2\beta\epsilon) + \beta^{-1}\frac{1}{T}\sum_{t=0}^{T-1} \mathbb{E}\| \widetilde{x}_{t+1}- x_t\| + \epsilon \nonumber\\
	\stackrel{(vi)}{\le}& 5\epsilon+\frac{1}{L_F\beta^2}\left[\frac{939\epsilon}{875L_F}+\frac{3}{5\epsilon T}\mathbb{E}[\Phi(x_0)-\Phi(x_T)]\right],
\end{align}
\noindent where (i) uses Eq. \eqref{eq_prox_grad}, (ii) uses $\widetilde{x}_{t+1}=\operatorname{prox}_{\lambda_{t}r}(x_{t}-\lambda_{t}{\widetilde{\nabla}F}( z_{t}))$ in \Cref{alg_double}, (iii) uses triangle inequality and the non-expansive property of proximal operator (See Section 31 of \cite{pryce1973r} for detail), (iv) uses triangle inequality, (v) uses Eq. \eqref{eq_xzdiff_bound_double}, Lemma \ref{lemma_approx_grad_double}, and $\lambda_t\ge\beta$ that holds for both hyperparameter choices \eqref{hyperpar_alpha_decay_DFS_nonconvex}\&\eqref{hyperpar_alpha_decay_DE_nonconvex}, and (vi) uses Eq. \eqref{eq_adjacent_Phi_bound_double}. 

\indent Using the hyperparameter choice \eqref{hyperpar_alpha_decay_DFS_nonconvex} for the problem $(\Sigma^2)$, it can be derived by substituting $\beta=\frac{\sqrt{3}}{2\sqrt{26(5C_1+4C_2+1)}L_F}$ into Eq. \eqref{eq_Egrad_bound_double} that
\begin{align}\label{eq_Egrad_bound2_DFS}
\mathbb{E}_{\xi}\|{\mathcal{G}}_{\lambda_\xi}(z_\xi)\|\le& 5\epsilon+\frac{104(5C_1+4C_2+1)(313\epsilon)}{875}+\frac{104(5C_1+4C_2+1)L_F}{5\epsilon T}\mathbb{E}[\Phi(x_0)-\Phi(x_T)]\nonumber\\
\le& 38(5C_1+4C_2+2)\epsilon+\frac{104(5C_1+4C_2+1)L_F}{5\epsilon T}\mathbb{E}[\Phi(x_0)-\Phi(x_T)],
\end{align}
\noindent which proves Eq. \eqref{eq_conclude_alpha_decay_DFS} by letting $x^*=\arg\min_{ x \in \mathbb{R}^{d}}\Phi(x)$ and implies that $\mathbb{E}_{\xi}\|{\mathcal{G}}_{\lambda_\xi}( z_\xi)\|\le \mathcal{O}(\epsilon)$ when $T=\mathcal{O}(\epsilon^{-2})$. Then, using the hyperparameter choice \eqref{hyperpar_alpha_decay_DFS_nonconvex}, the sample complexity is equal to $\mathcal{O}\left(\sum_{t=0}^{T-1} |\mathcal{A}_t|+|\mathcal{A}_t'|+|\mathcal{B}_t|\right)$ where
\begin{align}
\sum_{t=0}^{T-1} |\mathcal{A}_t|+|\mathcal{A}_t'|+|\mathcal{B}_t|\le& \left(\left\lfloor \frac{T}{\tau} \right\rfloor+1\right)(N+2n)+\left(T-\left\lfloor \frac{T}{\tau} \right\rfloor-1\right) \mathcal{O}\left[\left(\frac{2}{C_1}+\frac{1}{C_2}\right)\tau\right]\nonumber\\
=&\mathcal{O}(\sqrt{\max(N,n)}\epsilon^{-2}+N+n),\nonumber
\end{align}
\noindent which uses $\tau=\lfloor\sqrt{\max(N,n)}\rfloor$ and $N+2n=\mathcal{O}(\max(N,n))$. 

\indent Using the hyperparameter choice \eqref{hyperpar_alpha_decay_DE_nonconvex} for the problem $(\mathbb{E}^2)$, it can be derived by substituting  $\beta=\frac{3}{2\sqrt{26}L_F}\ge\frac{1}{4L_F}$ into Eq. \eqref{eq_Egrad_bound_double} that
\begin{align}\label{eq_Egrad_bound2_DE}
	\mathbb{E}_{\xi}\|{\mathcal{G}}_{\lambda_\xi}(z_\xi)\|\le& 5\epsilon+\frac{16(939\epsilon)}{875}+\frac{48L_F}{5\epsilon T}\mathbb{E}[\Phi(x_0)-\Phi(x_T)]\nonumber\\
	\le& 23\epsilon+\frac{48L_F}{5\epsilon T}\mathbb{E}[\Phi(x_0)-\Phi(x_T)],
\end{align}
\noindent which proves Eq. \eqref{eq_conclude_alpha_decay_DE} by letting $x^*=\arg\min_{ x \in \mathbb{R}^{d}}\Phi(x)$ and implies that $\mathbb{E}_{\xi}\|{\mathcal{G}}_{\lambda_\xi}( z_\xi)\|\le \mathcal{O}(\epsilon)$ when $T=\mathcal{O}(\epsilon^{-2})$. Then, using the hyperparameter choice \eqref{hyperpar_alpha_decay_DE_nonconvex}, the sample complexity is equal to $\mathcal{O}\left(\sum_{t=0}^{T-1} |\mathcal{A}_t|+|\mathcal{A}_t'|+|\mathcal{B}_t|\right)$ where
$$
\sum_{t=0}^{T-1} |\mathcal{A}_t|+|\mathcal{A}_t'|+|\mathcal{B}_t|=\left(\left\lfloor \frac{T}{\tau} \right\rfloor+1\right) \mathcal{O}(\epsilon^{-2})+\left(T-\left\lfloor \frac{T}{\tau} \right\rfloor-1\right) \mathcal{O}(\epsilon^{-1})=\mathcal{O}(\epsilon^{-3}),
$$
%\indent When using restart strategy, Eq. \eqref{eq_Egrad_bound2_DE} becomes
%$$\frac{1}{T}\sum_{t=0}^{T-1}\|{\mathcal{G}}_{\lambda_{t}}( z_{t,m})\|^2\le 23\epsilon+\frac{48L_F}{5\epsilon T}\left[\mathbb{E}\Phi\left(x_{0,m}\right)-\mathbb{E}\Phi(x_{T,m})\right].$$
%Hence,
%$$\frac{1}{MT}\sum_{m=1}^{M}\sum_{t=0}^{T-1}\|{\mathcal{G}}_{\lambda_{t}}( z_{t,m})\|^2\le 23\epsilon+\frac{48L_F}{5\epsilon TM}\left[\Phi\left(x_{0,1}\right)-\mathbb{E}\Phi(x_{T,M})\right],$$
%which uses $x_{T-1,m}=x_{0,m+1}$ and implies Eq. \eqref{eq_conclude_alpha_decay_DE_nonconvex_restart}. By letting $M=\mathcal{O}(\epsilon^{-1})$ and $T=\mathcal{O}(\epsilon^{-1})$, we obtain $\mathbb{E}_{\xi}\|{\mathcal{G}}_{\lambda_\xi}( z_\xi)\|\le C^2\epsilon$ for some constant $C>0$ as well. Accordingly, the numbers of evaluations for functions $g_\xi$, $g_\xi'$ and $\nabla f_\eta'$ are equal to
%$$
%M\sum_{t=0}^{T-1} |\mathcal{A}_t|+|\mathcal{A}_t'|+|\mathcal{B}_t|=M\left(\left\lfloor \frac{T}{\tau} \right\rfloor+1\right) \mathcal{O}(\epsilon^{-2})+M\left(T-\left\lfloor \frac{T}{\tau} \right\rfloor-1\right) \mathcal{O}(\epsilon^{-1})=\mathcal{O}(\epsilon^{-3}),
%$$
%which gives the same computation complexity as implementing algorithm \ref{alg_double} only once with $T=\mathcal{O}(\epsilon^{-2})$ iterations.\\

\subsection{Proof of Convergence under Periodic Restart}\label{sec_Mdecay_thm_double}

We further obtain the following convergence rate result under the periodic restart scheme.
\begin{thm}\label{thm_alpha_decay_restart_DE_nonconvex}
Restart \Cref{alg_double} $M$ times and denote $x_{t,m}$, $y_{t,m}$, $z_{t,m}$ as the generated sequences in the $m$-th run. Set {the initialization for the ($m+1$)-th time as $x_{0,m+1}=x_{T-1,m}$ ($m=1, \ldots, M-1$)}. It can be shown that by using the hyperparameters in items 1 and 2 of \Cref{thm_alpha_decay_DE_nonconvex} respectively for the problems $(\Sigma^2)$ and $(\mathbb{E}^2)$, the output satisfies
\begin{align}\label{eq_conclude_alpha_decay_double_restart}
&\mathbb{E} \|{\mathcal{G}}_{\lambda_\zeta}( z_{\zeta, \delta})\|\le \mathcal{O}\Big(\epsilon+\frac{\Phi(x_{0,1})-\Phi^*}{\epsilon MT}\Big), \nonumber
\end{align}
\noindent where $\zeta$ is uniformly sampled from $\{0,\ldots,T-1\}$ and $\delta$ is uniformly sampled from $\{1,\ldots,M\}$. In particular, by choosing $M=T=\mathcal{O}(\epsilon^{-1})$, we can {achieve $\epsilon$-accuracy with} the same sample complexity results as those in \Cref{thm_alpha_decay_DE_nonconvex}.
\end{thm}

\begin{proof}
	The proof is similar to that of \Cref{alg: 1} with the momentum restart scheme in \Cref{subsec: restart1}.
	
	With momentum restart strategy for both the problems $(\Sigma^2)$ and $(\mathbb{E}^2)$, Eq. \eqref{eq_Egrad_bound2_DE} implies that for the $m$-th restart period,
	\begin{align*}
	\frac{1}{T}\sum_{t=0}^{T-1} \|{\mathcal{G}}_{\lambda_t}( z_{t, m})\|\le 23\epsilon+\frac{48L_F}{5\epsilon T}\mathbb{E}[\Phi(x_{0,m})-\Phi(x_{T-1,m})].
	\end{align*}
	Hence, 
	\begin{align*}
	\mathbb{E} \|{\mathcal{G}}_{\lambda_\zeta}( z_{\zeta,\delta})\|=\frac{1}{MT}\sum_{m=1}^{M}\sum_{t=0}^{T-1} \|{\mathcal{G}}_{\lambda_t}( z_{t, m})\|\le 23\epsilon+\frac{48L_F}{5\epsilon T}\mathbb{E}[\Phi(x_{0,1})-\Phi(x_{T-1,M})]\le \mathcal{O}\Big(\epsilon+\frac{\Phi(x_{0,1})-\Phi^*}{\epsilon MT}\Big),
	\end{align*}
	where we use $x_{0,m+1}=x_{T-1,m}$ and $\Phi(x_{T-1,M})\ge\Phi^*$.
\end{proof}

%\section{Proof of Theorem \ref{thm_alpha_decay_restart_DE_nonconvex}}\label{sec_proof_Mdecay_thm_double}

\section{Auxiliary Lemmas for Proving Theorem \ref{thm_alpha_const_DE_nonconvex}}
\begin{lemma}\label{lemma_seq_bound_alpha_const_double}
	Implement algorithm \ref{alg_double} with $\alpha_{t}\equiv\alpha\in(0,1], \beta_t\equiv\beta, \beta\le\lambda_{t}\le(1+\alpha)\beta
	$. The generated sequences $\{x_{t}, y_{t}, z_{t}\}$ satisfy the following conditions:
	\begin{equation}\label{eq_xdiff_bound_double_const}
	\left\|x_{t+1}-x_{t}\right\| \leq \epsilon\lambda_{t} \le 2\beta\epsilon
	\end{equation}
	\begin{equation}\label{eq_xydiff_bound_double_const}
	\left\|y_{t}-x_{t}\right\|^{2} \leq \frac{8\beta^2\epsilon^2}{\alpha},
	\end{equation}
	\begin{equation}\label{eq_xzdiff_bound_double_const}
	\left\|z_{t}-x_{t}\right\|^{2} \leq \frac{8\beta^2\epsilon^2}{\alpha},
	\end{equation}
	\begin{equation}\label{eq_zdiff_bound_double_const}
	\left\|z_{t+1}-z_{t}\right\|^{2} \leq 16\beta^2\epsilon^2
	\end{equation}
	\noindent where $\Gamma_t=\frac{2}{t(t+1)}$. Again, when $t=0$, the summation $\sum_{t=1}^{0}$ is 0 by default. 
\end{lemma}

\begin{proof}
	\indent Eq. \eqref{eq_xdiff_bound_double_const} can be directly derived from the following equation in \Cref{alg_double}. 
	\begin{align}
	x_{t+1}=(1-\theta_t)x_t+\theta_{t} \widetilde{x}_{t+1},~
	\theta_t=\min\left\{\frac{\epsilon\lambda_{t}}{\|\widetilde{x}_{t+1}-x_t\|},\frac{1}{2}\right\}.\nonumber
	\end{align}
	\indent Eqs. \eqref{eq_xydiff_alpha_const}-\eqref{eq_zdiff_bound_alpha_const} in Lemma \ref{lemma_seq_bound_alpha_decay} still hold because they are derived from $z_{t}=\left(1-\alpha_{t+1}\right)  y_{t}+\alpha_{t+1}x_{t}$ and $y_{t+1}= z_{t}+\frac{\beta_{t}}{\lambda_{t}}(x_{t+1}-x_{t})$ that are shared by both Algorithms \ref{alg: 1} \& \ref{alg_double}. Hence, it can be derived from Eqs. \eqref{eq_xydiff_bound_alpha_const} \& \eqref{eq_xdiff_bound_double_const} that
	\begin{align}
	\left\|y_{t}-x_{t}\right\|^{2} \le& \frac{t}{t+1}\sum_{s=1}^{t} (1-\alpha)^{2(t-s)}(t-s+1)(t-s+2)\frac{(\beta-\lambda_{s-1})^2}{\lambda_{s-1}^2} \|x_s-x_{s-1}\|^2\nonumber\nonumber\\
	\le& \sum_{s=1}^{t} (1-\alpha)^{2(t-s)}(t-s+1)(t-s+2)\left(1-\frac{\beta}{\lambda_{s-1}}\right)^2 (4\beta^2\epsilon^2)\nonumber\\
	\le& (4\beta^2\epsilon^2)\sum_{k=0}^{t-1} (1-\alpha)^{2k}(k+1)(k+2) \left(1-\frac{\beta}{(1+\alpha)\beta}\right)^2\nonumber\\
	\stackrel{(i)}{\le}& \frac{4\alpha^2\beta^2\epsilon^2}{(1+\alpha)^2} \frac{2}{\alpha^3(2-\alpha)^3}\nonumber\\
	\le&\frac{8\beta^2\epsilon^2}{\alpha},\nonumber
	\end{align}
	\noindent where (i) uses \eqref{eq_sum_s_bound_alpha_const}. This proves Eq. \eqref{eq_xydiff_bound_double_const}. 
	Eq. \eqref{eq_xzdiff_bound_double_const} can be directly derived from Eq. \eqref{eq_xydiff_bound_double_const}, $z_{t}=\left(1-\alpha_{t+1}\right)y_{t}+\alpha_{t+1}x_{t}$ and $(1-\alpha_{t+1})^2\le 1$. Then, Eq. \eqref{eq_zdiff_bound_alpha_const} indicates that
	\begin{align}
		\left\|z_{t+1}-z_{t}\right\|^{2} \le&\frac{2\beta^2}{\lambda_t^2}\left\|x_{t+1}- x_t \right\|^2+2\alpha^2 \frac{t+1}{t+2}\nonumber\\
		&\sum_{s=1}^{t+1} (1-\alpha)^{2(t-s+1)}(t-s+2)(t-s+3)\frac{(\beta-\lambda_{s-1})^2}{\lambda_{s-1}^2} \|x_s-x_{s-1}\|^2\nonumber\\
		\le& 2(4\beta^2\epsilon^2)+2\alpha^2 \frac{4\alpha^2\beta^2\epsilon^2}{(1+\alpha)^2} \sum_{k=0}^{t} (1-\alpha)^{2k}(k+1)(k+2)\nonumber\\
		\le& 8\beta^2\epsilon^2+\frac{8\alpha^4\beta^2\epsilon^2}{(1+\alpha)^2} \frac{2}{\alpha^3(2-\alpha)^3}\nonumber\\
		\le& 16\beta^2\epsilon^2,\nonumber
	\end{align}
	\noindent which follows almost the same way as the proof of Eq. \eqref{eq_xydiff_bound_double_const} above and proves Eq. \eqref{eq_zdiff_bound_double_const}. 
\end{proof}

\begin{lemma}\label{lemma_approx_grad_double_const}
	Let Assumptions \ref{assumption4} and \ref{assumption5} hold and apply \Cref{alg_double} to solve the problems $(\Sigma^2)$ and $(\mathbb{E}^2)$, with the hyperparameter choices in items 1 and 2 of \Cref{thm_alpha_const_DE_nonconvex} respectively. Then, the variance of the stochastic gradient satisfies
	\begin{equation}\label{eq_approx_grad_double_const}
	\mathbb{E}\left\|\widetilde{\nabla} F\left(z_{t}\right)-\nabla F\left(z_{t}\right)\right\|^{2} \le\epsilon^2.
	\end{equation}
\end{lemma}

\begin{proof}
	\indent The proof is almost the same as that of Lemma \ref{lemma_approx_grad_double}, except for the hyperparameter difference. 
	
	\indent In \Cref{alg_double}, $\widetilde{g}_{t}= \widetilde{g}_{t-1}+\frac{1}{|\mathcal{A}_t|}\sum_{\xi\in\mathcal{A}_t}\big( g_{\xi}( z_t)- g_{\xi}( z_{t-1})\big)$ for $t\mod\tau\ne0$. Hence, we get the following two inequalities.
	\begin{equation}\label{eq_inner_g_diff_double_const}
	\|\widetilde{g}_{t}-\widetilde{g}_{t-1}\| \le  \frac{1}{|\mathcal{A}_t|}\sum_{\xi\in\mathcal{A}_t} \|g_{\xi}(z_t)- g_{\xi}(z_{t-1})\| \le \frac{1}{|\mathcal{A}_t|}\sum_{\xi\in\mathcal{A}_t} l_g\|z_t-z_{t-1}\| \le 4l_g\beta\epsilon,
	\end{equation}
	\noindent where the last step uses Eq. \eqref{eq_zdiff_bound_double_const}. The second last step of Eq. \eqref{eq_inner_g_err_double_short} still holds, so by also using Eq. \eqref{eq_zdiff_bound_double_const}, we obtain
	\begin{align}\label{eq_inner_g_err_double_short_const}
	\mathbb{E}\|\widetilde{g}_{t}-g(z_t)\|^2 \le& \mathbb{E}\left\|\widetilde{g}_{t-1}-g(z_{t-1})\right\|^2+ \frac{l_g^2}{|\mathcal{A}_t|^2}\sum_{\xi\in\mathcal{A}_t}\mathbb{E}\|z_t-z_{t-1}\|^2\nonumber\\
	\le& \mathbb{E}\left\|\widetilde{g}_{t-1}-g(z_{t-1})\right\|^2+ \frac{16l_g^2}{|\mathcal{A}_t|}\beta^2\epsilon^2,
	\end{align}	
	\noindent By telescoping Eq. \eqref{eq_inner_g_err_double_short_const}, we obtain
	\begin{equation}\label{eq_inner_g_err_double_const}
		\mathbb{E}\|\widetilde{g}_{t}-g(z_t)\|^2 \le \mathbb{E}\|\widetilde{g}_{\tau\lfloor t/\tau\rfloor}-g(z_{\tau\lfloor t/\tau\rfloor})\|^2+16l_g^2\beta^2\epsilon^2 \sum_{s=\tau\lfloor t/\tau\rfloor+1}^{t} \frac{1}{|\mathcal{A}_s|}
	\end{equation}
	\indent In a similar way, we can get 
	\begin{equation}\label{eq_inner_ggrad_diff_double_const}
		\|\widetilde{g}_{t}'-\widetilde{g}_{t-1}'\| \le 4L_g\beta\epsilon.
	\end{equation}
	\begin{equation}\label{eq_inner_ggrad_err_double_const}
		\mathbb{E}\|\widetilde{g}_{t}'-g'(z_t)\|^2 \le \mathbb{E}\|\widetilde{g}_{\tau\lfloor t/\tau\rfloor}'-g'(z_{\tau\lfloor t/\tau\rfloor})\|^2+16L_g^2\beta^2\epsilon^2 \sum_{s=\tau\lfloor t/\tau\rfloor+1}^{t} \frac{1}{|\mathcal{A}_s'|}.
	\end{equation}
	\begin{align}\label{eq_inner_fgrad_err_double_const}
		\mathbb{E}\|\widetilde{f}_{t}'-\nabla f(\widetilde{g}_t)\|^2 \stackrel{(i)}{\le}& \mathbb{E}\|\widetilde{f}_{\tau\lfloor t/\tau\rfloor}'-\nabla f(\widetilde{g}_{\tau\lfloor t/\tau\rfloor})\|^2 + \sum_{s=\tau\lfloor t/\tau\rfloor+1}^{t} \frac{L_f^2}{|\mathcal{B}_s|}\mathbb{E}\left\|\widetilde{g}_s-\widetilde{g}_{s-1}\right\|^2\nonumber\\
		\le& \mathbb{E}\|\widetilde{f}_{\tau\lfloor t/\tau\rfloor}'-\nabla f(\widetilde{g}_{\tau\lfloor t/\tau\rfloor})\|^2 + 16L_f^2 l_g^2 \beta^2 \epsilon^2 \sum_{s=\tau\lfloor t/\tau\rfloor+1}^{t} \frac{1}{|\mathcal{B}_s|}, 
	\end{align}
	where (i) comes from the second last step of Eq. \eqref{eq_inner_fgrad_err_double}.
	
	\indent Furthermore, it can be obtained by telescoping Eq. \eqref{eq_inner_ggrad_diff_double_const} that
	\begin{align} \label{eq_inner_ggrad_bound_double_const}
		\|\widetilde{g}_{t}'\| \le& \|\widetilde{g}_{\tau\lfloor t/\tau\rfloor}'\|+\sum_{s=\tau\lfloor t/\tau\rfloor+1}^{t} \|\widetilde{g}_{s}'-\widetilde{g}_{s-1}'\|\nonumber\\ 
		\le& \left\|\frac{1}{|\mathcal{A}_{\tau\lfloor t/\tau\rfloor}'|}\sum_{\xi\in\mathcal{A}_{\tau\lfloor t/\tau\rfloor}'} g_{\xi}'(z_{\tau\lfloor t/\tau\rfloor})\right\|+4L_g\beta\epsilon(t-\tau\lfloor t/\tau\rfloor)\nonumber\\
		\le& l_g+4\tau L_g\beta\epsilon, 
	\end{align}
	\noindent where the second $\le$ uses Eq. \eqref{eq_inner_ggrad_diff_double_const}.
	
	\indent Therefore, 
	\begin{align}\label{Fgrad_err_derive_double_const}
	&\mathbb{E}\|\widetilde{\nabla} F\left(z_{t}\right)-\nabla F\left(z_{t}\right)\|^2\nonumber\\
	=&\mathbb{E}\left\|\widetilde{g}_{t}'^{\top} \widetilde{f}_{t}'-g'(z_t)^{\top}\nabla f[g(z_t)]\right\|^2\nonumber\\
	=& \mathbb{E}\left\|\widetilde{g}_{t}'^{\top}\left[\widetilde{f}_{t}'-\nabla f(\widetilde{g}_t)+\nabla f(\widetilde{g}_t)-\nabla f[g(z_t)]\right]+[\widetilde{g}_{t}'-g'(z_t)]^{\top}\nabla f[g(z_t)]\right\|^2\nonumber\\
	\le& 3\mathbb{E}\left\|\widetilde{g}_{t}'^{\top}\left[\widetilde{f}_{t}'-\nabla f(\widetilde{g}_t)\right]\right\|^2 +3\mathbb{E}\left\|\widetilde{g}_{t}'^{\top}\left[\nabla f(\widetilde{g}_t)-\nabla f[g(z_t)]\right]\right\|^2+3\mathbb{E}\left\|[\widetilde{g}_{t}'-g'(z_t)]^{\top}\nabla f[g(z_t)]\right\|^2\nonumber\\
	\le& 3\mathbb{E}\left[\left\|\widetilde{g}_{t}'\right\|^2
	\left(\left\|\widetilde{f}_{t}'-\nabla f(\widetilde{g}_t)\right\|^2 +\left\|\nabla f(\widetilde{g}_t)-\nabla f[g(z_t)]\right\|^2\right)\right]+3\mathbb{E}\left[\left\|\widetilde{g}_{t}'-g'(z_t)\right\|^2\left\|\nabla f[g(z_t)]\right\|^2\right]\nonumber\\
	\le& 3\left(l_g+4\tau L_g\beta\epsilon\right)^2\left(\mathbb{E}\|\widetilde{f}_{\tau\lfloor t/\tau\rfloor}'-\nabla f(\widetilde{g}_{\tau\lfloor t/\tau\rfloor})\|^2+16L_f^2 l_g^2 \beta^2 \epsilon^2 \sum_{s=\tau\lfloor t/\tau\rfloor+1}^{t} \frac{1}{|\mathcal{B}_s|}+L_f^2\mathbb{E}\|\widetilde{g}_{t}-g(z_t)\|^2\right)\nonumber\\
	+&3l_f^2\left(\mathbb{E}\|\widetilde{g}_{\tau\lfloor t/\tau\rfloor}'-g'(z_{\tau\lfloor t/\tau\rfloor})\|^2+16L_g^2\beta^2\epsilon^2 \sum_{s=\tau\lfloor t/\tau\rfloor+1}^{t} \frac{1}{|\mathcal{A}_s'|}\right)\nonumber\\
	\le& \left(6l_g^2+96\tau^2 L_g^2\beta^2\epsilon^2\right)\left(\mathbb{E}\|\widetilde{f}_{\tau\lfloor t/\tau\rfloor}'-\nabla f(\widetilde{g}_{\tau\lfloor t/\tau\rfloor})\|^2\right.\nonumber\\
	&\left.+16L_f^2 l_g^2 \beta^2 \epsilon^2 \sum_{s=\tau\lfloor t/\tau\rfloor+1}^{t} \frac{1}{|\mathcal{B}_s|}+ L_f^2\mathbb{E}\|\widetilde{g}_{\tau\lfloor t/\tau\rfloor}-g(z_{\tau\lfloor t/\tau\rfloor})\|^2+16L_f^2l_g^2\beta^2\epsilon^2 \sum_{s=\tau\lfloor t/\tau\rfloor+1}^{t} \frac{1}{|\mathcal{A}_s|}
	\right)\nonumber\\
	&+3l_f^2\left(\mathbb{E}\|\widetilde{g}_{\tau\lfloor t/\tau\rfloor}'-g'(z_{\tau\lfloor t/\tau\rfloor})\|^2+16L_g^2\beta^2\epsilon^2 \sum_{s=\tau\lfloor t/\tau\rfloor+1}^{t} \frac{1}{|\mathcal{A}_s'|}\right). 
	\end{align}
	For the problem $(\Sigma^2)$, Eq. \eqref{eq_Err_sq_DFS} still holds. Use the following hyperparameter choice which fits the item 1 of \Cref{thm_alpha_const_DE_nonconvex}. 
	\begin{align}\label{hyperpar_alpha_const_DFS_nonconvex}
	&\alpha_{t}\equiv\alpha, \beta_t\equiv\beta\equiv \frac{1}{16L_F\sqrt{10C_1+7C_2+1}}, \nonumber\\
	&\tau=\lfloor\sqrt{\max(N,n)}\rfloor\nonumber\\
	%\min\left[\frac{\sqrt{3}}{2\sqrt{26(5C_1+4C_2)}L_F},\frac{3}{2\sqrt{26}L_F}\right],\\ &\beta\le\lambda_{t}\le(1+\alpha_t)\beta, \tau=\left\lfloor{\sqrt{\max(N,n)}}\right\rfloor\le\frac{l_fl_g}{\epsilon}\\ &I_{r}=\text{False (i.e., Sample without replacement when $t$ mod $\tau=0$)},\\
	&|\mathcal{A}_t|=|\mathcal{A}_t'|=\left\{ \begin{gathered}
	n; t ~\text{\rm mod}~\tau=0  \hfill \\
	\lceil\tau/C_1\rceil; \text{\rm Otherwise} \hfill \\ 
	\end{gathered}  \right.,\nonumber\\ 
	&|\mathcal{B}_t|=\left\{ \begin{gathered}
	N; t ~\text{\rm mod}~\tau=0  \hfill \\
	\lceil\tau/C_2\rceil; \text{\rm Otherwise} \hfill  
	\end{gathered}  \right. %\numberthis
	\end{align}
	\noindent where $C_1, C_2>1$ are the constant upper bounds of $\sqrt{N}/n$ and $\sqrt{n}/N$ respectively (They are constant since $N\le\mathcal{O}(n^2)$, $n\le\mathcal{O}(N^2)$ are assumed for the problem $(\Sigma^2)$). We can let $C_1, C_2>1$ such that $|\mathcal{A}_t|, |\mathcal{A}_t'|<n, |\mathcal{B}_t|<N$. Then, by substituting Eqs. \eqref{eq_Err_sq_DFS} \& \eqref{hyperpar_alpha_const_DFS_nonconvex} into Eq. \eqref{Fgrad_err_derive_double_const}, we have
	\begin{align}
		&\mathbb{E}\|\widetilde{\nabla} F\left(z_{t}\right)-\nabla F\left(z_{t}\right)\|^2\nonumber\\
		\stackrel{(i)}{\le}&\left(6l_g^2+96l_f^2l_g^2\epsilon^{-2}L_g^2\frac{\epsilon^2}{16^2L_F^2}\right)\frac{16L_f^2 l_g^2 \epsilon^2(C_1+C_2)}{16^2(10C_1+7C_2+1)L_F^2} +\frac{48l_f^2L_g^2\epsilon^2 C_1}{16^2(10C_1+7C_2+1)L_F^2}\nonumber\\ 
		\stackrel{(ii)}{\le}&7l_g^2\frac{ \epsilon^2(C_1+C_2)}{16(10C_1+7C_2+1)l_g^2} + \frac{3\epsilon^2 C_1}{16(10C_1+7C_2+1)}\nonumber\\
		\le&\epsilon^2,\nonumber
	\end{align}
	\noindent where (i) uses Eq. \eqref{eq_Err_sq_DFS}, $\beta\le1/(16L_F)$ and $\epsilon\le l_fl_g(\max(N,n))^{-1/2}\le l_fl_g\tau^{-1}\Rightarrow\tau\le l_fl_g\epsilon^{-1}$, (ii) uses $L_F=l_fL_g+l_g^2L_f\ge l_fL_g$ and $L_F\ge l_g^2L_f$. This proves Eq. \eqref{eq_approx_grad_double_const} for the problem $(\Sigma^2)$.
	
	For the problem $(\mathbb{E}^2)$, Eq. \eqref{eq_Err_sq_DE} still holds. Use the following hyperparameter choice which fits the item 2 of \Cref{thm_alpha_const_DE_nonconvex}. 
	\begin{align}\label{hyperpar_alpha_const_DE_nonconvex}
	&\alpha_{t}\equiv\alpha,  \beta_t\equiv\beta=\frac{1}{10L_F},\nonumber\\ &\beta\le\lambda_{t}\le(1+\alpha)\beta, \tau=\left\lfloor\frac{l_fl_g}{\epsilon}\right\rfloor, \nonumber\\
	&|\mathcal{A}_t|=\left\{ \begin{gathered}
	\left\lceil\frac{34L_f^2l_g^2\sigma_g^2}{\epsilon^2}\right\rceil; t ~\text{\rm mod}~\tau=0  \hfill \\
	6\tau=6\left\lfloor\frac{l_fl_g}{\epsilon}\right\rfloor; \text{\rm Otherwise} \hfill \\ 
	\end{gathered}  \right., \nonumber\\ 
	&|\mathcal{A}_t'|=\left\{ \begin{gathered}
	\left\lceil\frac{34l_f^2\sigma_{g'}^2}{\epsilon^2}\right\rceil; t ~\text{\rm mod}~\tau=0  \hfill \\
	6\tau=6\left\lfloor\frac{l_fl_g}{\epsilon}\right\rfloor; \text{\rm Otherwise} \hfill \\ 
	\end{gathered}  \right., \nonumber\\
	&|\mathcal{B}_t|=\left\{ \begin{gathered}
	\left\lceil\frac{34l_g^2\sigma_{f'}^2}{\epsilon^2}\right\rceil; t ~\text{\rm mod}~\tau=0  \hfill \\
	6\tau=6\left\lfloor\frac{l_fl_g}{\epsilon}\right\rfloor; \text{\rm Otherwise} \hfill \\ 
	\end{gathered}  \right..  %\numberthis
	\end{align}
	Then, by substituting Eqs. \eqref{eq_Err_sq_DE} \& \eqref{hyperpar_alpha_const_DE_nonconvex} into Eq. \eqref{Fgrad_err_derive_double_const}, we have
	\begin{align}
		&\mathbb{E}\|\widetilde{\nabla} F\left(z_{t}\right)-\nabla F\left(z_{t}\right)\|^2\nonumber\\
		\stackrel{(i)}{\le}&\left(6l_g^2+96l_f^2l_g^2\epsilon^{-2}L_g^2\frac{\epsilon^2}{100L_F^2}\right)\left[\frac{\sigma_{f'}^2}{|\mathcal{B}_{\tau\lfloor t/\tau\rfloor}|}+\frac{16L_f^2 l_g^2 \epsilon^2\left(\frac{1}{6}+\frac{1}{6}\right)}{100L_F^2}  +\frac{L_f^2\sigma_{g}^2}{|\mathcal{A}_{\tau\lfloor t/\tau\rfloor}|}\right] + \frac{3l_f^2\sigma_{g'}^2}{|\mathcal{A}_{\tau\lfloor t/\tau\rfloor}'|} +\frac{48l_f^2L_g^2\epsilon^2 }{100L_F^2}\frac{1}{6}\nonumber\\ 
		\stackrel{(ii)}{\le}& 7l_g^2 \left(\frac{\sigma_{f'}^2\epsilon^2}{34l_g^2\sigma_{f'}^2}+\frac{16\epsilon^2}{300l_g^2}  +\frac{L_f^2\sigma_{g}^2\epsilon^2}{34L_f^2l_g^2\sigma_{g}^2}\right)+ \frac{3l_f^2\sigma_{g'}^2\epsilon^2}{34l_f^2\sigma_{g'}^2}+ \frac{8\epsilon^2 }{100}\nonumber\\
		\le& \epsilon^2\nonumber
	\end{align}
	\noindent where (i) uses Eq. \eqref{eq_Err_sq_DE}, (ii) uses $L_F=l_fL_g+l_g^2L_f\ge l_fL_g$ and $L_F\ge l_g^2L_f$. This proves Eq. \eqref{eq_approx_grad_double_const} for the problem $(\mathbb{E}^2)$.
%	\indent In the first case where the objective function is DE \eqref{doubleE} with Assumptions \ref{assumption4} $\&$ \ref{assumption5} and hyperparameter choice \eqref{hyperpar_alpha_const_DE_nonconvex}, Eq. \eqref{eq_Err_sq_DE} still holds. Then, Eq. \eqref{eq_approx_grad_double} for the first case can be obtained by subtituning Eqs. \eqref{hyperpar_alpha_const_DE_nonconvex} \& \eqref{eq_Err_sq_DE} into Eq. \eqref{Fgrad_err_derive_double_const} and using $L_F=l_fL_g+l_g^2L_f\ge l_fL_g$, $L_F\ge l_g^2L_f$.\\ 
%	\indent In the second case where the objective function is DFS \eqref{doublesum} with Assumptions \ref{assumption4} $\&$ \ref{assumption_DFS_Nn_bound} and hyperparameter choice \eqref{hyperpar_alpha_const_DFS_nonconvex}, Eq. \eqref{eq_Err_sq_DFS} still holds. Then, Eq. \eqref{eq_approx_grad_double} for the second case can be obtained by substituting Eqs. \eqref{hyperpar_alpha_const_DFS_nonconvex} \& \eqref{eq_Err_sq_DFS} into Eq. \eqref{Fgrad_err_derive_double} and using the following inequalities:
%	\begin{align}
%	&L_F=l_fL_g+l_g^2L_f\ge l_fL_g, L_F\ge l_g^2L_f;\nonumber\\ 
%	&\tau/|\mathcal{A}_t|=\tau/|\mathcal{A}_t'|\le C_1, \tau/|\mathcal{B}_t|\le C_2,~\forall t~{\rm mod}~\tau\ne 0.\nonumber
%	\end{align}	
\end{proof}
\section{Proof of Theorem \ref{thm_alpha_const_DE_nonconvex}}
	\indent The proof is almost the same as that of Theorems  \ref{thm_alpha_decay_DE_nonconvex} in Appendix \ref{sec_proof_decay_thm_double}, except for the differences in hyperparameter choice. The step (i) in Eq. \eqref{eq_adjacent_Phi_diff_bound_double} still holds as its derivation does not involve the hyperparameter difference. Starting from there, we have
	\begin{align} \label{eq_adjacent_Phi_diff_bound_double_const}
	\mathbb{E}\Phi( x_{t+1})-\mathbb{E}\Phi( x_t)
	\le & \frac{1}{2L_F} \mathbb{E}\|\nabla F( x_{t})-{\widetilde{\nabla}F}(z_{t})\|^2 +\mathbb{E}\left[\left(L_F\theta_{t}^2-\frac{\theta_{t}}{\lambda_t}\right)\| \widetilde{x}_{t+1}- x_t\|^2\right]\nonumber \\
	\stackrel{(i)}{\le}& \frac{1}{L_F} \mathbb{E}\|\nabla F( x_{t})-\nabla F(z_{t})\|^2 + \frac{1}{L_F} \mathbb{E}\|\nabla F( z_{t})-{\widetilde{\nabla}F}(z_{t})\|^2 \nonumber \\
	&+\mathbb{E}\left[\left(L_F\theta_t^2-5L_F\theta_t\right)\| \widetilde{x}_{t+1}- x_t\|^2\right]\nonumber \\
	\stackrel{(ii)}{\le}& L_F\mathbb{E}\|z_t-x_t\|^2+\frac{\epsilon^2}{L_F}+\mathbb{E}\left[\max\left(\frac{1}{4}L_F-\frac{5}{2}L_F,\right.\right.\nonumber \\
	&\left.\left.\frac{L_F\epsilon^2\lambda_{t}^2}{\|\widetilde{x}_{t+1}-x_t\|^2}-5L_F\frac{\epsilon\lambda_{t}}{\|\widetilde{x}_{t+1}-x_t\|}\right) \| \widetilde{x}_{t+1}- x_t\|^2\right]\nonumber \\
	\stackrel{(iii)}{\le}& \frac{8L_F}{\alpha}\beta^2\epsilon^2+\frac{\epsilon^2}{L_F}+L_F\mathbb{E}\left[\max\left(-\frac{9}{4}\| \widetilde{x}_{t+1}- x_t\|^2, \epsilon^2\lambda_{t}^2-5\epsilon\lambda_{t}\|\widetilde{x}_{t+1}-x_t\|\right)\right]\nonumber \\
	\stackrel{(iv)}{\le}& \frac{2\epsilon^2}{25\alpha L_F}+\frac{\epsilon^2}{L_F}+L_F\mathbb{E}\left[\frac{25}{9}\epsilon^2\lambda_{t}^2-5\epsilon\lambda_{t}\|\widetilde{x}_{t+1}-x_t\|\right]\nonumber \\
	\stackrel{(v)}{\le}&\frac{1+2/(25\alpha)}{L_F}\epsilon^2+L_F\frac{25\epsilon^2}{9}\frac{1}{25L_F^2}-5L_F\beta\epsilon\mathbb{E}\|\widetilde{x}_{t+1}-x_t\|\nonumber \\
	=&\frac{10/9+2/(25\alpha)}{L_F}\epsilon^2-5 L_F\beta\epsilon\mathbb{E}\|\widetilde{x}_{t+1}-x_t\|,
	\end{align}
	\noindent where (i) uses the inequality $\|a+b\|^2\le 2\|a\|^2+2\|b\|^2$ and the inequality $\lambda_t\le 2\beta\le \frac{1}{5L_F}$ that holds for both hyperparameter choices \eqref{hyperpar_alpha_const_DFS_nonconvex}\&\eqref{hyperpar_alpha_const_DE_nonconvex}, (ii) uses Lemma \ref{lemma_approx_grad_double_const} and  $\theta_t=\min\left(\frac{\epsilon\lambda_{t}}{\|\widetilde{x}_{t+1}-x_t\|}, \frac{1}{2}\right)$ from \Cref{alg_double}, (iii) uses Eq. \eqref{eq_xzdiff_bound_double_const}, (iv) uses $\beta\le \frac{1}{10L_F}$, and the following two inequalities where $U=\|\widetilde{x}_{t+1}-x_t\|$, $V=\epsilon\lambda_t$ $\footnote{This technique to remove max is inspired by \cite{zhang2019multi}.}$, (v) uses $\beta\le\lambda_t\le \frac{1}{5L_F}$ that holds for both hyperparameter choices \eqref{hyperpar_alpha_const_DFS_nonconvex}\&\eqref{hyperpar_alpha_const_DE_nonconvex}.  \\
	$$\frac{25}{9}V^2-5UV+\frac{9}{4}U^2=\left(\frac{3}{2}U-\frac{5}{3}V\right)^2\ge 0 \Rightarrow -\frac{9}{4}U^2\le \frac{25}{9}V^2-5UV,$$ 
	$$V^2-5UV\le \frac{25}{9}V^2-5UV.$$
	
	\indent By telescoping Eq. \eqref{eq_adjacent_Phi_diff_bound_double_const}, we obtain
	\begin{align}\label{eq_adjacent_Phi_bound_double_const}
		&\mathbb{E}\Phi(x_{T})-\mathbb{E}\Phi(x_0)\le \frac{10/9+2/(25\alpha)}{L_F}\epsilon^2 T-5 L_F\beta\epsilon\sum_{t=0}^{T-1}\mathbb{E}\|\widetilde{x}_{t+1}-x_t\|\nonumber\\
		\Rightarrow& \frac{1}{T}\sum_{t=0}^{T-1}\mathbb{E}\|\widetilde{x}_{t+1}-x_t\| \le \frac{1}{5L_F\beta}\left[\frac{10/9+2/(25\alpha)}{L_F}\epsilon +\frac{1}{\epsilon T}\mathbb{E}[\Phi(x_0)-\Phi(x_T)]\right],
	\end{align}
	
	The Step (iv) of Eq. \eqref{eq_Egrad_bound_double} still holds as its derivation does not involve the hyperparamter choice. Starting from there, we obtain
	\begin{align}\label{eq_Egrad_bound_double_const}
	\mathbb{E}_{\xi}\|{\mathcal{G}}_{\lambda_\xi}( z_\xi)\| {\le}&\frac{2}{T}\sum_{t=0}^{T-1} \lambda_t^{-1}\mathbb{E}\| z_t- x_t\|+\frac{1}{T}\sum_{t=0}^{T-1} \lambda_t^{-1}\mathbb{E}\| \widetilde{x}_{t+1}- x_t\|+\frac{1}{T}\sum_{t=0}^{T-1} \mathbb{E}\left\|\nabla F(z_t)-{\widetilde{\nabla}F}( z_{t})]\right\|\nonumber \\
	\stackrel{(i)}{\le}& \frac{2}{T} T\beta^{-1}(\sqrt{8/\alpha}\beta\epsilon) + \beta^{-1}\frac{1}{T}\sum_{t=0}^{T-1} \mathbb{E}\| \widetilde{x}_{t+1}- x_t\| + \epsilon \nonumber \\
	\stackrel{(ii)}{\le}& (4\sqrt{2/\alpha}+1)\epsilon+\frac{1}{5L_F\beta^2}\left[\frac{10/9+2/(25\alpha)}{L_F}\epsilon +\frac{1}{\epsilon T}\mathbb{E}[\Phi(x_0)-\Phi(x_T)]\right],
	\end{align}
	\noindent where (i) uses Eq. \eqref{eq_xzdiff_bound_double_const}, Lemma \ref{lemma_approx_grad_double_const}, and the inequality $\lambda_t\ge\beta$ that holds for both hyperparameter choices \eqref{hyperpar_alpha_const_DFS_nonconvex}\&\eqref{hyperpar_alpha_const_DE_nonconvex}, and (ii) used Eq. \eqref{eq_adjacent_Phi_bound_double_const}. Hence, Eqs. \eqref{eq_conclude_alpha_const_DFS} \& \eqref{eq_conclude_alpha_const_DE} always hold.
	
	\indent For the problem $(\Sigma^2)$, by using $T=\mathcal{O}(\epsilon^{-2})$ and the hyperparameter choice \eqref{hyperpar_alpha_const_DFS_nonconvex} which fit the item 1 of \Cref{thm_alpha_const_DE_nonconvex}, the sample complexity is $\mathcal{O}[ \sum_{t=0}^{T-1} (|\mathcal{A}_t|+|\mathcal{A}_t'|+|\mathcal{B}_t|)]$ where
	\begin{align}
	\sum_{t=0}^{T-1} (|\mathcal{A}_t|+|\mathcal{A}_t'|+|\mathcal{B}_t|)\le& \left(\left\lfloor \frac{T}{\tau} \right\rfloor+1\right)(N+2n)+\left(T-\left\lfloor \frac{T}{\tau} \right\rfloor-1\right) \mathcal{O}\left[\left(\frac{2}{C_1}+\frac{1}{C_2}\right)\tau\right]\nonumber\\
	=&\mathcal{O}(\sqrt{\max(N,n)}\epsilon^{-2}+N+n),\nonumber
	\end{align}
	\noindent Here we uses $N+2n=\mathcal{O}(\max(N,n))$. 
	
	\indent For the problem $(\mathbb{E}^2)$, by using $T=\mathcal{O}(\epsilon^{-2})$ and the hyperparameter choice \eqref{hyperpar_alpha_const_DE_nonconvex} which fit the item 2 of \Cref{thm_alpha_const_DE_nonconvex}, the sample complexity is $\mathcal{O}[ \sum_{t=0}^{T-1} (|\mathcal{A}_t|+|\mathcal{A}_t'|+|\mathcal{B}_t|)]$ where
	\begin{align}
	\sum_{t=0}^{T-1} (|\mathcal{A}_t|+|\mathcal{A}_t'|+|\mathcal{B}_t|)=\left(\left\lfloor \frac{T}{\tau} \right\rfloor+1\right) \mathcal{O}(\epsilon^{-2})+\left(T-\left\lfloor \frac{T}{\tau} \right\rfloor-1\right) \mathcal{O}(\epsilon^{-1})=\mathcal{O}(\epsilon^{-3}).\nonumber
	\end{align}

\end{document}